\PassOptionsToPackage{unicode}{hyperref}
\PassOptionsToPackage{hyphens}{url}
\PassOptionsToPackage{dvipsnames,svgnames,x11names}{xcolor}
\documentclass[
  12pt]{article}

\usepackage{amsmath,amssymb}
\usepackage{iftex}
\ifPDFTeX
  \usepackage[T1]{fontenc}
  \usepackage[utf8]{inputenc}
  \usepackage{textcomp} 
\else 
  \usepackage{unicode-math}
  \defaultfontfeatures{Scale=MatchLowercase}
  \defaultfontfeatures[\rmfamily]{Ligatures=TeX,Scale=1}
\fi
\usepackage{lmodern}
\ifPDFTeX\else  
\fi
\IfFileExists{upquote.sty}{\usepackage{upquote}}{}
\IfFileExists{microtype.sty}{
  \usepackage[]{microtype}
  \UseMicrotypeSet[protrusion]{basicmath} 
}{}
\makeatletter
\@ifundefined{KOMAClassName}{
  \IfFileExists{parskip.sty}{%
    \usepackage{parskip}
  }{
    \setlength{\parindent}{0pt}
    \setlength{\parskip}{6pt plus 2pt minus 1pt}}
}{
  \KOMAoptions{parskip=half}}
\makeatother
\usepackage{xcolor}
\makeatletter
\ifx\paragraph\undefined\else
  \let\oldparagraph\paragraph
  \renewcommand{\paragraph}{
    \@ifstar
      \xxxParagraphStar
      \xxxParagraphNoStar
  }
  \newcommand{\xxxParagraphStar}[1]{\oldparagraph*{#1}\mbox{}}
  \newcommand{\xxxParagraphNoStar}[1]{\oldparagraph{#1}\mbox{}}
\fi
\ifx\subparagraph\undefined\else
  \let\oldsubparagraph\subparagraph
  \renewcommand{\subparagraph}{
    \@ifstar
      \xxxSubParagraphStar
      \xxxSubParagraphNoStar
  }
  \newcommand{\xxxSubParagraphStar}[1]{\oldsubparagraph*{#1}\mbox{}}
  \newcommand{\xxxSubParagraphNoStar}[1]{\oldsubparagraph{#1}\mbox{}}
\fi
\makeatother

\usepackage{longtable,booktabs,array}
\usepackage{calc} 
\usepackage{etoolbox}
\makeatletter
\patchcmd\longtable{\par}{\if@noskipsec\mbox{}\fi\par}{}{}
\makeatother
\IfFileExists{footnotehyper.sty}{\usepackage{footnotehyper}}{\usepackage{footnote}}
\makesavenoteenv{longtable}
\usepackage{graphicx}
\makeatletter
\def\maxwidth{\ifdim\Gin@nat@width>\linewidth\linewidth\else\Gin@nat@width\fi}
\def\maxheight{\ifdim\Gin@nat@height>\textheight\textheight\else\Gin@nat@height\fi}
\makeatother
\setkeys{Gin}{width=\maxwidth,height=\maxheight,keepaspectratio}
\makeatletter
\def\fps@figure{htbp}
\makeatother

\makeatletter
\@ifpackageloaded{caption}{}{\usepackage{caption}}
\AtBeginDocument{%
\ifdefined\contentsname
  \renewcommand*\contentsname{Table of contents}
\else
  \newcommand\contentsname{Table of contents}
\fi
\ifdefined\listfigurename
  \renewcommand*\listfigurename{List of Figures}
\else
  \newcommand\listfigurename{List of Figures}
\fi
\ifdefined\listtablename
  \renewcommand*\listtablename{List of Tables}
\else
  \newcommand\listtablename{List of Tables}
\fi
\ifdefined\figurename
  \renewcommand*\figurename{Figure}
\else
  \newcommand\figurename{Figure}
\fi
\ifdefined\tablename
  \renewcommand*\tablename{Table}
\else
  \newcommand\tablename{Table}
\fi
}
\@ifpackageloaded{float}{}{\usepackage{float}}
\floatstyle{ruled}
\@ifundefined{c@chapter}{\newfloat{codelisting}{h}{lop}}{\newfloat{codelisting}{h}{lop}[chapter]}
\floatname{codelisting}{Listing}

\makeatother
\makeatletter
\@ifpackageloaded{caption}{}{\usepackage{caption}}
\@ifpackageloaded{subcaption}{}{\usepackage{subcaption}}
\makeatother

\ifLuaTeX
  \usepackage{selnolig}  
\fi
\usepackage[]{natbib}
\usepackage{bookmark}

\IfFileExists{xurl.sty}{\usepackage{xurl}}{} 
\hypersetup{
  pdftitle={Private Generative Bootstrap via Blocking},
  pdfauthor={Jinwon Sohn, Veronika Rockova},
  pdfkeywords={Amortized inference; Differential privacy; Loss-likelihood inference; Nonparametric inference; Quantile regression; Uncertainty quantification},
  colorlinks=true,
  linkcolor={blue},
  filecolor={Maroon},
  citecolor={Blue},
  urlcolor={Blue},
  pdfcreator={LaTeX via pandoc}}

\newcommand{\anon}{1}

\usepackage{amsthm}
\usepackage{amsfonts}
\usepackage{mathtools}
\usepackage{multirow}
\usepackage{algorithm}
\usepackage{algpseudocode}
\usepackage{placeins}
\usepackage{bibunits}   
\defaultbibliographystyle{agsm}
\allowdisplaybreaks

\graphicspath{{images/}}
\newtheorem{theorem}{Theorem}[section]
\newtheorem{corollary}{Corollary}[section]
\newtheorem{proposition}{Proposition}[section]
\newtheorem{remark}{Remark}[section]

\newtheorem{lemma}{Lemma}[section]

\newtheorem{definition}{Definition}[section]

\DeclareMathOperator*{\argmin}{arg\,min}
\DeclareMathOperator{\Var}{Var}
\DeclareMathOperator{\Cov}{Cov}
\DeclareMathOperator{\diam}{diam}
\newcommand{\E}{{\mathbb E}}
\newcommand{\R}{{\mathbb R}}

\renewcommand{\P}{{\mathbb P}}

\newcommand{\Diri}{\text{Dirichlet}}
\newcommand{\sub}{\text{Block}}
\newcommand{\Beta}{\text{Beta}}
\newcommand{\DP}{{\cal DP}}
\newcommand{\bX}{{\bf X}}

\newcommand{\bu}{{\bf u}}
\newcommand{\bw}{{\bf w}}
\newcommand{\bc}{{\bf c}}

\newcommand{\op}{\text{op}}

\newcommand{\tr}{\text{Tr}}

\newcommand{\cM}{\mathcal{M}}

\newcommand{\bomega}{\boldsymbol{\omega}}

\begin{document}
\begin{bibunit}[agsm]   

\def\spacingset#1{\renewcommand{\baselinestretch}%
{#1}\small\normalsize} \spacingset{1}


\if1\anon
{
  \title{\bf Private Generative Bootstrap via Blocking}
  \author{Jinwon Sohn\thanks{Corresponding author: \texttt{Jinwon.Sohn@ChicagoBooth.edu}.}\hspace{.2cm}\\
    Booth School of Business, University of Chicago\\
    and \\
    Veronika Ro\v{c}kov\'a\\
    Booth School of Business, University of Chicago}
  \maketitle
} \fi

\if0\anon
{
  \bigskip\bigskip\bigskip
  \begin{center}
    {\LARGE\bf Private Generative Bootstrap via Blocking}
  \end{center}
  \medskip
} \fi

\bigskip
\begin{abstract}
With  AI systems gaining more access to individuals' information, it is important to protect privacy when reporting statistical answers. Equally important is to privatize the reporting of uncertainty in such answers. To this end, we adopt a Bayesian likelihood-free framework and make simulation from the posterior private. In particular, we propose a new private instantiation of the Bayesian bootstrap using a blocking strategy. Rather than assigning idiosyncratic random weights to each individual, we randomly group individuals and assign a single weight to each group. By concealing individuals' contributions within a group, we fortify differential privacy gates. We harness amortized inference that decouples private learning from posterior sampling. A push-forward map from observation weights to posterior samples is learned privately by adding calibrated noise during training. Subsequent posterior draws require no additional privacy and computation budget. We call the resulting method the Private Generative Bayesian Bootstrap (PGBB). We establish a differential privacy guarantee, analyze convergence to the non-private blocked-bootstrap target, and quantify the discrepancy between the ordinary and blocked Bayesian-bootstrap posteriors. In addition, we derive data-free tuning of the block Dirichlet concentration parameter that restores posterior dispersion asymptotically. We also show a single fit of PGBB can support a family of loss-based decision rules simultaneously without additional privacy cost. In simulations and in applications to U.S. Census returns to schooling and U.S. natality birthweight quantiles, PGBB gives competitive private uncertainty quantification and improves over private Bayesian alternatives that require a specified data-generating model in common settings.
\end{abstract}

\noindent%
\vfill

\newpage
\spacingset{1.8} 
\section{Introduction}

Protecting individuals' privacy while reporting statistical answers is an important consideration in modern data analysis. Differential privacy (DP), proposed by \cite{dwor:etal:06}, has become a standard mathematical formalism for such guarantees. This framework controls the effect of any single record on the released output even against adversaries with auxiliary information. Beyond point summaries, a private analysis must also quantify how uncertain those summaries are, all the more so because the injected privacy noise is itself an additional source of variability. We call the task of reporting statistically valid uncertainty under a formal privacy guarantee \emph{private uncertainty quantification} (private UQ), a goal targeted by a growing literature \citep{wang:etal:15,ferr:etal:22,avel:etal:23,awan:wang:25,wang:etal:25}. Private UQ nonetheless remains challenging when one wishes to report statistically and practically reasonable uncertainty estimates without committing to a fully specified likelihood or data-generating process.

There have been various approaches for private UQ. Related to our work, a Bayesian formulation is natural for UQ because the posterior directly encodes uncertainty about the target, and a substantial literature has studied Bayesian inference under DP. Categorized mainly into three branches are privatizing posterior sampling itself \citep{wang:etal:15,dimi:etal:17,jews:etal:23}, perturbing MCMC transitions or gradients \citep{heik:etal:19,bai:etal:19,zhan:zhan:23}, and releasing noisy summary statistics for subsequent Bayesian updating such as Gibbs sampling \citep{bern:shel:18} or simulation-based draws \citep{gong:22,ju:etal:22}. However, private UQ under standard Bayesian inference is compelling only when a working likelihood matches an underlying data-generating model. It is well-known that, under misspecification, the posterior may concentrate around a pseudo-true parameter and report credible sets calibrated to the wrong model \citep{mull:13}. While the nonparametric private bootstrap \citep{wang:etal:25} avoids likelihood specification, it repeatedly accesses the data as many times as needed to collect private bootstrap estimates. This would require nontrivial privacy budgets in practice when numerous bootstrap samples are required. For these reasons, this work aims to develop a private Bayesian bootstrap method that directly overcomes the limitations of existing branches. 

As a backbone inferential framework, we adopt a decision-theoretic perspective that defines the target as a functional of the population distribution. Namely, for $X\sim F$ and a loss $l(\theta,x):\Theta\times{\cal X}\to\mathbb R^+$, the inferential target is $\theta(F):=\argmin_{\theta\in\Theta}\int l(\theta,x)\,dF(x)$, so the parameter is a decision rule rather than a component of a parametric model \citep{lydd:etal:19,fong:etal:19}. Uncertainty about $\theta(F)$ can be propagated from a nonparametric posterior on $F$ through the bootstrap. For instance, the posterior bootstrap of \cite{fong:etal:19} places a Dirichlet-process prior $F\sim {\cal DP}(\alpha,F_{\text{base}})$ on the data $\bX_n=(X_1,\dots,X_n)$ and repeatedly re-solves the weighted optimization to obtain draws of $\theta(\bw)$ for $\bw\sim\Diri(\alpha+n,\frac{\alpha}{\alpha+n}F_{\text{base}}+\frac{n}{n+\alpha}\sum_{i=1}^n\delta_{X_i})$, reducing to Rubin's Bayesian bootstrap \citep{rubi:81} when $\alpha=0$. The bootstrap is thus a natural vehicle that avoids specifying a working likelihood or data-generating process, toward private UQ nonparametrically. A recent advancement of Bayesian bootstrap  \citep{nie:rock:23,shin:etal:24} finds a generative map between $\bw$ and $\theta(\bw)$, known to enjoy computational scalability and produce effective i.i.d. posterior samples by construction.

Addressing the major limitations of existing work, this work suggests finding a private generative map under a block bootstrap scheme. To make this concrete, recall that the Bayesian bootstrap \citep{rubi:81} draws random weights $\bw=(w_1,\dots,w_n)\sim\Diri(\mathbf 1_n)$ over the $n$ records and reports $\theta(\bw)$. Our construction, however, begins to replace these per-record Dirichlet weights by block-level Dirichlet weights $\bu=(u_1,\dots,u_S)\sim\Diri(\gamma\mathbf 1_S)$ where $n$ individuals are randomly partitioned into $S$ equal-sized blocks. In this design, an individual contribution to the optimization reduces from 1 to $1/n_S$ in the worst case by concealing an individual into a subgroup, and therefore the block scheme provides a better privacy-preserving environment for devising an efficient private mechanism than the original scheme. Then, a proposed noisy training finds an amortized map which is the push-forward map from the blocked Dirichlet weights onto the block bootstrap posterior. We call the resulting map the Private Generative Bayesian Bootstrap (PGBB). Once PGBB is trained, any number of posterior draws is obtained by evaluating it at fresh weights, which is post-processing and spends no additional budget, in contrast to private MCMC \citep{heik:etal:19,zhan:zhan:23} and private bootstrap \citep{wang:etal:25}, whose budget grows with the number of draws to be released. Besides, PGBB can carry out private UQ for a range of decision rules without paying more privacy budgets, through joint amortized learning proposed in \cite{shin:etal:24}, which is an inherent advantage of PGBB. Table~\ref{tab:method-comparison} in SM contrasts PGBB with representative private inference methods. To our knowledge, this is the first work to harness the block bootstrap structure for private UQ. 

The blocked construction for privacy introduces two statistical questions. First, larger blocks improve privacy protection, but they also make the blocked posterior bootstrap under-dispersed relative to the ordinary posterior bootstrap. \emph{Thus, blocking creates a tension between privacy and posterior fidelity}. To address this, we calibrate the concentration parameter $\gamma$ of the blocked Dirichlet weights so that the posterior dispersion is restored asymptotically, and notably, it does not require additional privacy budget. Second, because the posterior sampler is learned through noisy private training rather than solved exactly, understanding the gap between the private generator and the blocked bootstrap target is statistically important for understanding how privacy noise affects posterior approximation. 

Our contributions are as follows. First, we develop the PGBB under a block-bootstrap scheme, with a DP guarantee that accounts for the randomness of block-level Dirichlet weights (Section~\ref{sec:private-GBB}). Second, we analyze the statistical error from noisy private training. Section~\ref{sec:theory} gives, to our knowledge, the first convergence guarantee for an amortized Bayesian bootstrap under noisy training, with the non-private generator as a special case. Third, we quantify and correct the posterior distortion caused by blocking. Section~\ref{sec:blocking} gives a finite-sample Wasserstein bound against the ordinary Bayesian bootstrap and introduces a concentration calibration that improves under-dispersion. Finally, Sections~\ref{sec:simul} and~\ref{sec:realdata} show that PGBB gives reliable private uncertainty under misspecification and that one trained generator can answer many quantile targets without additional privacy budget. Proofs and additional experiments are collected in the Supplementary Material (SM).

\section{Private Generative Bayesian Bootstrap}
\label{sec:private-GBB}

\subsection{Differential Privacy}

Bayesian inference may leak individuals' privacy. Imagine posterior summaries of a target rule $\theta$ are reported. An adversary may attempt to identify someone in the report or their private data record for potential exploitation by contrasting the institution's statistics (or decision) with their own posterior summaries obtained from auxiliary information. In such a case, differential privacy (DP) shall ask the reported posterior quantities to satisfy an indistinguishability condition defined as follows.  
\begin{definition}[\cite{dwor:etal:06b}]
\label{def:dp}
    Let $\bX_n=\{X_i\}_{i=1}^n$ and its neighboring database $\bX_n'=\{X_i'\}_{i=1}^n$ which differs in a single record, i.e., $X_i=X_i'$ holds for all $i\neq j$ except for a certain $j$. For fixed  $\epsilon >0$ and $\delta\in[0,1]$, a randomized mechanism $\cM$ is called $(\epsilon,\delta)$-DP if 
    \begin{align}
    \label{eqn:approx-dp}
        P(\cM(\bX_n)\in {\cal B})\le e^{\epsilon}P(\cM(\bX_n')\in {\cal B})+\delta,
    \end{align}
    for any neighboring databases $\bX_n, \bX_n'$ and all Borel measurable sets ${\cal B}$. 
\end{definition}
\noindent In words, the two posterior quantities $\cM(\bX_n)$ and $\cM(\bX_n')$, differing at only one arbitrary individual, are statistically indistinguishable up to $e^{\epsilon}$ tilting and $\delta$ probability shift. We say $(\epsilon,\delta)$ is the privacy budget. As $\epsilon,\delta$ get smaller, $\cM$ has stronger privacy protection. A $(\epsilon,\delta)$-DP mechanism satisfies the following important property. 
\begin{proposition}[Post processing, \cite{dwor:etal:14}]
\label{prop:post-processing}
Let $\cM$ be an $(\epsilon,\delta)$-DP mechanism. For any measurable map $A$, the composed mechanism
$A\circ \cM$ is also $(\epsilon,\delta)$-DP.
\end{proposition}
In our context, therefore, if the amortized generator ${\cM}$ is $(\epsilon,\delta)$-DP, then any downstream statistics constructed from the private sampler inherit the same $(\epsilon,\delta)$ privacy level. 

\subsection{Nonparametric Bayesian Bootstrap and Blocking}
\label{sec:npbb}

As established in the Introduction, our inferential target is the decision rule $\theta(F)=\argmin_{\theta\in\Theta}\int l(\theta,x)\,dF(x)$ for $X\sim F$, and we quantify uncertainty about it through the Bayesian bootstrap \citep{rubi:81}, drawing weights $\bw=(w_1,\dots,w_n)\sim\Diri(\mathbf 1_n)$ and reports $\theta^\star_{\Diri}:=\theta^\star(\bw):=\argmin_{\theta}\sum_{i=1}^n w_i\,l(\theta,X_i)$ for observations $X_1,\dots,X_n \sim F$ i.i.d. This is the $\alpha=0$ instance of the posterior bootstrap of \cite{fong:etal:19}, whose Dirichlet-process prior $F\sim\DP(\alpha,F_{\text{base}})$ additionally mixes the data with pseudo-samples from a prior $F_{\text{base}}$. We work throughout in the fully data-driven regime $\alpha=0$ and defer the extension to $\alpha>0$ to Discussion. We next recall the amortized form of this bootstrap and then introduce blocking, the foundation of our privacy mechanism.

\subsubsection{Generative Bootstrap via Amortization}

With remarkable computational efficiency, a recent amortized bootstrap method aims to find a generative map from the bootstrap weights to the estimates of the decision rule $\theta(F)$ \citep{shin:etal:21,nie:rock:23,shin:etal:24}. Founded on the aforementioned nonparametric Bayesian framework, the amortized sampler is defined as 
\begin{align}
\label{eqn:nbb}
    G^{\star}:=\argmin_{G\in {\cal G}} \E_{\bw \sim \Diri({\bf 1}_n)}\Big[\sum_{i=1}^n w_il(G(\bw),X_i)\Big], \quad G(\bw):\mathbb{R}^{n}\rightarrow \mathbb{R}^{d_{\theta}}, 
\end{align}
where $l(\cdot,\cdot):\mathbb{R}^{d_{\theta}} \times {\cal X}\rightarrow \mathbb{R}^+$ is a non-negative real-valued function. Intuitively, an ideal generator that attains the pointwise weighted minimum, $G^\star(\bw):=\argmin_{\theta}\sum_i w_i l(\theta,X_i)$ for almost every $\bw$, satisfies $G^\star(\bw)=\theta^\star_{\Diri}(\bw)$ for a fresh draw $\bw$ \citep{shin:etal:24}. Thus the ideal $G^\star$ serves as the push-forward map from $\Diri({\bf 1}_n)$ to the posterior distribution of $\theta^{\star}_{\Diri}$. In practice, a class ${\cal G}_{\text{NN}}=\{G_{\phi}:\phi\in \Phi\}$ of neural networks is employed where $\Phi$ is the space of neural-net's parameters. A feed-forward network $G_{\phi}(x)=W_V h_{V-1}(W_{V-1}h_{V-2}(\cdots W_1x))$ is a common example where $\phi=(W_1,\dots,W_V)$ and $h_1,\dots,h_{V-1}$ are nonlinear elementwise activation functions. The optimization \eqref{eqn:nbb} then simplifies to finding $G^{\star}_{\phi}(\bw):=\argmin_{\phi\in \Phi} \E_{\bw \sim \Diri({\bf 1}_n)}[\sum_{i=1}^n w_il(G_{\phi}(\bw),X_i)]$. Thus, for a fresh draw $\bw\sim \Diri({\bf 1}_n)$, $G^{\star}_{\phi}(\bw)$ approximates $G^{\star}(\bw)$ and accordingly $\theta^{\star}_{\Diri}(\bw)$ as well. In particular, it is well-known that, if $G^{\star}$ is continuous, then $G_{\phi}$ is a universal approximator; there exists $G_{\phi}\in {\cal G}_{\text{NN}}$ such that $G_{\phi}$ gets closer to $G^{\star}$ as $G_{\phi}$ gets deeper and wider \citep{hani:19}. Posterior bootstrap samples are obtained by repeatedly drawing $\bw$ and evaluating $G^\star_{\phi}(\bw)$, and this work harnesses this decoupled inference from learning the stage for the privacy purpose remarked below.
\begin{remark}
Due to the decoupled phases between finding a posterior map and generating posterior samples from it, once the generative map is $(\epsilon,\delta)$-DP, the generation phase is also $(\epsilon,\delta)$-DP by the post-processing property (Proposition~\ref{prop:post-processing}). 
\end{remark}

The neural-net generator $G^\star_{\phi}$ is estimated by the iterative update of $\phi$ as follows. After $\phi^{(0)}$ is initialized, at iteration $t$, we \textit{(i)} sample $\bw_t^{(k)}\sim\Diri({\bf 1}_n)$ and construct an independent minibatch $B_t^{(k)}$ for $k=1,\dots,K$, and \textit{(ii)} update
$\phi^{(t+1)}=\phi^{(t)}-\eta\nabla g_t$ where 
\begin{align}
\label{eqn:opt-grad}
    \nabla g_t=\frac{1}{K}\sum_{k=1}^K\sum_{i\in B_t^{(k)}} w_{i,t}^{(k)} \nabla l(g^{(k)}_t,X_i)|_{g^{(k)}_t=G_{\phi^{(t)}}(\bw^{(k)}_t)} \cdot \nabla_{\phi}G_{\phi^{(t)}}(\bw^{(k)}_t)
\end{align}
for $t=1,\dots,T$. When $n$ is huge, the minibatch scheme brings substantial scalability. The averaging across $1,\dots,K$ stands for the Monte Carlo estimate of the gradient since the target optimization~\eqref{eqn:nbb} is defined as the average on the Dirichlet distribution. The last iterate $G_{\phi^{(T)}}$ is usually employed as an estimate of $G^{\star}_{\phi}$.

\subsubsection{Blocked Bootstrap to Curb Individual's Contribution}
\label{sec:subgroup}

Denote by $\pi:[n]\to[S]$ a blocking process, where $[N]=\{1,\dots,N\}$. It partitions the $n$ observations into $S$ disjoint groups $I_1,\dots,I_S$ of equal size $n_S:=n/S$ by sampling without replacement. Assume $n_S\in\mathbb N$ for simplicity. The blocked bootstrap collects the solutions of $\theta^{\star}_{\sub}:=\theta^{\star}(\bar\bw):=\arg\min_{\theta}\sum_{i=1}^n \bar w_i\,l(\theta,X_i)$, where $\bar\bw=(\bar w_1,\dots,\bar w_n)$ and $\bar w_i=u_{\pi(i)}/n_S$, by drawing $\bu:=(u_1,\dots,u_S)\sim\Diri(\gamma\mathbf 1_S)$ for $\gamma>0$. The blocking scheme underlies $\sum_{i=1}^n\bar w_i\delta_{X_i}=\sum_{s=1}^S u_s\hat F_s$ with $\hat F_s:=\tfrac1{n_S}\sum_{i\in I_s}\delta_{X_i}$ for a given $\pi$. In other words, it can be a structured prior on $F$ that puts a random Dirichlet mass on each block and an equal deterministic weight within it. Thus, the block bootstrap can be an {exact} posterior sampler when the estimand depends on the data only through the block aggregates $\{\hat F_s\}$, and an approximation otherwise.

In this work, the block bootstrap with the partition $\pi$ randomized independently to $(\bu,\bX_n)$ offers two foundational merits in designing our privacy mechanisms. Specifically, we use the uniform random partition $\pi=\pi_\tau(i)=\lceil \tau(i)\,S/n\rceil$, where $\tau$ is a uniformly random permutation of $[n]$ independent of $(\bu,\bX_n)$. First, whereas the original bootstrap with $\bw\sim \Diri({\bf 1}_n)$ lets an individual contribute up to $1$, the block scheme caps each record's worst-case influence at most $\bar w_i=u_{\pi(i)}/n_S\leq 1/n_S$. In other words, concealing an individual under the block scheme becomes more efficient than the original one in designing a DP mechanism. Also, it advocates using equal-sized blocks since an unbalanced partition scheme may increase the risk of revealing an individual's worst-case contribution. Second, drawing $\pi$ as a data-independent random partition makes the guarantee agnostic to the realized blocks. If $\cM_{\pi}$ is $(\epsilon,\delta)$-DP  mechanism conditional on every fixed partition $\pi$, marginalizing preserves the same DP guarantee, i.e., $P(\cM(D)\in{\cal B})=\int P(\cM_\pi(D)\in{\cal B})\,d\nu(\pi)\le e^{\epsilon}P(\cM(D')\in{\cal B})+\delta$, so privacy holds regardless of which blocks are realized, e.g., even if a block, by chance, mirrors some structure in the data.

Applying the block bootstrap for privacy purposes is novel to our knowledge. In statistics, a block-bootstrap procedure has been used for diverse reasons: time-series bootstraps \citep{kuns:89} to capture a time-dependency effect across data instances and amortized learning \citep{shin:etal:21} to improve scalability under a large sample regime. While grouping data for devising a privacy mechanism is not new in a high-level view, e.g., sample-and-aggregate for point estimation \citep{niss:etal:07}, teacher ensembles in supervised learning \citep{pape:etal:17}, we first utilize the block-bootstrap scheme for private uncertainty quantification.

\subsection{Noisy Training}
\label{sec:noisy-training}

Blocking alone is not enough to achieve privacy. Capping each record's weight at $1/n_S$ limits any single record's {influence}, but it does not by itself make a released bootstrap optimizer DP. As a counterexample, consider $X_i=0$ for $i<n$ and $X_n=M$ for a large $M$. Even though $X_n$ carries only its group's share of the weight, the blocked bootstrap solution $\theta^{\star}(\bar\bw)$, and also $\theta^{\star}(\bw)$, still shifts detectably with $M$, so releasing it with no further randomization can reveal whether the $n$th record is present, violating the indistinguishability in Definition~\ref{def:dp}. Formal $(\epsilon,\delta)$-DP therefore requires an additional step, and we inject Gaussian noise during training. The role of blocking thus is seen to reduce an individual's contribution, so that less noise is needed to conceal the individual during training.

To make the amortized learning phase $(\epsilon,\delta)$-differentially private, we suggest perturbing the update procedure of \eqref{eqn:opt-grad} with Gaussian noise. At the $t$th iteration, for a given partition rule $\pi$, the gradient descent takes 
\begin{align}
\label{eqn:perturbed-grad}
    \frac{1}{K}\sum_{k=1}^K \Big(\sum_{i\in B_t^{(k)}} \bar w_{i,t}^{(k)} \nabla l(g^{(k)}_t,X_i) +Z_{t}^{(k)}\Big) \cdot \nabla_{\phi}G_{\phi^{(t)}}^{\text{DP}}(\bu_{t}^{(k)}), \quad g^{(k)}_t=G_{\phi^{(t)}}^{\text{DP}}(\bu_{t}^{(k)}),
\end{align}
where $Z_{t}^{(k)}\sim {\rm N}(0,\sigma^2{\bf I})$ and $\bu_{t}^{(k)}\sim \Diri(\gamma {\bf 1}_S)$, and each index in $B_t^{(k)}$ is randomly picked by $q$ probability for every $t,k$. The superscript of $G_{\phi}^{\text{DP}}$ indicates $G_{\phi}$ is paired with the noisy training \eqref{eqn:perturbed-grad}. The Gaussian noise is added to the weighted output-gradient statistic before multiplication by the generator's gradient because the generator does not take data points directly. For this reason, this mechanism may be seen as structurally less vulnerable to the Gaussian noise when compared to a common setup in the literature of standard private stochastic gradient descent \citep[DP-SGD,][]{abad:etal:16}. Algorithm~\ref{alg:bayesian_bootstrap} in SM~\ref{app:alg} describes the updating scheme in detail. Note the final $T^{\text{th}}$ iterate of the gradient-based update by \eqref{eqn:perturbed-grad} stands for $G_{\phi^{(T)}}^{\text{DP}}$, and next section shows $G_{\phi^{(T)}}^{\text{DP}}$ satisfies $(\epsilon,\delta)$-DP. More precisely, the collection $\{\phi^{(t)}\}_{t=1}^T$ from $\bX_n$ is $\cM(\bX_n)$, so evaluation of $\bu$ through $G_{\phi^{(T)}}^{\text{DP}}$ is DP by the post-processing. 

The block structure contributes substantially to enhancing the privacy protection over non-block noisy training. The Dirichlet weights mainly influence the randomness in \eqref{eqn:perturbed-grad} on top of the Gaussian privacy noise. Addressing the Dirichlet randomness is particularly important. Intuitively, an individual's influence in \eqref{eqn:perturbed-grad} can be quantified by $\sup_{D, D',g}\|\sum_{i=1}^n \bar w_i \nabla l(g,x_i) -\sum_{i=1}^n \bar w_i \nabla l(g,x_i') \|=\sup_{x,x',g} \|\bar w_i(\nabla l(g,x)-\nabla l(g,x')) \|$. If $\|\nabla l(\cdot,\cdot)\|\leq C$ is uniformly bounded, the magnitude is bounded by $2\bar w_iC$. Although the block strategy limits $\bar w_i\le 1/n_S$, quantifying an individual's influence by this deterministic upper bound, as a nominal DP-SGD application, may still be too conservative. Under the blocked Dirichlet weights, the marginal distribution satisfies $n_S\bar w_i\sim \mathrm{Beta}(\gamma,(S-1)\gamma)$, and most of its distributional mass is away from the endpoint $1$. Accounting for the Dirichlet randomness, therefore, can substantially reduce the individual contribution measured in privacy accounting. 

\subsection{Certified Privacy for Noisy Training with Block}
\label{sec:formal-dp}
This section certifies that the noisy training of Section~\ref{sec:noisy-training} produces an amortized generator satisfying $(\epsilon,\delta)$-DP, and characterizes the privacy gain from the Dirichlet randomness. For better readability, we first present the asymptotic analysis on the size of the privacy noise $\sigma$ in Theorem~\ref{thm:asymp-dp}, and then describe the formal $(\epsilon,\delta)$-DP guarantee in Proposition~\ref{prop:formal-dp}. 
\begin{theorem}
\label{thm:asymp-dp}
Let $G^{\text{DP}}_{\phi^{(T)}}$ be the $T^{\text{th}}$ iterate of \eqref{eqn:perturbed-grad}. Suppose $\sup_{x,g}\|\nabla l(g,x)\|\leq C$ for some $C>0$, and fix $\epsilon>0$, $T,K,S\in\mathbb N$, $\gamma>0$, $q\in(0,1)$, and a target $\delta\in(0,1)$. Set $\kappa_{S,\gamma}:=\tfrac{\gamma+1}{S(S\gamma+1)}$. Assume $\sqrt{\tfrac{\log(1/\delta)}{2TK\kappa_{S,\gamma}}}\ll 1$ and $\tfrac{q\log(1/\delta)}{(1-q)\epsilon}=O(1)$. At leading order, a sufficient condition for $G^{\text{DP}}_{\phi^{(T)}}$ to satisfy $(\epsilon,\delta)$-DP is
\begin{align}\label{eqn:thm-asymp-sigma-box}
    \sigma&\gtrsim  \frac{Cq}{n_S} 
    \frac{\sqrt{TK \log(1/\delta) \kappa_{S,\gamma}}}{\epsilon},
    \qquad
    \frac{2TKq^{2}C^{2}\kappa_{S,\gamma}}{\sigma^{2}n_S^{2}}
    \le \log(1/\delta),
    \qquad
    \frac{C}{\sigma n_S}=o(1).
\end{align}
\end{theorem}
Equation~\eqref{eqn:thm-asymp-sigma-box} shows how the block size $n_S$ enters the privacy condition. The direct role of a larger $n_S$ is in the third term $C/(\sigma n_S)=o(1)$, while a smaller $S$ increases $\kappa_{S,\gamma}$ and thereby makes the external regime $TK\kappa_{S,\gamma}\gg\log(1/\delta)$ easier to satisfy. The noise scale $\sigma$ can be negligible for a large-sample regime, but still increases by $\sqrt{TK}$. For $\sqrt{T}$, the 1/2-rate aligns with the standard DP-SGD literature \citep{abad:etal:16}. Although $\sqrt K$ appears to increase the required noise scale, its effect on the update is mitigated by the averaging over $K$ Monte Carlo draws. This point is made more explicit in the convergence analysis. The following remark characterizes the privacy amplification effect by encompassing the Dirichlet randomness, in contrast to the standard private SGD. SM~\ref{app:privacy-amplification} writes the derivation.
\begin{remark}
\label{rem:amplification}
    For the Dirichlet-aware accountant, the leading noise scale is of order $Cq\sqrt{TK\,\log(1/\delta)\kappa_{S,\gamma}}/(n_S\epsilon)$ where $\kappa_{S,\gamma}=\tfrac{\gamma+1}{S(S\gamma+1)}<1$ for $S\geq 2$. If one ignores the Dirichlet distribution and uses only the deterministic bound $\bar w\le 1/n_S$, the corresponding worst-case scale is of order $Cq\sqrt{TK\log(1/\delta)}/(n_S\epsilon)$. Thus the ratio of the two leading scales is $\sqrt{\kappa_{S,\gamma}}$ when the validity conditions for each method are met.
\end{remark}
The assumption $\sup_{x,g}\|\nabla l(g,x)\|\le C$ is common in DP statistical estimation. A representative example is a private robust M-estimation \citep{avel:etal:23}. A wide range of statistical problems satisfies the uniform bound under realistic situations. For example, for the absolute loss $l(g;(y,x))=|y-x^\top g|$, a subgradient with respect to $g$ is $-\mathrm{sign}(y-x^\top g)x$, whose norm is bounded when $x$ lies in a bounded domain.  Likewise, on the bounded input domain, the pinball loss for quantile regression \citep{koen:bass:78} has the bounded derivative, and the Huber loss for robust regression does as well \citep{hube:92}. If the condition is not met, one can enforce it by clipping only the multiplier $\nabla l(X_i,g)$ in \eqref{eqn:perturbed-grad} before multiplying by $\nabla_\phi G_\phi^{\rm DP}(\bu)$. 

Proposition~\ref{prop:formal-dp} provides the formal $(\epsilon,\delta)$-DP guarantee of the noisy training by \eqref{eqn:perturbed-grad}. One subtle point of our analysis is that neighboring databases $\bX_n,\bX_n'$ are given the same Dirichlet distributions. In the DP-SGD literature, \cite{birr:etal:24} presented the analysis of the noisy gradient with the Gaussian noise and Poisson subsampling under the replace-one neighboring scheme. Our analysis extends it to accommodate the Dirichlet randomness and the Monte Carlo averaging. 
\begin{proposition}
\label{prop:formal-dp}
Suppose $\sup_{x,g}\|\nabla l(g,x)\|\leq C$ for some $C>0$. Fix $\epsilon>0$, $\sigma^{2}>0$, $T,K,S\in\mathbb N$, $\gamma>0$ and $q\in(0,1)$. With pre-specified truncation orders $J\in\mathbb N$ and $m\ge 3$, define, for integer $\lambda\ge2$, $M_{J,\lambda}:= 2\lambda(\lambda-1)\sum_{j=0}^{J}\tfrac{(C/\sigma)^{4j+2}}{(2j+1)!}\cdot\tfrac{(\gamma)_{4j+2}}{(S\gamma)_{4j+2} n_S^{4j+2}}$, and $R_{J,\lambda}:= 2\lambda(\lambda-1)[\sinh (C^{2}/(\sigma^{2}n_S^{2}))-\sum_{j=0}^{J}\tfrac{(C/(\sigma n_S))^{4j+2}}{(2j+1)!}]$
respectively, where $(x)_k:=x(x+1)\cdots(x+k-1)$. Let 
$A_{J,\lambda,m}(q):=\log[1+q^{2}(M_{J,\lambda}+R_{J,\lambda})+\mathcal{R}_{J,\lambda,m}(q)]$ where $\mathcal{R}_{J,\lambda,m}(q)$ appears in \eqref{eqn:bir-Rcal-bar} of the SM. Then, $G^{\text{DP}}_{\phi^{(T)}}$ satisfies $(\epsilon,\delta_{\star}(\epsilon))$-DP with $\delta_{\star}(\epsilon):= \inf_{\lambda\in\mathbb N, \lambda\ge 2}\exp \bigl\{ TK A_{J,\lambda,m}(q) - (\lambda-1)\epsilon \bigr\}$.
\end{proposition}
Because the analysis tracks consumed privacy cost across $T$, it is usually called privacy accounting. The leading term $M_{J,\lambda}$ specifically reflects Dirichlet randomness via the moment $\mathbb E[\bar w^r]=(\gamma)_r/((S\gamma)_r n_S^r)$, in particular with $r=4j+2$. $R_{J,\lambda}$ accounts for the remainder of the $J$th order approximation $M_{J,\lambda}$ through the deterministic bound of $\bar w_i\leq 1/n_S$. The same technique applies to $\mathcal{R}_{J,\lambda,m}$. In this way, the formal privacy guarantee incorporates the randomness of the Dirichlet weights, leveraging the deterministic worst-case contribution. The choice of $J,m$ depends on the numerical stability, and we set $J=32$ and $m=5$ by default. Figure~\ref{fig:budget} draws the pattern of the privacy curve in various cases. Notably, smaller $S$ brings stronger protection. This shows the distinctive privacy-side advantage of using $S\ll n$, beyond computational scalability.
\begin{figure}[ht!]
    \centering
    \includegraphics[width=1.0\linewidth]{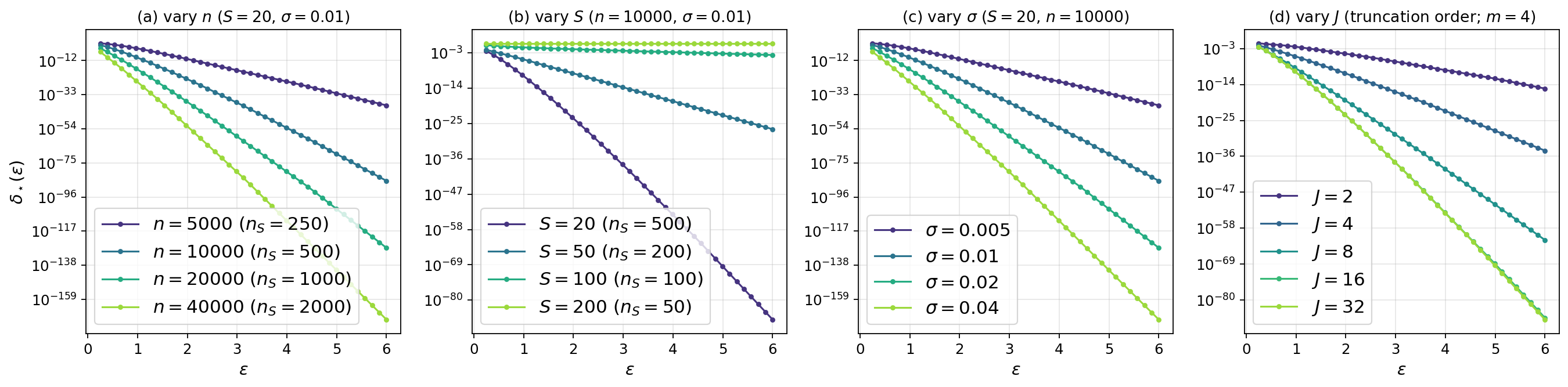}
    \caption{ $(\varepsilon,\delta_\star(\varepsilon))$-DP curve in Proposition~\ref{prop:formal-dp}, varying (a) the sample size $n$, (b) the block count $S$, (c) the noise scale $\sigma$, and (d) the Beta-moment truncation order $J$. The four panels illustrate the anticipated monotone dependence on the privacy-related parameters.}
    \label{fig:budget}
\end{figure}

\subsection{Joint Learning across Decision Rules}
\label{sec:joint-learning}
This section introduces a distinctive feature of the amortized sampler.  A single generator can privatize the posterior of a range of decision rules $\theta_{\tau}$ for $\tau\sim P_{\tau}$ simultaneously under a given $(\epsilon,\delta)$ privacy budget. More specifically, the amortized learning need not stick to a single decision rule. Let $\tau\in\mathcal T$ index a family of losses $l_\tau(\theta,x)$ with targets $\theta_\tau(F):=\argmin_{\theta\in\Theta}\int l_\tau(\theta,x)\,dF(x)$. The index may encode a tuning parameter, a loss, or a scientific target, and in our quantile-regression application it is the quantile level. Following the multi-purpose sampler of \cite{shin:etal:24}, we feed both the bootstrap weight and the index to one generator, and solve $\min_{G\in\mathcal G}\ \E_{\bu\sim\Diri(\gamma\mathbf 1_S),\,\tau\sim P_\tau} \Big[\textstyle\sum_{i=1}^n \bar w_i\, l_\tau(G(\bu,\tau),X_i)\Big]$ via the noisy training, extending of \eqref{eqn:perturbed-grad}, 
\begin{align*}
    \frac{1}{K}\sum_{k=1}^K \Big(\sum_{i\in B_t^{(k)}} \bar w_{i,t}^{(k)} \nabla l(g_{\tau,t}^{(k)},X_i) +Z_{t}^{(k)}\Big) \cdot \nabla_{\phi}G_{\phi}^{\text{DP}}(\bu_{t}^{(k)},\tau_t^{(k)}), \quad g_{\tau,t}^{(k)}=G_{\phi}^{\text{DP}}(\bu_{t}^{(k)},\tau_t^{(k)}),
\end{align*}
where $\tau_t^{(k)} \sim P_\tau$ is independently sampled regardless of data. For DP guarantee, it suffices to set the uniform bound $\sup_{x,g_\tau}\|\nabla l(g_{\tau},x)\|\leq C$. Once the private generator $G^{\rm DP}_{\phi^{(T)}}(\bu,\tau)$ is trained, evaluating it at any $\tau$ is post-processing, so draws and summaries across the entire $\tau\sim P_{\tau}$ are tied to $(\epsilon,\delta)$ budget. This brings a concrete privacy gain. In stark contrast, reporting private statistics over a grid of $\tau$ by running a separate private method per $\tau$ accesses the data repeatedly, so by composition, the budget grows according to the grid size. To our knowledge, this joint structure has not been developed (or explored) in existing private UQ literature.

\section{Convergence Analysis}
\label{sec:theory}

Having established the privacy guarantee, this section studies how the privacy noise affects the training of the neural-net-based generator $G^{\text{DP}}_{\phi^{(T)}}$. We consider a standard feed-forward neural network with a bounded final activation. Set $G_\phi(\bu)=B_G h_V(\tilde G_\phi(\bu))$, where $\tilde G_\phi(\bu):=W_V h_{V-1}(W_{V-1}h_{V-2}(\cdots W_1\bu\cdots))$ and $\phi=(W_1,\dots,W_V)$. $h_1,\dots,h_V$ are activation functions and suppose $h_V$ are bounded. Bias parameters are suppressed, since they can be represented by concatenating a constant $1$ to the input. Output-clipping radius $B_G>0$ can be chosen so that $[-B_G,B_G]^{d_\theta}\subseteq \Theta$. This structural design allows $G_{\phi}(\bu)\in \Theta$ for any $\phi$ and $\bu$. We impose $\phi\in \Phi:=\{\phi : \|W_v\|_{\op} \le B_v,\;v = 1,\dots,V\}$ and suppose $h_1,\dots,h_V$ are twice differentiable.
This covers a wide range of activation functions, e.g., tanh, shifted-sigmoid, GELU, etc. The bounded weight matrices are common in theoretical analysis when neural networks are involved \citep{golo:etal:18} and can also be readily satisfied via normalization or clipping techniques \citep{miya:etal:18}.

Our analysis imposes the following conditions. 
\begin{description}
\item[(A1)]
$X_1,\dots,X_n$ are i.i.d.\ from a distribution $F$ supported on a compact set $\mathcal X\subset\R^{d_X}$.

\item[(A2)]
$\Theta\subset\R^{d_\theta}$ is convex and compact, and the population minimizer $\theta_\infty^\star := \arg\min_{\theta\in\Theta}\E\,l(\theta,X)$ is unique and lies in $\mathrm{int}(\Theta)$.

\item[(A3)]
$l(x,\theta)$ is $C^3$ in $(x,\theta)$ on a neighborhood of $\mathcal X\times\Theta$ with uniformly bounded derivatives up to order three.

\item[(A4)]
$l(x,\theta)$ is convex in $\theta$ for each $x\in\mathcal X$.

\item[(A5)]
The population Hessian $H_\infty := \E[\nabla_\theta^2 l(\theta_\infty^\star,X)]$ satisfies $H_\infty \succeq \mu_\infty {\bf I}$ for some $\mu_\infty>0$, and $\Sigma_\infty := \E[\nabla l(\theta_\infty^\star,X)\nabla l(\theta_\infty^\star,X)^\top]$ denotes the population score second moment.
\end{description}
Together with the bounded network class, these assumptions imply uniform bounds on the generator Jacobian and Hessian. Define $L(\theta;\bw):=\sum_{i=1}^n w_i l(\theta,X_i)$ for any given weight vector $\bw$. Let $D_\Theta:=\diam(\Theta)$ and denote by $\mathcal{L}(\phi;\bX_n,\pi)=\E_{\bu\sim\Diri(\gamma\mathbf 1_S)}[L(G_\phi(\bu);\bar\bw(\bu))]$ the Dirichlet-averaged objective conditional on $\bX_n,\pi$ and $J_{\phi}(\bu)=\nabla_{\phi} G_\phi(\bu)$. Under (A1-5), since $\Phi$ is compact and $G_\phi(\bu)$ is a finite composition of smooth activations and bounded weight maps, the generator's Jacobian and parameter Hessian are uniformly bounded, $\beta_1 := \sup_{\phi\in\Phi,\,\bu}\|J_\phi(\bu)\|_{\op} <\infty$ and $\beta_2 := \sup_{\phi\in\Phi,\,\bu}\|\nabla_\phi^2 G_\phi(\bu)\|_{\op} <\infty$. Letting $H_{\bar \bw(\bu)}(\theta):=\nabla^2_{\theta} L(\theta;\bar \bw(\bu))$, a standard chain rule applied to $\nabla_\phi^2\mathcal{L}(\phi;\bX_n,\pi)=\E_\bu[J_\phi(\bu)^\top H_{\bar\bw(\bu)}(G_\phi(\bu))J_\phi(\bu)] +\E_\bu[\langle\nabla_\phi^2 G_\phi(\bu),\nabla_\theta L(G_\phi(\bu);\bar\bw(\bu))\rangle]$ via the dominated convergence theorem yields the closed-form smoothness constant $M_{\mathcal L}=\sup_{\bX_n,\pi,\phi\in\Phi}\|\nabla_\phi^2\mathcal{L}(\phi;\bX_n,\pi)\|_{\op}\;\le\;\beta_1^2\,K_2\;+\;\beta_2\,C$ where $K_2:=\sup_{x,g}\|\nabla^2 l(g,x)\|_{\op}$ and $\sup_{x,g}\|\nabla l(g,x)\|\leq C$. Hence we view $\mathcal{L}(\phi;\bX_n,\pi)$ as a $M_{\cal L}$-smooth function over a convex domain $\phi \in \Phi$ for any given $\bX_n,\pi$.

Our convergence analysis characterizes the discrepancy between the private amortized generator $G_{\phi^{(T)}}^{\text{DP}}(\bu)$ and the exact blocked-bootstrap-solution $\theta^{\star}_{\sub}$. Let $\bomega_T$ collect the algorithmic randomness used to obtain the $T^\text{th}$ iterate, including initialization, minibatch sampling, Dirichlet draws used during training, and Gaussian perturbations. 
\begin{theorem}
\label{thm:conv-gbb}
Suppose (A1-5) hold and $\phi_t\in \Phi$ holds. Let ${\cal L}^\star_{\Phi,\bX_n,\pi}:=\inf_{\phi\in\Phi}\mathcal L(\phi;\bX_n,\pi)$ and suppose that there exists $\mu_{\mathcal L}>0$ such that, for all $\phi\in\Phi$, $\|\nabla_\phi\mathcal L(\phi;\bX_n,\pi)\|^2 \ge 2\mu_{\mathcal L} (\mathcal L(\phi;\bX_n,\pi)-\mathcal L^\star_{\Phi,\bX_n,\pi})$ for every given $\bX_n,\pi$. For fixed $\gamma>0$, $q\in(0,1)$, $\eta\in(0,\,1/(qM_{\mathcal L})]$, the last iterate $G_{\phi^{(T)}}^{\text{DP}}$, for a fresh draw $\bu\sim \Diri(\gamma {\bf 1}_S)$, satisfies
\begin{align}
\label{eqn:thm-conv-G}
\E_{\bX_n,\pi,\bomega_T,\bu}
\bigl\|G^{\text{DP}}_{\phi^{(T)}}(\bu)-\theta^\star_\sub(\bar\bw(\bu))\bigr\|^{2}
\;\lesssim\; \mathcal B_T ,
\end{align}
where $\mathcal B_T
={(1-\eta q\mu_{\mathcal L})^{T}\,\Delta_0}+{\frac{M_{\mathcal L}\,\eta}{q\,\mu_{\mathcal L}}\,\Pi_{\rm total}}+{\varepsilon_{\rm app}}+{n^{-2}}$ with initial gap $\Delta_0:=\E_{\bX_n,\pi,\phi_0} [\mathcal L(\phi_0;\bX_n,\pi)-\mathcal L^\star_{\Phi,\bX_n,\pi}]$ and amortization gap $\varepsilon_{\rm app}:=\E_{\bX_n,\pi}[\mathcal L^\star_{\Phi,\bX_n,\pi} -\E_\bu\{L(\theta^\star_\sub(\bar\bw(\bu));\bar\bw(\bu))\}]$. It splits into an optimization transient, a noise floor, an approximation error and a remainder, respectively. In particular, 
\begin{align*}
\Pi_{\rm total}=
\underbrace{\frac{\beta_1^2 C^2 q(1-q)(\gamma+1)S}{K n(S\gamma+1)}}_{\text{Poisson subsampling}}
+\underbrace{\frac{q^2\beta_1^2 C^2}{K}}_{\text{Monte Carlo}}
+\underbrace{\frac{\beta_1^2 d_\theta \sigma^2}{K}}_{\text{privacy noise}} .
\end{align*}
\end{theorem}
 The first term is a geometric transient that contracts at the per-iteration rate $1-\eta q\mu_{\mathcal L}$, so the initial optimization gap decays with the training horizon $T$. The term $\varepsilon_{\rm app}$ is the generator-class approximation error, measuring the amortization gap between a single trained map and the pointwise bootstrap optima $\theta^\star_\sub(\bar\bw(\bu))$. This vanishes as the network class $\Phi$ becomes sufficiently expressive and $\Theta=[-B_G,B_G]^{d_\theta}$, as is natural under (A2).

The remaining term $\Pi_{\rm total}$ is the noise floor from the private update. It decomposes the effects of Poisson subsampling, $K$-fold Monte Carlo averaging of the Dirichlet draws, and Gaussian privacy noise. Increasing $K$ shrinks the Poisson and Monte Carlo summands, while the privacy term is effectively $K$-invariant after substituting the leading-order noise scale from Theorem~\ref{thm:asymp-dp}. In this upper bound, the privacy term is then of smaller order because of its $n^{-2}$ rate, although this should be read as an ordering within the bound rather than a tight comparison of the three randomness sources.

The existence of such $\mu_{\cal L}$ is the Polyak-{\L}ojasiewicz (PL) condition. This condition is weaker than strong convexity and does not require convexity. It is commonly used to prove convergence for nonconvex objectives and over-parametrized models such as neural networks \citep{kari:etal:16,liu:zhu:belk:22}. To our knowledge, this is the first convergence analysis for a neural-net-based amortized bootstrap generator. The theorem covers the non-private case by setting $\sigma=0$.

\section{Posterior Analysis of the Blocking Bootstrap}
\label{sec:blocking}

Previous sections show that the noisy training satisfies $(\epsilon,\delta)$-DP and that a larger block size can improve the privacy accounting. The statistical cost is that applying the blocked bootstrap to i.i.d. data may underestimate posterior uncertainty, since observations within a block share the same weight. This section studies this posterior distortion and how to correct it while keeping the privacy benefit.

\subsection{Posterior Gap Induced by Blocking}
Let $\theta_0^\star:=\arg\min_{\theta\in\Theta}L(\theta;n^{-1}\mathbf 1_n)$ be the empirical risk minimizer, where $L(\theta;\bw):=\sum_{i=1}^n w_i l(\theta,X_i)$. Recall that $\bu\sim\Diri(\gamma{\bf 1}_S)$, $\pi$ is a uniform random partition, and $\bar w_i=u_{\pi(i)}/n_S$. The blocked solution (blk-BB) is $\theta^\star_\sub(\bar\bw):=\arg\min_\theta L(\theta;\bar\bw)$, while the ordinary Bayesian-bootstrap solution (org-BB) is $\theta^\star_\Diri(\bw):=\arg\min_\theta L(\theta;\bw)$ with $\bw\sim\Diri(\mathbf 1_n)$.
\begin{theorem}
\label{thm:solution-gap}
Suppose (A1-5) hold with $3\le S\le n$, and $v_\infty^{\bc} := \bc^\top H_\infty^{-1}\Sigma_\infty H_\infty^{-1}\bc > 0$ for a fixed contrast $\bc\in\R^{d_\theta}$, $\|\bc\|=1$. Suppose $\gamma \in [\underline \gamma,\overline\gamma]\subset(0,\infty)$ and $S\gamma\ge3$. Define the projected bootstrap laws $\mathcal L_\sub^{\bc,\bX_n,\pi}:=\mathrm{Law}\{\bc^\top\sqrt n\,(\theta^\star_\sub(\bar\bw)-\theta_0^\star)\mid \bX_n,\pi\}$ and $\mathcal L_\Diri^{\bc,\bX_n}:=\mathrm{Law}\{\bc^\top\sqrt n\,(\theta^\star_\Diri(\bw)-\theta_0^\star)\mid \bX_n\}$. Then it follows that
\begin{align}
\label{eqn:solution-gap}
\E_{\bX_n,\pi}\bigl[W_1\bigl(\mathcal L_\sub^{\bc,\bX_n,\pi}, \mathcal L_\Diri^{\bc,\bX_n}\bigr)\bigr] \lesssim  (S\gamma+1)^{-1/2}+n^{-1/2} + \bigl|1-\sqrt{\tfrac{(S-1)(n+1)}{(S\gamma+1)(n-1)}}\bigr|.
\end{align}
\end{theorem}
We remark that standard bootstrap techniques requiring exchangeability \citep{prae:well:93,chen:huan:10} do not directly apply because the blocked weights are non-exchangeable conditional on $\pi$. \cite{shin:etal:24} address this issue by randomizing $\pi$ and borrowing the standard consistency result with $S\gg\sqrt n$. In contrast, since privacy favors fixed and moderate $S$, we compare the linearized blk-BB and org-BB through a finite normal-approximation bound \citep{chen:etal:10}, which leaves the blocking error of order $S^{-1/2}$ explicit. The third term in \eqref{eqn:solution-gap} stemming from the mismatch between each law's  variance disappears at 
\begin{align}
\label{eqn:alpha-star}
\gamma=\gamma^\star(n,S) := \frac{S(n+1)-2n}{S(n-1)}. 
\end{align}
When $\gamma=\gamma^\star(n,S)$ is used in Theorem~\ref{thm:solution-gap},
the condition $S\gamma\ge3$ is guaranteed for all $n>1$ if $S\ge5$. After rescaling back by $n^{-1/2}$, the same bound implies that the unscaled blocked and ordinary bootstrap laws become close for any fixed $S$. Provided that $\mathcal L_\Diri^{\bc,\bX_n}$ is known to converge to a normal distribution \citep{lydd:etal:19}, $\theta^\star_\sub(\bar\bw)$ involves approximation error for frequentist coverage for fixed $S$ and improves by order $S^{-1/2}$. 

The following proposition more specifically studies the implication of $\gamma^{\star}$ in relating the first and second moments of the bootstrap solutions. 
\begin{proposition}
\label{prop:choice-alpha}
Suppose (A1-5) hold with fixed $3\le S\le n$. Then $\|\E_{\bX_n,\pi}[\E_{\bu}[\theta^\star_\sub(\bar\bw)|\bX_n,\pi]-\E_{\bw}[\theta^\star_\Diri(\bw)|\bX_n]]\|=O(n^{-1})$ for any fixed $\gamma\in[\underline\gamma,\overline\gamma]$. Denote by $\Var(Z|{\cal Y}):=\E[\|Z-\E [Z|{\cal Y}]\|^2|{\cal Y}]$ the conditional variance of $Z$ given ${\cal Y}$. At $\gamma=\gamma^{\star}$ of \eqref{eqn:alpha-star}, it follows $|\E_{\bX_n,\pi}[\Var[\theta^\star_\sub(\bar\bw)|\bX_n,\pi]]-\E_{\bX_n}[\Var[\theta^\star_\Diri(\bw)|\bX_n]]|=O(n^{-3/2})$.
\end{proposition}
The calibrated concentration parameter $\gamma^\star(n,S)\leq 1$ is an increasing function over $S$. Intuitively, it diffuses the variability more for smaller $S$ so that the block weights and thus matches the posterior dispersion of blk-BB to that of org-BB. The value $\gamma^\star(n,S)$ depends only on $(n,S)$ and not on the observed data $\bX_n$. Hence, \emph{this variance calibration is data-free and requires no additional privacy accounting}. This is remarkably useful in the private setting. For instance, data-driven calibration methods, such as double-bootstrap corrections, would themselves have to be implemented under DP and composed with the rest of the analysis, which may require additional noise or privacy budget.

Figure~\ref{fig:thm21-regression} examines Theorem~\ref{thm:solution-gap} and Proposition~\ref{prop:choice-alpha} on the linear regression example. We draw $X_i\sim \mathcal{N}(\mathbf 0,{\bf I})$ and $y_i=X_i^{\!\top}\theta_0+\varepsilon_i$, $\varepsilon_i\sim \mathcal{N}(0,1)$, where $\theta_{0,j}=(-1)^{j}/(1+j)$ for $j=0,\dots,p-1$. The ridge-type solution is $\theta^{\star}(\bw):=(X^{\!\top}\mathrm{diag}(\mathbf w)X+\lambda I)^{-1}X^{\!\top}\mathrm{diag}(\mathbf w)y$. We set $p=5$ and $\lambda=10^{-4}$. The left plot illustrates the $1$-Wasserstein distance averaged over the $p$ coefficients when either $S$ or $n_S$ is fixed. The derived rate in Theorem~\ref{thm:solution-gap} is consistent with the simulation. It also justifies that the unscaled blk-BB distribution converges to the org-BB for any fixed $S$ as $n$ increases. The middle one depicts the blk-BB with $\gamma=1$ tends to substantially underestimate the posterior variance compared to the org-BB for fixed $S$, but $\gamma=\gamma^{\star}$ nearly matches the two variances. The right panel compares the frequentist coverage when either $\gamma=1$ or $\gamma=\gamma^{\star}$ is set. This highlights that fixing a relatively small $S$ can be reasonable when $n$ is moderately large.
\begin{figure}[ht!]
    \centering
    \includegraphics[width=1.0\linewidth]{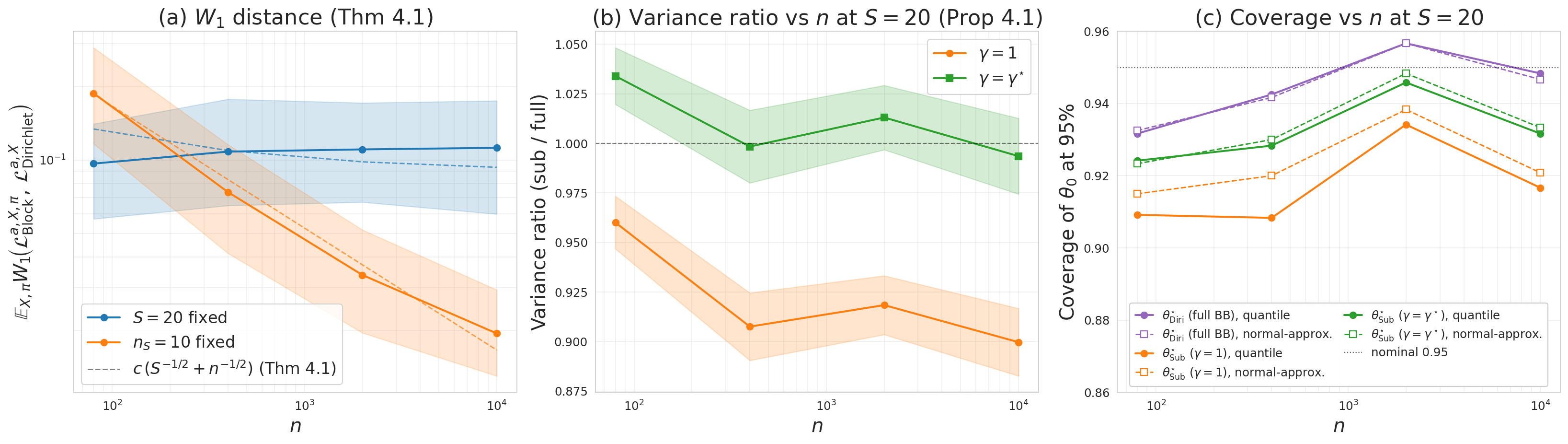}
    \caption{Numerical illustration of Theorem~\ref{thm:solution-gap} and Proposition~\ref{prop:choice-alpha} in the regression example with $S=20$ fixed and $n_S=10$ fixed, based on 300 replicates and 20k bootstrap draws. Panels show (a) 1-Wasserstein distance versus $n$, (b) total variance ratio ${\widehat\Var}[\theta^{\star}_{\sub}]/{\widehat\Var}[\theta^{\star}_{\Diri}]$, and (c) $95\%$ coverage under quantile intervals (solid) and normal approximations (dashed).}
\label{fig:thm21-regression}
\end{figure}

\subsection{Role of $\gamma^{\star}$ in Private Training} 

A natural follow-up question is then the role of $\gamma^{\star}$ in noisy training. The convergence analysis Theorem~\ref{thm:conv-gbb} extends to answer this, and to this end, we suggest $\gamma^{\star}$ in the theorem is a conservative choice in noisy training. We make the dependence on $\gamma$ explicit by writing $G_{\phi^{(T)}}^{(\gamma),\text{DP}}$ and define $\Delta_T(\gamma) := \E_{\bX_n,\pi,\bomega_T,\bu} \| G_{\phi^{(T)}}^{(\gamma),\text{DP}}(\bu) - \theta^\star_\sub(\bar\bw(\bu)) \|^2 $. The following corollary extends Proposition~\ref{prop:choice-alpha} for $G_{\phi^{(T)}}^{(\gamma),\text{DP}}$. 
\begin{corollary}
\label{prop:var-transfer}
Suppose the assumptions of Theorems~\ref{thm:conv-gbb} and \ref{thm:solution-gap} hold. Let $a_{S,n}(\gamma)=\tfrac{(S-1)}{(S\gamma+1)(n-1)}-\tfrac{1}{n+1}$. Then, it follows $| \E_{\bX_n,\pi,\bomega_T}\Var_\bu[G_{\phi^{(T)}}^{(\gamma),\text{DP}}(\bu)]-\E_{\bX_n}\Var_\bw[\theta^\star_\Diri(\bw)]|\leq |a_{S,n}(\gamma)a_n|+2D_\Theta \Delta_T(\gamma)^{1/2}+O((1+\sqrt{\tfrac{S-1}{S\gamma+1}}\gamma^{-3/2})n^{-3/2})$ for some constant $a_n\geq0$.
\end{corollary}
At the variance-matching concentration $\gamma^\star(n,S)$ in Theorem~\ref{thm:solution-gap}, the first leading term is $|a_{S,n}(\gamma^{\star})a_n|=0$ exactly. Thus the trained generator approximates the original-bootstrap posterior variance up to its own training error $\Delta_T(\gamma^\star)^{1/2}$ and $O(n^{-3/2})$. Observe that $a_{S,n}(\gamma)$ is strictly decreasing w.r.t. $\gamma$. In the following remark, we observe that $\gamma^{\star}(n,S)$ serves as a conservative choice in general. 
\begin{remark}
\label{rem:gamma-conservative}
Write the leading term in Corollary \ref{prop:var-transfer} as ${\cal R}_T(\gamma)=|a_{S,n}(\gamma)|a_n+\Psi_T(\gamma)$ with $\Psi_T(\gamma):=2D_{\Theta}\Delta_T(\gamma)^{1/2}$. Assume $\Delta_T(\gamma)$ is differentiable at $\gamma^{\star}$. Since $a_{S,n}(\gamma^\star)=0$, the term $|a_{S,n}(\gamma)|$ has a kink at $\gamma^\star$ with one-sided slopes $\pm a_n|a'_{S,n}(\gamma^\star)|$, where $a_n|a'_{S,n}(\gamma^\star)|=\frac{S(n-1)a_n}{(S-1)(n+1)^2}$, so the one-sided derivatives of ${\cal R}_T$ at $\gamma^\star$ equal $\pm a_n|a'_{S,n}(\gamma^\star)|+\Psi_T'(\gamma^\star)$. Hence $\gamma^\star$ is locally optimal when $|\Psi_T'(\gamma^\star)|< a_n|a'_{S,n}(\gamma^\star)|$, whereas $\Psi_T'(\gamma^\star)< -a_n|a'_{S,n}(\gamma^\star)|$ admits a training-error-adjusted $\tilde\gamma^\star>\gamma^\star$ with ${\cal R}_T(\tilde\gamma^\star)<{\cal R}_T(\gamma^\star)$. As expected, $\gamma^{\star}<\tilde\gamma^\star$ would yield a larger uncertainty to the generator's output and thus be regarded as conservative. In either case $\gamma^\star$ removes the leading variance gap exactly and remains a principled data-free default.
\end{remark}
The assumption $\Psi_T'(\gamma^{\star})\le 0$ is realistic from our perspective. A larger $\gamma$ injects less stochastic volatility to the training system than $\gamma^{\star}$, and thus the generator-error term is likely to stay or decrease for $\gamma\geq \gamma^{\star}$. The remark is an effort to discuss that $\gamma^\star$ can still serve as the principled default in noisy training. Estimating $\tilde \gamma^{\star}$, however, seems practically intractable because it requires an additional privacy-accounting process, which may cost much more expensive budget than private training with $\gamma^{\star}$.

\section{Simulation Studies for Private UQ}
\label{sec:simul}

This section presents simulation studies to understand the statistical behavior of our neural-net-based private generative Bayesian bootstrap method (PGBB). 
Our experiments consider two settings: likelihood misspecification in linear regression and well-specified exponential-family mean estimation. In all following experiments, PGBB employs a feed-forward network with three hidden layers, each of which has 256 nodes. The Adam optimizer with a learning rate decay executes the gradient descent procedure after perturbing the gradient. Since the optimizer operates on the already perturbed gradients, it essentially satisfies the post-processing and therefore does not affect the established privacy guarantee. PGBB returns 2k posterior samples. By default, $\gamma=\gamma^{\star}$ for $\bu\sim \Diri(\gamma {\bf 1}_S)$ is set. Other details of simulation studies appear in SM~\ref{app:simulation}.

\subsection{Linear Regression under Heavy-tailed Errors}
\label{sec:simul-linear}

For each replicate, we draw $n\in \{10\text{k},40\text{k}\}$ independent observations $\{(X_i, Y_i)\}_{i=1}^n$ from $X_i \sim \mathrm{Uniform}[-1,1]^p$ and $Y_i = X_i^\top \boldsymbol{\beta}^\star + \varepsilon_i$ with $\boldsymbol{\beta}^\star = (1.5, -1.5, 1.0)^\top$ and noise scale $\sigma_\varepsilon = 0.5$. Two heavy-tailed error regimes are considered: Student-$t$ with $\varepsilon_i \sim \sigma_\varepsilon\, t_3$ and {contaminated normal} with $\varepsilon_i \sim 0.9\, N(0, \sigma_\varepsilon^2) + 0.1\, N(0, 64\sigma_\varepsilon^2)$. Both regimes are mean-zero, but represent a model misspecification setting if a likelihood-based Bayesian model assumes a standard normal distribution.

We compare PGBB with three private Bayesian or bootstrap baselines under the same target $(\epsilon,\delta)$-DP level, with $\delta=1/n$ unless stated otherwise. PGBB and DP-Bootstrap directly use the Huber loss $l((\beta,\sigma_H),(Y_i,X_i)) = \sigma_H\rho_c((Y_i-X_i^\top\beta)/\sigma_H)+\kappa\sigma_H/2$, where $\rho_c(z)=z^2/2$ for $|z|\le c$ and $\rho_c(z)=c|z|-c^2/2$ otherwise. We set $c=1.345$ and $\kappa=0.71$, following \cite{avel:etal:23}. NoisyGibbs \citep{bern:shel:19} implements a Gibbs sampler under a Gaussian linear model with summary statistics perturbed, where the non-private statistics play as a latent variable. To compare the performance of private UQ, we customize a DP-SGLD baseline (DP-SGLD) that uses Poisson subsampling and Gaussian perturbation to the updating gradient, so that it provides multiple posterior samples under the same privacy accounting procedure as ours except for the Dirichlet randomness. Note that a last iterate of SGLD can be DP under regularity conditions \citep{bai:etal:19,bert:etal:25}, but a single posterior draw is not enough for reasonable UQ. Our customized DP-SGLD is designed to use the same Huber loss via the Gibbs posterior \citep{biss:etal:16}, i.e., $p(\beta,\sigma_H|\{Y_i,X_i\}_{i=1}^n)\propto \exp(-\omega\sum_{i=1}^n l((\beta,\sigma_H),(Y_i,X_i)))$. Notably, our bootstrap-based method does not introduce such a scaling factor $\omega$.

For a fixed-draw comparison, PGBB releases 2k posterior draws after $T=10\text{k}$ iterations. DP-SGLD runs $20\text{k}$ iterations to secure a moderate effective sample size (ESS) and uses half of its chain for inference. Similarly, NoisyGibbs runs $400\text{k}$ iterations, considering ESS, and returns half of the chain. See SM Table~\ref{tab:ess-expo} summarizing ESS for both methods. Notably, DP-Bootstrap and PGBB generate fully effective samples by construction. DP-Bootstrap releases $20$ private bootstrap samples. We avoid the trivial failure mode for each method. For instance, the performance of DP-SGLD and DP-Bootstrap tends to deteriorate as more private estimates are collected, while PGBB can generate an arbitrary number of posterior estimates without any privacy restriction. Likewise, using too many iterations $T$ for PGBB may result in an uninformative sampler. 

\begin{figure}[ht!]
\centering
\begin{subfigure}[t]{\linewidth}
    \centering
    \includegraphics[width=1.00\linewidth]{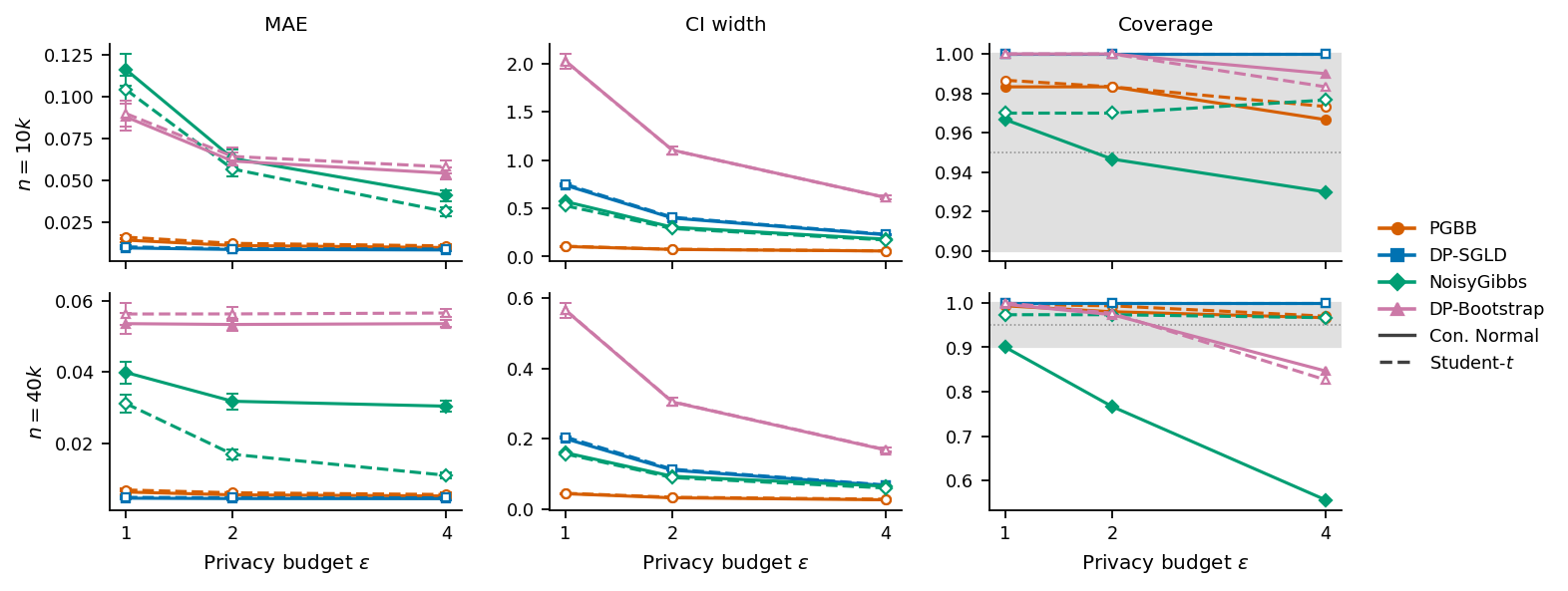}
    \caption{Comparison of competing methods}
    \label{fig:linear-reg}
\end{subfigure}

\vspace{0.2em}

\begin{subfigure}[t]{\linewidth}
    \centering
    \includegraphics[width=1.00\linewidth]{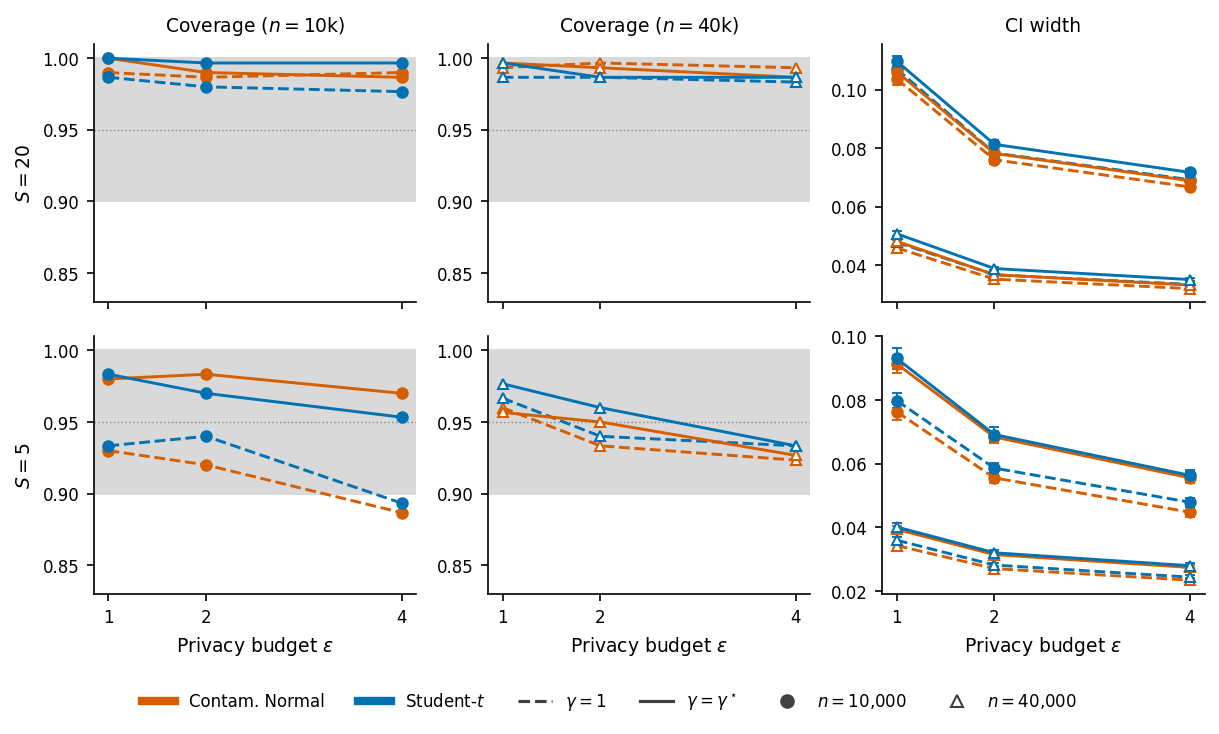}
    \caption{Comparison of $\gamma=1$ and $\gamma^\star(n,S)$ at $S\in\{20,5\}$}
    \label{fig:linear-reg-gamma}
\end{subfigure}
\caption{Linear regression under heavy-tailed errors. In panel (a), each point averages the metric over $3$ coefficients within each replicate and over $100$ replicates; error bars for MAE and CI width show $\pm2$ Monte Carlo SEs across replicate-level averages. Panel (b) compares $\gamma=1$ and $\gamma^\star$ by coverage and CI width. }
\label{fig:linear}
\end{figure}

\paragraph{Results} Figure~\ref{fig:linear-reg} overlays the results. PGBB attains competitive MAE and the narrowest credible intervals in every combination of $n\times\epsilon$, while maintaining frequentist coverage close to the nominal $95\%$ level. Performance overall improves as $n$ and $\epsilon$ increase. PGBB targets the robust loss-based decision rule directly and therefore does not have to fit the contaminated data through a Gaussian likelihood. In contrast, the likelihood-based private Bayesian methods pay a cost for this mismatch, either through wider uncertainty or deteriorating coverage. NoisyGibbs makes this cost especially clear under contamination, as its empirical coverage deteriorates as $n$ grows due to the wrong Gaussian likelihood. DP-Bootstrap exposes a different failure mode. It is extremely sensitive to the choice of $\ell_2$-regularizer (used for formal privacy guarantee) even when the number of posterior samples is scarce. Lastly, SM~\ref{app:mc-K-effect} verifies the $K$-invariance of PGBB's estimation performance discussed in Theorem~\ref{thm:conv-gbb}. SM Table~\ref{tab:linreg-summary} expands the numerical values of Figure~\ref{fig:linear-reg}, and SM Tables~\ref{tab:linreg-b1}--\ref{tab:linreg-b3} give coefficient-specific results.

\subsubsection{Choice of $\gamma$}
Figure~\ref{fig:linear-reg-gamma} contrasts the coverage on $\gamma=1$ with the variance-matched $\gamma=\gamma^\star(n,S)$ of Proposition~\ref{prop:choice-alpha}. Since the MAE remains almost the same, the figure summarizes coverage and CI width only. The canonical $\gamma=1$ under-covers at $S=5$, dipping to $0.87\sim 0.88$ at $(\varepsilon, n) = (4,\, 10\text{k})$ under both noise regimes, whereas $\gamma=\gamma^\star$ recovers the nominal coverage level by slightly widening the credible intervals. Thus, while $\gamma=1$ can be a serviceable default at moderate $S$, the calibrated choice $\gamma^\star$ provides a simple correction when coverage is sensitive to blocking. 

\subsection{Well-specified Likelihood}
\label{app:expo-family}

This section is designed to see whether the PGBB's uncertainty quantification is comparable to a known noise-aware Bayesian method. We consider Gaussian-mean estimation $X_i\sim\mathcal{N}(1.5,0.5^2)$ and exponential-rate estimation $X_i\sim\mathrm{Exp}(1/2)$ at relatively small sample sizes $n\in\{500,5000\}$. The target estimand is $\E[X]$ in both cases. For these exponential-family models, the noise-aware sufficient-statistic perturbation of \cite{bern:shel:18}, called NoisyGibbs, explicitly models the privacy noise in constructing the inference. Both DP-SGLD and PGBB use their clipped log likelihood. DP-Bootstrap reports 1k deconvoluted bootstrap estimates still far less than 2k, which is available only for a univariate parameter \citep{wang:etal:25}. The total number of iterations of DP-SGLD and NoisyGibbs is determined to secure moderate ESS (SM Table~\ref{tab:ess-expo}). Not surprisingly, for DP-SGLD, longer chains require more noise per step, which can degrade the estimate. PGBB and DP-Bootstrap draw independent samples always. DP-Bootstrap, however, spends more privacy budget to collect more draws, while NoisyGibbs does not in collecting more samples. 

\begin{figure}[h]
    \centering
    \includegraphics[width=1.0\linewidth]{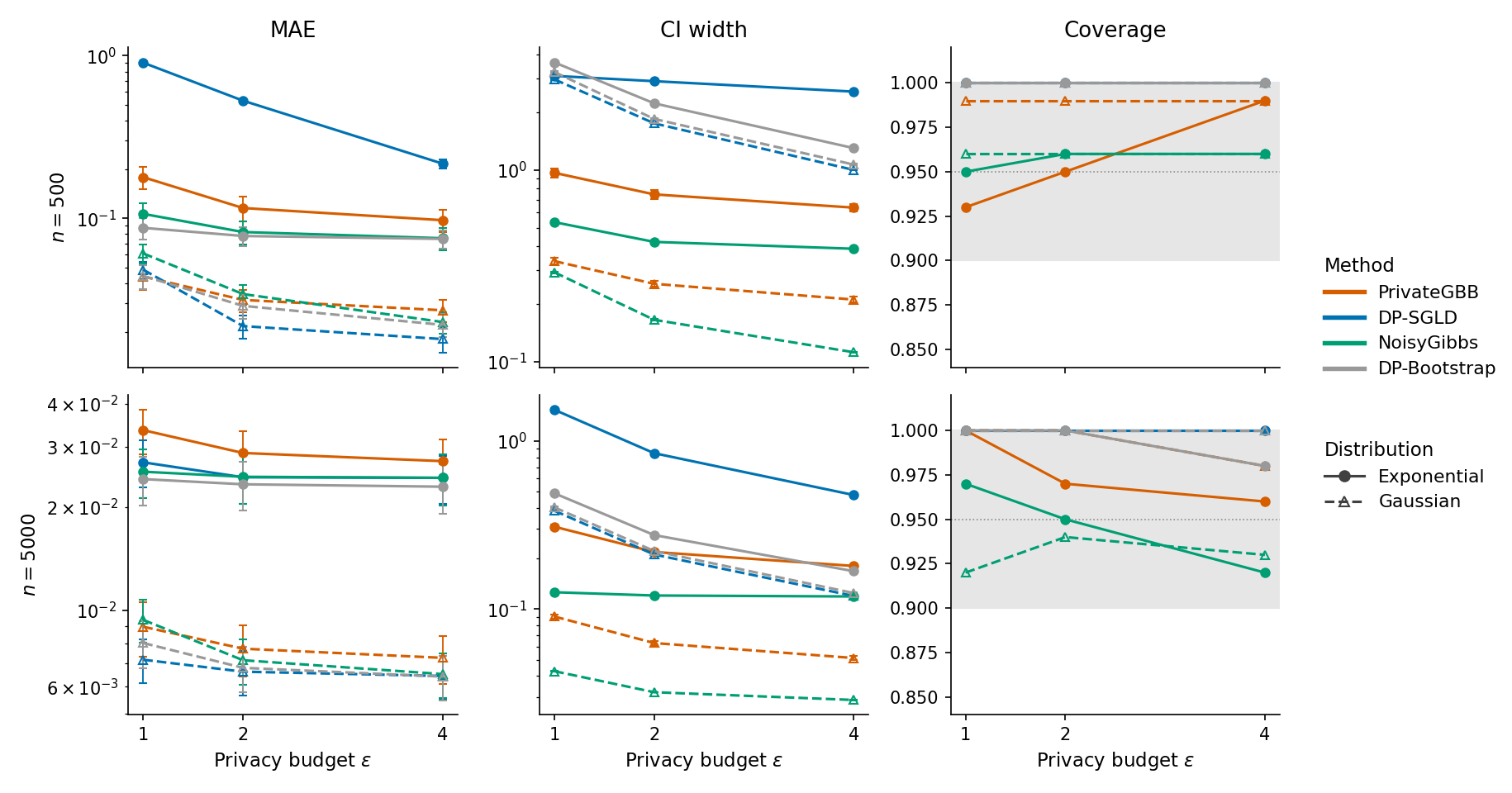}
    \caption{Exponential-family experiment. Points are means over $100$ replicates. Error bars, shown for MAE and CI width, give $\pm2$ Monte Carlo SEs across replicates. Coverage is the empirical inclusion rate of nominal $95\%$ credible intervals.}
    \label{fig:expo-family}
\end{figure}

Figure~\ref{fig:expo-family} demonstrates that PGBB is comparable when a model is correctly specified. All models have nominal coverages, although DP-SGLD has overly extensive coverages. NoisyGibbs tends to give relatively smaller MAE, coverage closer to the target $95\%$, and the shortest intervals. Interestingly, PGBB stays closer to the gold standard than competing methods, particularly for the case of $n=500$; MAE and interval width are slightly larger than those of NoisyGibbs, and coverage remains near nominal. While DP-Bootstrap underperforms in linear regression, it efficiently works in this exponential family example when $n=5\text{k}$. DP-SGLD achieves a smaller MAE in the Gaussian case, but it has substantially wider CI width. 

\section{Real Data Application}
\label{sec:realdata}

Quantile regression is an effective tool for distributional-level analysis. For instance, \cite{angr:etal:06} fit quantile regression for a grid of quantiles to study how the \texttt{return-to-education} varies across the \texttt{wage} distribution on \textbf{U.S. Census data}, where the response is the log weekly wage and the covariates are years of schooling, potential work experience, and a black indicator. They reported rising within-group inequality in the upper half of the distribution between 1980 ($n=65,023$) and 2000 ($n=97,397$). In another example, \cite{abre:01} studied how maternal demographics and behavior shift the distribution of birthweight for \textbf{U.S. Natality data}. Fitting the quantile regression of \texttt{birthweight} for various levels, the study found that several factors, e.g., notably \texttt{maternal-smoking}, \texttt{race}, \texttt{education}, and \texttt{prenatal-care}, have a markedly larger effect at the lower quantiles.

\begin{figure}[ht!]
\centering
\begin{subfigure}[t]{1.0\linewidth}
    \centering
    \includegraphics[width=0.8\linewidth]{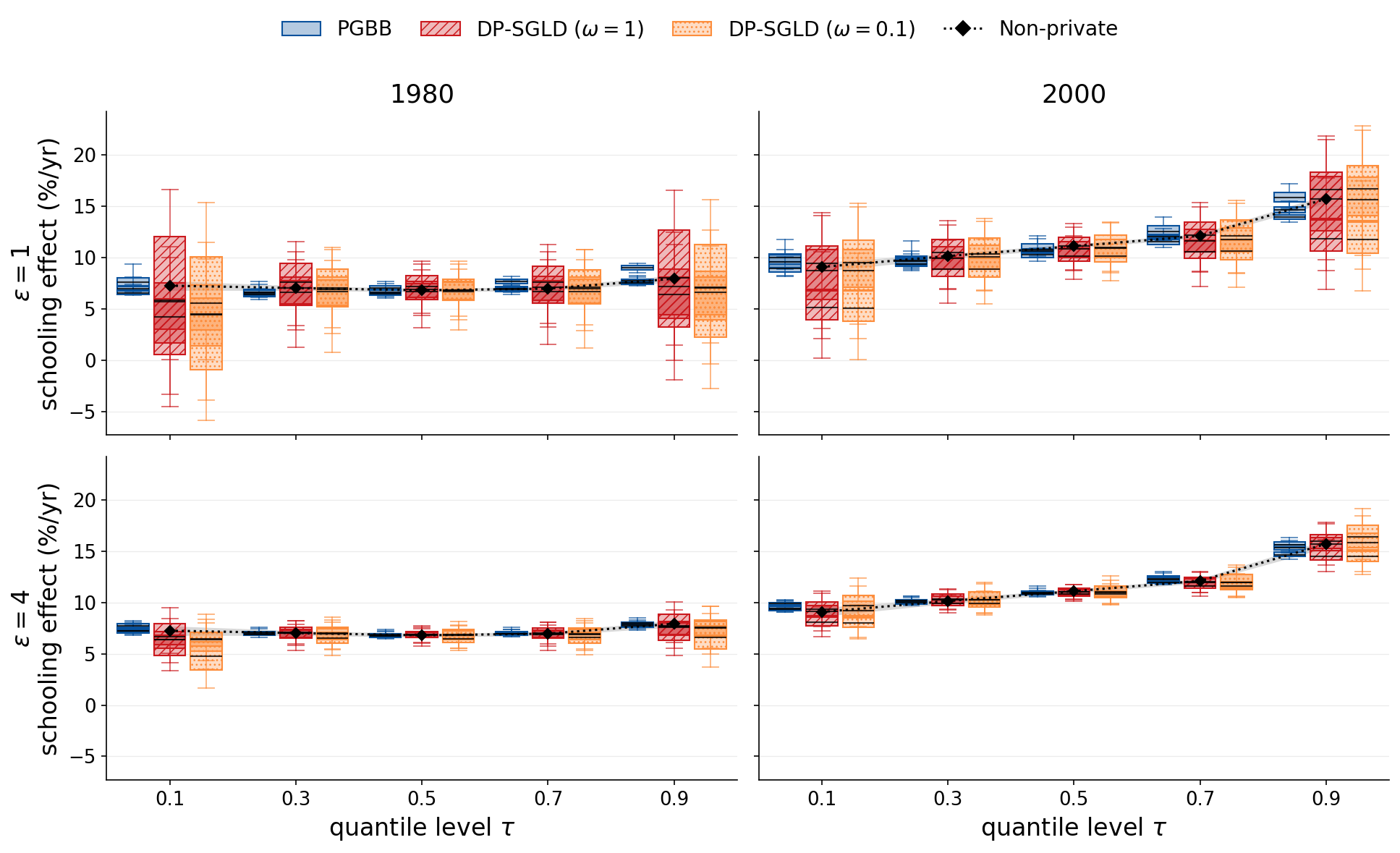}
    \caption{Returns to schooling in 1980 and 2000.}
    \label{fig:acfv_schooling}
\end{subfigure}
\vspace{0.6em}
\begin{subfigure}[t]{1.0\linewidth}
    \centering
    \includegraphics[width=0.8\linewidth]{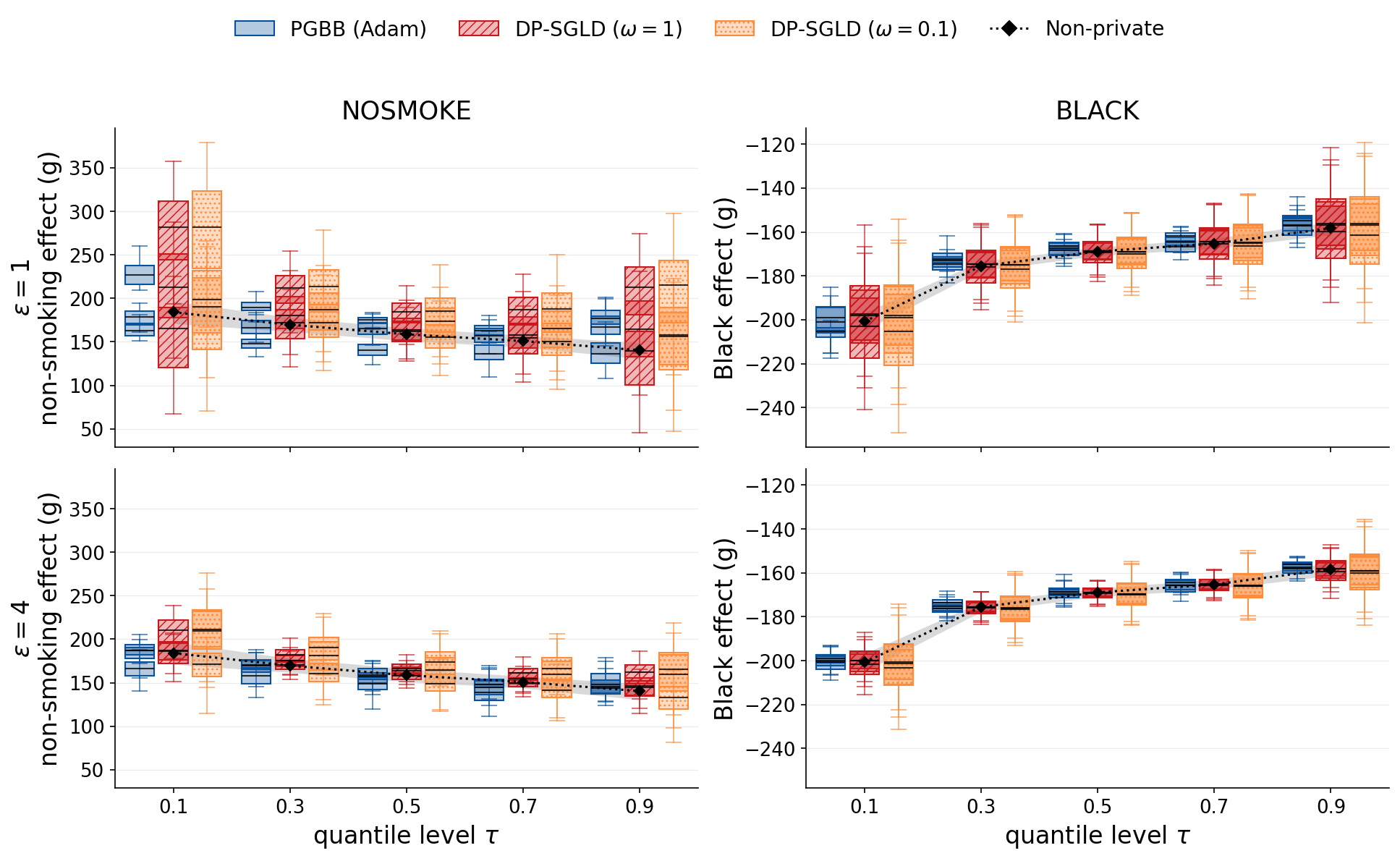}
    \caption{Birthweight coefficients (grams) for no-smoking and black indicators.}
    \label{fig:natality_focal}
\end{subfigure}
\vspace{-0.6cm}
\caption{Released DP quantile-regression coefficients across $\tau$: PGBB and DP-SGLD ($\omega\in\{1,0.1\}$) versus the non-private fit. Boxplots overlay three random runs per method.}
\label{fig:real_data_subfigures}
\end{figure}

The proposed joint PGBB scheme in Section~\ref{sec:joint-learning} can be highly beneficial when protecting individuals across all quantile levels is desirable, like the two aforementioned applications. We apply the kernel-smoothed quantile loss $\rho_{\tau,h}(u)=(\rho_\tau * K_h)(u)=\int_{-\infty}^{\infty}\rho_\tau(v)\,K_h(u-v)\,\mathrm{d}v$ \citep{tan:etal:22}, where $K_h(\cdot)$ is a non-negative kernel of bandwidth $h$, whose derivative lies in $[\tau-1,\tau]$ and is therefore bounded in magnitude by $\max\{\tau,\,1-\tau\}$. Accordingly, DP-SGLD uses the per-$\tau$ bound $C(\tau)=\max\{\tau,1-\tau\}B_X$ with $\|X\|\leq B_X$ and PGBB uses $\sup_{\tau} C(\tau)$ for the joint training. We set $\tau \sim \mathrm{discrete}(\{o_k\}_{k=1}^5,\{p_{k}\}_{k=1}^5)$ where $o_k=(2k-1)/10$ and $p_k\propto {(o_k(1-o_k))^{-1}}$. In each example, the original non-private quantile regression and the customized DP-SGLD via the Gibbs posterior are provided as reference models. Because DP-SGLD runs over the 5 grids of $\tau$ independently, the privacy accountant of DP-SGLD is updated to reflect this composition under the R\'enyi composition \citep{miro:17} rather than using the loose linear composition. Two DP-SGLD models are fitted for $\omega\in\{1,0.1\}$ for $T=200\text{k}$; both have the average effective sample size across all variables still less than $100$. To avoid trivial failure (increasing $T$ and perturbing gradients substantially more), we do not increase $T$ further. Further details appear in SM~\ref{sec:supp_realdata}. 

\begin{table}[t]
\centering
\scriptsize
\caption{U.S.\ Census and Natality quantile applications. Metrics are the MAE of the posterior mean relative to the non-private fit, the inclusion rate of the non-private fit in nominal $95\%$ credible intervals, and interval width. Values are averaged over the quantile grid and $10$ runs; parentheses give Monte Carlo SEs across runs.}
\label{tab:main_bias_unc}
\begin{tabular}{lll ccc ccc}
\toprule
& & & \multicolumn{3}{c}{$\varepsilon=1$} & \multicolumn{3}{c}{$\varepsilon=4$} \\
\cmidrule(lr){4-6}\cmidrule(lr){7-9}
Data & Variable & Method & MAE (SE) & Cov & Width (SE) & MAE (SE) & Cov & Width (SE) \\
\midrule
\multirow{6}{*}{U.S.\ Census}
 & \multirow{3}{*}{educ (1980)} & PGBB & 0.52\,(0.09) & 0.700 & 1.29\,(0.09) & 0.17\,(0.04) & 0.840 & 0.77\,(0.02) \\
 & & DP-SGLD ($\omega{=}1$) & 0.98\,(0.20) & 1.000 & 10.91\,(0.44) & 0.30\,(0.05) & 1.000 & 3.16\,(0.12) \\
 & & DP-SGLD ($\omega{=}0.1$) & 1.05\,(0.19) & 1.000 & 11.11\,(0.45) & 0.37\,(0.06) & 1.000 & 4.30\,(0.13) \\
\cmidrule(lr){2-9}
 & \multirow{3}{*}{educ (2000)} & PGBB & 0.69\,(0.11) & 0.680 & 1.70\,(0.11) & 0.30\,(0.03) & 0.820 & 0.99\,(0.06) \\
 & & DP-SGLD ($\omega{=}1$) & 0.51\,(0.09) & 1.000 & 9.26\,(0.48) & 0.17\,(0.02) & 1.000 & 2.79\,(0.16) \\
 & & DP-SGLD ($\omega{=}0.1$) & 0.49\,(0.10) & 1.000 & 9.59\,(0.54) & 0.22\,(0.05) & 1.000 & 4.04\,(0.22) \\
\midrule
\multirow{6}{*}{Natality}
 & \multirow{3}{*}{\texttt{NOSMOKE}} & PGBB & 16.5\,(2.7) & 0.820 & 58.7\,(6.7) & 6.4\,(1.0) & 0.960 & 38.1\,(2.5) \\
 & & DP-SGLD ($\omega{=}1$) & 21.2\,(5.1) & 0.980 & 109.4\,(3.1) & 6.5\,(1.3) & 1.000 & 40.7\,(1.3) \\
 & & DP-SGLD ($\omega{=}0.1$) & 22.8\,(4.6) & 1.000 & 137.6\,(4.3) & 11.8\,(1.8) & 1.000 & 91.2\,(3.9) \\
\cmidrule(lr){2-9}
 & \multirow{3}{*}{\texttt{BLACK}} & PGBB & 2.0\,(0.3) & 0.980 & 17.5\,(1.0) & 1.0\,(0.2) & 1.000 & 11.2\,(0.3) \\
 & & DP-SGLD ($\omega{=}1$) & 0.7\,(0.1) & 1.000 & 47.1\,(0.5) & 0.3\,(0.0) & 1.000 & 17.2\,(0.2) \\
 & & DP-SGLD ($\omega{=}0.1$) & 0.8\,(0.2) & 1.000 & 57.6\,(0.6) & 0.6\,(0.1) & 1.000 & 37.0\,(0.3) \\
\bottomrule
\end{tabular}
\end{table}

\paragraph{Results}
Figures~\ref{fig:acfv_schooling} and \ref{fig:natality_focal} show the released coefficients across $\tau$ with boxplots over three runs. In both applications, the posterior mean of both methods tracks the non-private fit. In particular, PGBB returns only slightly wider $95\%$ intervals, so the findings of non-private analysis are preserved. Intervals widen with $\max\{\tau,1-\tau\}$ and for less precisely identified coefficients such as not-smoking, where the privacy noise dominates. Figure~\ref{fig:se-totalvar} more specifically assesses such phenomena in detail. It relates the standard error of non-private coefficients to within-variability and between-variability, respectively. As we observe, PGBB has smaller within-variability overall, but slightly larger between-variability. Over $10$ independent runs, Table~\ref{tab:main_bias_unc} shows that PGBB has smaller MAE and moderate widths at comparable coverage, whereas DP-SGLD overall has the overcoverage issue and is sensitive to the choice of $\omega$. SM~\ref{app:further-result-realdata} reports the remaining covariates and the pooled plots.

\begin{figure}[t]
\centering
\includegraphics[width=\linewidth]{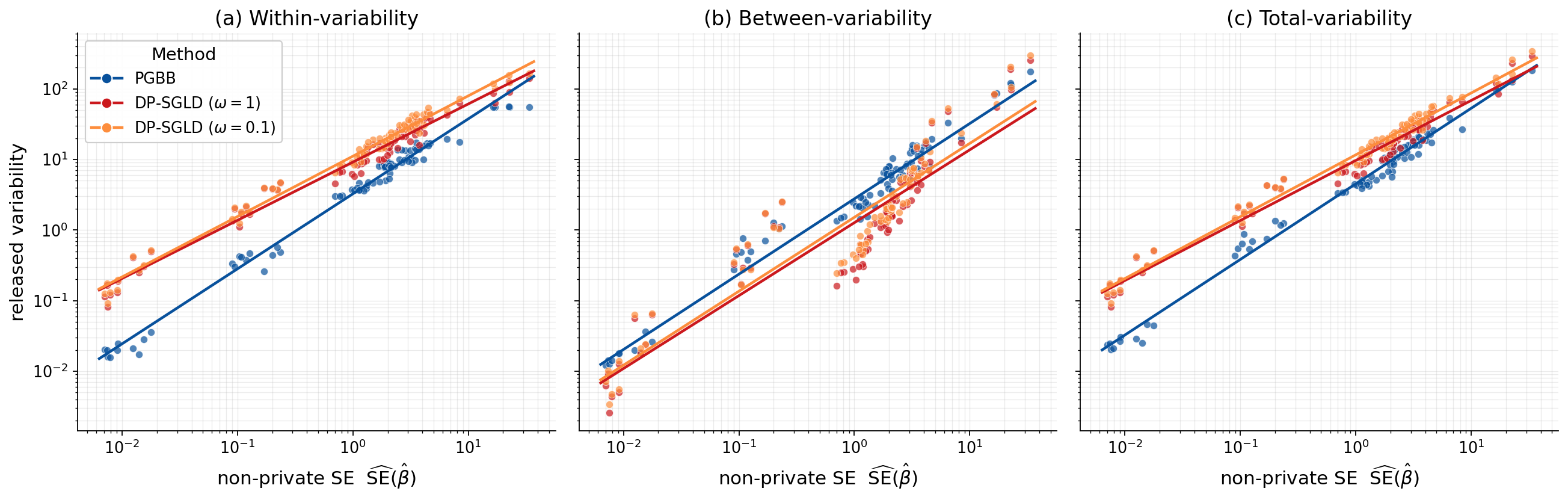}
\caption{Posterior variability versus the non-private standard error $\widehat{\mathrm{SE}}(\hat\beta)$, decomposed over $R=10$ independent runs and plotted on the standard-deviation scale. Panels show (a) within-run variability $\{R^{-1}\sum_r\widehat{\mathrm{Var}}_b(\beta_r^{(b)})\}^{1/2}$, (b) between-run variability $\{\widehat{\mathrm{Var}}_r(\bar\beta_r)\}^{1/2}$, and (c) total variability, the square root of the sum of the two variances. Panels pool all Census and Natality covariates at $\varepsilon=1$ across the quantile grid, with shared log axes.}
\label{fig:se-totalvar}
\end{figure}

\section{Discussion}
\label{sec:discussion}

This work develops a differentially private generative Bayesian bootstrap (PGBB). The method combines three advantages not jointly provided by existing private-UQ methods: decoupled private learning followed by post-processing, freedom from specifying a data-generating process, and joint learning for a range of decision rules under a single privacy budget. We believe these features can reduce the impact of privacy noise to improve uncertainty quantification in diverse real-world applications, such as when a range of loss-based decision rules with uncertainty is required for reliable inference, as in our quantile regression applications (Section~\ref{sec:realdata}). Nevertheless, we acknowledge these promising properties can come at a computational cost, since fitting a large amortized generator may require a high-performance GPU.

Several practical extensions may further improve private amortized learning. Shuffling the partition $\pi$ at each training iteration could help when a subgroup contains outliers and would preserve privacy, although the calibration rule $\gamma^\star$ would no longer apply. Nested random reweighting within groups is another possibility, but an efficient formal privacy accountant may be nontrivial. Finally, allowing pseudo-samples with $\alpha>0$ in the posterior bootstrap is attractive for privacy (SM~\ref{app:pseudo-sample-alpha}) and connects to the growing use of generative-AI outputs \citep{ohag:rock:25}. We leave this to future work because posterior-variance calibration and bias control become more subtle.

\section{Disclosure Statement}\label{disclosure-statement}
The authors report that there are no competing interests to declare.

\section{Data Availability Statement}\label{data-availability-statement}
The U.S.\ Census data and the U.S.\ natality data analyzed in this article are publicly available. The U.S.\ Census data are the public-use samples analyzed by \cite{angr:etal:06}, and the natality records are distributed by the National Bureau of Economic Research at \url{https://www.nber.org/research/data/vital-statistics-natality-birth-data}.

\renewcommand{\refname}{References}
\putbib[ref]
\end{bibunit}

\clearpage
\def\CombinedMode{}
\begin{bibunit}[agsm]   

\setcounter{section}{0}
\renewcommand{\thesection}{S\arabic{section}}
\renewcommand{\theHsection}{supp.\arabic{section}}
\setcounter{equation}{0}\renewcommand{\theequation}{S\arabic{equation}}\renewcommand{\theHequation}{supp.\arabic{equation}} \setcounter{table}{0}\renewcommand{\thetable}{S\arabic{table}}\renewcommand{\theHtable}{supp.\arabic{table}} \setcounter{figure}{0}\renewcommand{\thefigure}{S\arabic{figure}}
\begin{center}{\large\bf Supplementary Material (SM)}\end{center}
\begin{center}{Title: Private Generative Bayesian Bootstrap}\end{center}
\medskip
\section{Further Discussion}

\subsection{Comparison to Existing Studies}

Table~\ref{tab:method-comparison} summarizes representative private UQ methods by inference paradigm. The functionality of estimating a range of decision rules is further discussed in Section~\ref{sec:joint-learning}.
\begin{table}[h]
\centering
\caption{\emph{Data-model free}: requires only a loss $l(\theta,x)$; $\times^{s}$ means simulated data are required. \emph{Scaling free}: no loss scale or temperature is chosen, e.g., the scaling factor of the Gibbs posterior. \emph{Budget-free draws}: posterior draws are DP by post-processing. \emph{Multiple decision rules}: one private run covers a range of loss-based decision rules. $\dagger$: one posterior draw is DP under regularity conditions.}
\label{tab:method-comparison}
\scriptsize
\setlength{\tabcolsep}{6pt}
\renewcommand{\arraystretch}{1.2}
\begin{tabular}{@{}l cccc@{}}
\toprule
Method
 & \shortstack{Data-model\\free}
 & \shortstack{Scaling\\free}
 & \shortstack{Budget-free\\draws}
 & \shortstack{Multiple\\decision rules}\\
\midrule
\multicolumn{5}{@{}l}{\textit{Bayesian}}\\
\quad Posterior sampling$^{\dagger}$ \citep{wang:etal:15,jews:etal:23}
 & \checkmark & $\times$ & $\times$ & $\times$\\
\quad Stochastic Gradient Langevin Dynamics$^{\dagger}$ \citep{bai:etal:19}
 & \checkmark & $\times$ & $\times$ & $\times$\\
\quad MCMC \citep{heik:etal:19,zhan:zhan:23}
 & $\times$ & $\times$ & $\times$ & $\times$\\
\quad Noisy Summary Statistics \citep{bern:shel:18,bern:shel:19}
 & $\times$ & \checkmark & \checkmark & $\times$\\
\quad Approximate Bayesian Computation \citep{gong:22}
 & $\times^{s}$ & \checkmark & \checkmark & $\times$\\
\addlinespace[2pt]
\multicolumn{5}{@{}l}{\textit{Frequentist}}\\
\quad Repro-sample inference \citep{awan:wang:25}
 & $\times^{s}$ & \checkmark & \checkmark & $\times$\\
\quad Nonparametric bootstrap \citep{wang:etal:25}
 & \checkmark & \checkmark & $\times$ & $\times$\\
\midrule
\textbf{PGBB (proposed)}
 & \checkmark & \checkmark & \checkmark & \checkmark\\
\bottomrule
\end{tabular}
\end{table}

\subsection{Algorithm}
\label{app:alg}

Algorithm~\ref{alg:bayesian_bootstrap} trains the generator map $G_\phi$ by perturbing the gradient. Each step draws a Poisson minibatch and $K$ sets of Dirichlet block weights, forms the weighted loss gradient with added Gaussian noise, averages it over the $K$ draws, and descends on $\phi$. The trained map $\phi_T$ then generates posterior samples from fresh weights without further privacy cost.

\begin{algorithm}[ht!]
\caption{Practical Private Bayesian Bootstrap Map $G$}\label{alg:bayesian_bootstrap}
\footnotesize
\noindent\textbf{Require:} Learning rate $\eta$, noise scale $\sigma$, subsampling rate $q\in(0,1)$, block count $S$, MC samples $K$\\
\textbf{Ensure:} Parameter map $\phi_T$
\begin{list}{\arabic{enumi}:}{\usecounter{enumi}%
  \setlength{\leftmargin}{2.2em}\setlength{\labelsep}{0.4em}%
  \setlength{\itemsep}{2pt}\setlength{\topsep}{3pt}\setlength{\parsep}{0pt}}
\item Initialize $\phi_0$
\item Define group size $n_S := n/S$
\item \textbf{for} $t = 1, \dots, T$ \textbf{do}
\item \quad Initialize gradient accumulator $\Delta \phi_t \gets 0$
\item \quad \textbf{for} $k = 1, \dots, K$ \textbf{do}
\item \quad\quad Construct minibatch $B_t^{(k)} \subseteq \{1, \dots, n\}$ by including each index $i$ with probability $q$
\item \quad\quad Generate block weights $\bu^{(k)}$
\item \quad\quad Set $w_j^{(k)} = u_{\lceil j/n_S \rceil} / n_S$
\item \quad\quad Evaluate $g^{(k)}_t = G_{\phi^{(t)}}(\bu^{(k)})$
\item \quad\quad $\nabla{g}_t^{(k)} \gets \big(\sum_{j \in B_t^{(k)}} w_j^{(k)} l'(g^{(k)}_t,x_j) + \mathcal{N}(0, \sigma^2 \mathbf{I})\big)\, \nabla_{\phi} G_{\phi^{(t)}}(\bu^{(k)})$
\item \quad\quad $\Delta \phi_t \gets \Delta \phi_t + \nabla{g}_t^{(k)}$
\item \quad \textbf{end for}
\item \quad \textbf{Descent:}\ $\phi_{t+1} \gets \phi_t - \eta \Delta \phi_t/K$
\item \textbf{end for}
\end{list}
\end{algorithm}

For quantile regression, a single privacy budget can cover the entire quantiles. Algorithm~\ref{alg:joint-qr} adapts Algorithm~\ref{alg:bayesian_bootstrap} to a $\tau$-conditioned map $G_\phi(\bu,\tau)$ trained across the level grid $\tau\in \mathcal T$. Denote by $\psi_\tau(\cdot)$ the derivative of the convoluted quantile loss for the quantile $\tau$. At each Monte Carlo draw, a level $\tau^{(k)}$ is sampled from $P_\tau$ over $\mathcal T$ independently of the data. The level-uniform clip $C_\star=\bigl(\max_{\tau\in\mathcal T}\max(\tau,1-\tau)\bigr)B_x\sqrt p$ bounds the per-record contribution at every level, so the per-record sensitivity is identical across $\tau$.

\begin{algorithm}[ht!]
\caption{Joint Quantile-Process Training of PGBB $G_\phi(\bu,\tau)$}\label{alg:joint-qr}
\footnotesize
\noindent\textbf{Require:} Learning rate $\eta$, noise scale $\sigma$ (calibrated at clip $C_\star$), subsampling rate $q\in(0,1)$, block count $S$, Monte Carlo samples $K$, quantile-level grid $\mathcal T$ with sampling law $P_\tau$\\
\textbf{Ensure:} Parameter map $\phi_T$
\begin{list}{\arabic{enumi}:}{\usecounter{enumi}%
  \setlength{\leftmargin}{2.2em}\setlength{\labelsep}{0.4em}%
  \setlength{\itemsep}{2pt}\setlength{\topsep}{3pt}\setlength{\parsep}{0pt}}
\item Initialize $\phi_0$
\item Define group size $n_S := n/S$ and level-uniform clip $C_\star := \bigl(\max_{\tau\in\mathcal T}\max(\tau,1-\tau)\bigr)\,B_x\sqrt p$
\item \textbf{for} $t = 1, \dots, T$ \textbf{do}
\item \quad Initialize gradient accumulator $\Delta \phi_t \gets 0$
\item \quad \textbf{for} $k = 1, \dots, K$ \textbf{do}
\item \quad\quad Construct minibatch $B_t^{(k)} \subseteq \{1, \dots, n\}$ by including each index $i$ with probability $q$
\item \quad\quad Generate block weights $\bu^{(k)}$ and set $w_j^{(k)} = u_{\lceil j/n_S \rceil} / n_S$
\item \quad\quad Sample a level $\tau^{(k)} \sim P_\tau$ over $\mathcal T$, independent of the data
\item \quad\quad Evaluate $\beta^{(k)} = G_{\phi}(\bu^{(k)}, \tau^{(k)})$ and residuals $r_j = y_j - X_j^\top \beta^{(k)}$
\item \quad\quad $\nabla{g}_t^{(k)} \gets \big(\sum_{j \in B_t^{(k)}} w_j^{(k)}\, \psi_{\tau^{(k)}}(r_j)\, X_j + \mathcal{N}(0, \sigma^2 \mathbf{I})\big)\, \nabla_{\phi} G_{\phi}(\bu^{(k)}, \tau^{(k)})$
\item \quad\quad $\Delta \phi_t \gets \Delta \phi_t + \nabla{g}_t^{(k)}$
\item \quad \textbf{end for}
\item \quad \textbf{Descent:}\ $\phi_{t+1} \gets \phi_t - \eta \Delta \phi_t/K$
\item \textbf{end for}
\end{list}
\end{algorithm}

\subsection{Further Simulation Results}
\label{app:extra-simulation-results}

\subsubsection{Linear-regression Summary}
\label{app:linreg-summary}

Table~\ref{tab:linreg-summary} lists the values underlying Figure~\ref{fig:linear-reg}: MAE, $95\%$ coverage, and credible-interval width by method, $n$, and $\epsilon$, for contaminated-normal and Student-$t$ noise. Standard errors across the $100$ replicates appear in parentheses for MAE and width, and are the source of the error bars in Figure~\ref{fig:linear-reg}. Tables~\ref{tab:linreg-b1}-\ref{tab:linreg-b3} break Table~\ref{tab:linreg-summary} down by individual regression coefficient, without averaging over the coordinates.

\clearpage

\begin{table}[ht!]
\centering
\scriptsize
\setlength{\tabcolsep}{3pt}
\renewcommand{\arraystretch}{0.95}
\caption{Numerics of Figure~\ref{fig:linear-reg}, over $100$ replicates at $S=20$ and $\gamma=\gamma^\star$ for PGBB. MAE is the mean absolute deviation from the true coefficient, Cov is the $95\%$ inclusion rate, and Width is the credible-interval length, each averaged over the $3$ coordinates and $100$ replicates. Standard errors across the replicates are in parentheses for MAE and Width.}
\label{tab:linreg-summary}
\begin{tabular}{c c l c c c c c c}
\toprule
\multicolumn{3}{c}{} & \multicolumn{3}{c}{Contam. normal} & \multicolumn{3}{c}{Student-$t$} \\
\cmidrule(lr){4-6} \cmidrule(lr){7-9}
$\epsilon$ & $n$ & Method & MAE (SE) & Cov & Width (SE) & MAE (SE) & Cov & Width (SE) \\
\midrule
\multirow{8}{*}{1} & \multirow{4}{*}{10{,}000} & \texttt{PGBB} & 0.0142\,(0.0006) & 0.983 & 0.104\,(0.001) & 0.0160\,(0.0006) & 0.987 & 0.107\,(0.001) \\
 &  & \texttt{DP-SGLD} & 0.0096\,(0.0004) & 1.000 & 0.738\,(0.003) & 0.0102\,(0.0004) & 1.000 & 0.749\,(0.004) \\
 &  & \texttt{NoisyGibbs} & 0.1160\,(0.0047) & 0.967 & 0.569\,(0.002) & 0.1041\,(0.0042) & 0.970 & 0.526\,(0.002) \\
 &  & \texttt{DP-Bootstrap} & 0.0878\,(0.0040) & 1.000 & 2.023\,(0.038) & 0.0898\,(0.0039) & 1.000 & 2.022\,(0.038) \\
\cmidrule(lr){2-9}
 & \multirow{4}{*}{40{,}000} & \texttt{PGBB} & 0.0062\,(0.0003) & 0.993 & 0.044\,($\approx 0$) & 0.0068\,(0.0003) & 0.997 & 0.046\,($\approx 0$) \\
 &  & \texttt{DP-SGLD} & 0.0045\,(0.0002) & 1.000 & 0.200\,($\approx 0$) & 0.0047\,(0.0002) & 1.000 & 0.205\,($\approx 0$) \\
 &  & \texttt{NoisyGibbs} & 0.0399\,(0.0016) & 0.900 & 0.161\,($\approx 0$) & 0.0311\,(0.0012) & 0.973 & 0.157\,($\approx 0$) \\
 &  & \texttt{DP-Bootstrap} & 0.0536\,(0.0015) & 0.997 & 0.564\,(0.011) & 0.0564\,(0.0016) & 1.000 & 0.565\,(0.010) \\
\midrule
\multirow{8}{*}{2} & \multirow{4}{*}{10{,}000} & \texttt{PGBB} & 0.0110\,(0.0005) & 0.983 & 0.074\,(0.001) & 0.0123\,(0.0005) & 0.983 & 0.076\,(0.001) \\
 &  & \texttt{DP-SGLD} & 0.0087\,(0.0004) & 1.000 & 0.399\,(0.001) & 0.0091\,(0.0004) & 1.000 & 0.407\,(0.001) \\
 &  & \texttt{NoisyGibbs} & 0.0630\,(0.0026) & 0.947 & 0.304\,($\approx 0$) & 0.0567\,(0.0022) & 0.970 & 0.290\,(0.001) \\
 &  & \texttt{DP-Bootstrap} & 0.0614\,(0.0026) & 1.000 & 1.103\,(0.020) & 0.0642\,(0.0027) & 1.000 & 1.101\,(0.021) \\
\cmidrule(lr){2-9}
 & \multirow{4}{*}{40{,}000} & \texttt{PGBB} & 0.0054\,(0.0002) & 0.980 & 0.033\,($\approx 0$) & 0.0060\,(0.0002) & 0.993 & 0.034\,($\approx 0$) \\
 &  & \texttt{DP-SGLD} & 0.0043\,(0.0002) & 1.000 & 0.111\,($\approx 0$) & 0.0046\,(0.0002) & 1.000 & 0.114\,($\approx 0$) \\
 &  & \texttt{NoisyGibbs} & 0.0318\,(0.0011) & 0.767 & 0.094\,($\approx 0$) & 0.0168\,(0.0007) & 0.973 & 0.090\,($\approx 0$) \\
 &  & \texttt{DP-Bootstrap} & 0.0534\,(0.0008) & 0.973 & 0.305\,(0.006) & 0.0564\,(0.0009) & 0.977 & 0.305\,(0.006) \\
\midrule
\multirow{8}{*}{4} & \multirow{4}{*}{10{,}000} & \texttt{PGBB} & 0.0098\,(0.0005) & 0.967 & 0.058\,($\approx 0$) & 0.0108\,(0.0004) & 0.973 & 0.059\,(0.001) \\
 &  & \texttt{DP-SGLD} & 0.0084\,(0.0003) & 1.000 & 0.225\,($\approx 0$) & 0.0089\,(0.0004) & 1.000 & 0.231\,($\approx 0$) \\
 &  & \texttt{NoisyGibbs} & 0.0408\,(0.0017) & 0.930 & 0.181\,($\approx 0$) & 0.0312\,(0.0012) & 0.977 & 0.172\,($\approx 0$) \\
 &  & \texttt{DP-Bootstrap} & 0.0541\,(0.0017) & 0.990 & 0.612\,(0.011) & 0.0580\,(0.0018) & 0.983 & 0.611\,(0.011) \\
\cmidrule(lr){2-9}
 & \multirow{4}{*}{40{,}000} & \texttt{PGBB} & 0.0051\,(0.0002) & 0.967 & 0.027\,($\approx 0$) & 0.0055\,(0.0002) & 0.970 & 0.029\,($\approx 0$) \\
 &  & \texttt{DP-SGLD} & 0.0043\,(0.0002) & 1.000 & 0.068\,($\approx 0$) & 0.0046\,(0.0002) & 1.000 & 0.069\,($\approx 0$) \\
 &  & \texttt{NoisyGibbs} & 0.0304\,(0.0008) & 0.557 & 0.066\,($\approx 0$) & 0.0109\,(0.0005) & 0.967 & 0.061\,($\approx 0$) \\
 &  & \texttt{DP-Bootstrap} & 0.0537\,(0.0005) & 0.847 & 0.168\,(0.003) & 0.0567\,(0.0006) & 0.827 & 0.169\,(0.003) \\
\bottomrule
\end{tabular}
\end{table}

\clearpage
\begin{table}[ht!]
\centering
\scriptsize
\setlength{\tabcolsep}{3pt}
\renewcommand{\arraystretch}{0.95}
\caption{Results of $\beta_1$ in the regression example. Refer to Table~\ref{tab:linreg-summary} for exposition.}
\label{tab:linreg-b1}
\begin{tabular}{c c l c c c c c c}
\toprule
\multicolumn{3}{c}{} & \multicolumn{3}{c}{Contam. normal} & \multicolumn{3}{c}{Student-$t$} \\
\cmidrule(lr){4-6} \cmidrule(lr){7-9}
$\epsilon$ & $n$ & Method & MAE (SE) & Cov & Width (SE) & MAE (SE) & Cov & Width (SE) \\
\midrule
\multirow{8}{*}{1} & \multirow{4}{*}{10{,}000} & \texttt{PGBB} & 0.0133\,(0.0010) & 0.990 & 0.110\,(0.002) & 0.0176\,(0.0012) & 0.990 & 0.113\,(0.002) \\
 &  & \texttt{DP-SGLD} & 0.0095\,(0.0008) & 1.000 & 0.736\,(0.004) & 0.0100\,(0.0007) & 1.000 & 0.748\,(0.004) \\
 &  & \texttt{NoisyGibbs} & 0.1087\,(0.0081) & 0.980 & 0.566\,(0.002) & 0.0941\,(0.0075) & 0.990 & 0.516\,(0.003) \\
 &  & \texttt{DP-Bootstrap} & 0.0928\,(0.0062) & 1.000 & 1.961\,(0.066) & 0.0959\,(0.0064) & 1.000 & 1.955\,(0.066) \\
\cmidrule(lr){2-9}
 & \multirow{4}{*}{40{,}000} & \texttt{PGBB} & 0.0063\,(0.0005) & 0.990 & 0.045\,(0.001) & 0.0071\,(0.0005) & 0.990 & 0.047\,(0.001) \\
 &  & \texttt{DP-SGLD} & 0.0042\,(0.0003) & 1.000 & 0.201\,($\approx 0$) & 0.0045\,(0.0003) & 1.000 & 0.205\,($\approx 0$) \\
 &  & \texttt{NoisyGibbs} & 0.0381\,(0.0030) & 0.860 & 0.160\,($\approx 0$) & 0.0287\,(0.0023) & 0.980 & 0.157\,($\approx 0$) \\
 &  & \texttt{DP-Bootstrap} & 0.0620\,(0.0026) & 1.000 & 0.564\,(0.017) & 0.0656\,(0.0026) & 1.000 & 0.565\,(0.017) \\
\midrule
\multirow{8}{*}{2} & \multirow{4}{*}{10{,}000} & \texttt{PGBB} & 0.0101\,(0.0007) & 0.990 & 0.077\,(0.001) & 0.0134\,(0.0008) & 0.990 & 0.080\,(0.001) \\
 &  & \texttt{DP-SGLD} & 0.0083\,(0.0006) & 1.000 & 0.399\,(0.001) & 0.0088\,(0.0007) & 1.000 & 0.407\,(0.001) \\
 &  & \texttt{NoisyGibbs} & 0.0582\,(0.0045) & 0.950 & 0.304\,(0.001) & 0.0520\,(0.0039) & 0.980 & 0.287\,(0.001) \\
 &  & \texttt{DP-Bootstrap} & 0.0662\,(0.0042) & 1.000 & 1.069\,(0.036) & 0.0704\,(0.0045) & 1.000 & 1.063\,(0.036) \\
\cmidrule(lr){2-9}
 & \multirow{4}{*}{40{,}000} & \texttt{PGBB} & 0.0057\,(0.0005) & 0.970 & 0.033\,($\approx 0$) & 0.0061\,(0.0004) & 0.980 & 0.036\,($\approx 0$) \\
 &  & \texttt{DP-SGLD} & 0.0041\,(0.0003) & 1.000 & 0.111\,($\approx 0$) & 0.0044\,(0.0003) & 1.000 & 0.114\,($\approx 0$) \\
 &  & \texttt{NoisyGibbs} & 0.0337\,(0.0019) & 0.770 & 0.094\,($\approx 0$) & 0.0156\,(0.0013) & 0.980 & 0.090\,($\approx 0$) \\
 &  & \texttt{DP-Bootstrap} & 0.0615\,(0.0015) & 0.970 & 0.304\,(0.009) & 0.0652\,(0.0015) & 0.960 & 0.305\,(0.009) \\
\midrule
\multirow{8}{*}{4} & \multirow{4}{*}{10{,}000} & \texttt{PGBB} & 0.0088\,(0.0007) & 0.990 & 0.059\,(0.001) & 0.0113\,(0.0007) & 0.990 & 0.062\,(0.001) \\
 &  & \texttt{DP-SGLD} & 0.0079\,(0.0006) & 1.000 & 0.225\,($\approx 0$) & 0.0088\,(0.0007) & 1.000 & 0.230\,($\approx 0$) \\
 &  & \texttt{NoisyGibbs} & 0.0377\,(0.0030) & 0.940 & 0.180\,($\approx 0$) & 0.0287\,(0.0022) & 0.990 & 0.171\,($\approx 0$) \\
 &  & \texttt{DP-Bootstrap} & 0.0597\,(0.0029) & 0.980 & 0.593\,(0.020) & 0.0643\,(0.0031) & 0.970 & 0.587\,(0.020) \\
\cmidrule(lr){2-9}
 & \multirow{4}{*}{40{,}000} & \texttt{PGBB} & 0.0053\,(0.0004) & 0.950 & 0.028\,($\approx 0$) & 0.0055\,(0.0004) & 0.970 & 0.030\,($\approx 0$) \\
 &  & \texttt{DP-SGLD} & 0.0041\,(0.0003) & 1.000 & 0.068\,($\approx 0$) & 0.0044\,(0.0003) & 1.000 & 0.069\,($\approx 0$) \\
 &  & \texttt{NoisyGibbs} & 0.0336\,(0.0014) & 0.510 & 0.066\,($\approx 0$) & 0.0105\,(0.0008) & 0.980 & 0.061\,($\approx 0$) \\
 &  & \texttt{DP-Bootstrap} & 0.0613\,(0.0009) & 0.800 & 0.167\,(0.005) & 0.0649\,(0.0009) & 0.770 & 0.169\,(0.005) \\
\bottomrule
\end{tabular}
\end{table}

\clearpage
\begin{table}[ht!]
\centering
\scriptsize
\setlength{\tabcolsep}{3pt}
\renewcommand{\arraystretch}{0.95}
\caption{Results of $\beta_2$ in the regression example. Refer to Table~\ref{tab:linreg-summary} for exposition.}
\label{tab:linreg-b2}
\begin{tabular}{c c l c c c c c c}
\toprule
\multicolumn{3}{c}{} & \multicolumn{3}{c}{Contam. normal} & \multicolumn{3}{c}{Student-$t$} \\
\cmidrule(lr){4-6} \cmidrule(lr){7-9}
$\epsilon$ & $n$ & Method & MAE (SE) & Cov & Width (SE) & MAE (SE) & Cov & Width (SE) \\
\midrule
\multirow{8}{*}{1} & \multirow{4}{*}{10{,}000} & \texttt{PGBB} & 0.0163\,(0.0012) & 0.970 & 0.110\,(0.002) & 0.0149\,(0.0011) & 0.990 & 0.115\,(0.002) \\
 &  & \texttt{DP-SGLD} & 0.0090\,(0.0006) & 1.000 & 0.740\,(0.004) & 0.0106\,(0.0007) & 1.000 & 0.751\,(0.004) \\
 &  & \texttt{NoisyGibbs} & 0.1148\,(0.0081) & 0.980 & 0.567\,(0.002) & 0.1080\,(0.0076) & 0.980 & 0.518\,(0.003) \\
 &  & \texttt{DP-Bootstrap} & 0.0889\,(0.0069) & 1.000 & 2.080\,(0.061) & 0.0913\,(0.0070) & 1.000 & 2.083\,(0.061) \\
\cmidrule(lr){2-9}
 & \multirow{4}{*}{40{,}000} & \texttt{PGBB} & 0.0064\,(0.0005) & 0.990 & 0.046\,(0.001) & 0.0066\,(0.0005) & 1.000 & 0.047\,(0.001) \\
 &  & \texttt{DP-SGLD} & 0.0050\,(0.0004) & 1.000 & 0.200\,($\approx 0$) & 0.0043\,(0.0003) & 1.000 & 0.205\,($\approx 0$) \\
 &  & \texttt{NoisyGibbs} & 0.0443\,(0.0030) & 0.890 & 0.161\,($\approx 0$) & 0.0322\,(0.0021) & 0.990 & 0.157\,($\approx 0$) \\
 &  & \texttt{DP-Bootstrap} & 0.0603\,(0.0028) & 1.000 & 0.577\,(0.016) & 0.0630\,(0.0028) & 1.000 & 0.577\,(0.016) \\
\midrule
\multirow{8}{*}{2} & \multirow{4}{*}{10{,}000} & \texttt{PGBB} & 0.0118\,(0.0008) & 1.000 & 0.077\,(0.001) & 0.0109\,(0.0008) & 1.000 & 0.079\,(0.001) \\
 &  & \texttt{DP-SGLD} & 0.0083\,(0.0006) & 1.000 & 0.400\,(0.001) & 0.0094\,(0.0007) & 1.000 & 0.408\,(0.001) \\
 &  & \texttt{NoisyGibbs} & 0.0648\,(0.0047) & 0.930 & 0.304\,(0.001) & 0.0590\,(0.0040) & 0.970 & 0.288\,(0.001) \\
 &  & \texttt{DP-Bootstrap} & 0.0660\,(0.0046) & 1.000 & 1.134\,(0.033) & 0.0693\,(0.0047) & 1.000 & 1.137\,(0.033) \\
\cmidrule(lr){2-9}
 & \multirow{4}{*}{40{,}000} & \texttt{PGBB} & 0.0054\,(0.0004) & 0.990 & 0.033\,(0.001) & 0.0057\,(0.0004) & 1.000 & 0.035\,($\approx 0$) \\
 &  & \texttt{DP-SGLD} & 0.0049\,(0.0003) & 1.000 & 0.111\,($\approx 0$) & 0.0040\,(0.0003) & 1.000 & 0.114\,($\approx 0$) \\
 &  & \texttt{NoisyGibbs} & 0.0370\,(0.0020) & 0.630 & 0.094\,($\approx 0$) & 0.0175\,(0.0012) & 0.970 & 0.090\,($\approx 0$) \\
 &  & \texttt{DP-Bootstrap} & 0.0606\,(0.0016) & 0.980 & 0.312\,(0.009) & 0.0633\,(0.0016) & 1.000 & 0.312\,(0.009) \\
\midrule
\multirow{8}{*}{4} & \multirow{4}{*}{10{,}000} & \texttt{PGBB} & 0.0099\,(0.0007) & 0.990 & 0.059\,(0.001) & 0.0098\,(0.0007) & 0.990 & 0.061\,(0.001) \\
 &  & \texttt{DP-SGLD} & 0.0081\,(0.0006) & 1.000 & 0.226\,($\approx 0$) & 0.0091\,(0.0006) & 1.000 & 0.231\,($\approx 0$) \\
 &  & \texttt{NoisyGibbs} & 0.0461\,(0.0031) & 0.900 & 0.180\,($\approx 0$) & 0.0336\,(0.0023) & 0.960 & 0.171\,($\approx 0$) \\
 &  & \texttt{DP-Bootstrap} & 0.0614\,(0.0030) & 1.000 & 0.630\,(0.018) & 0.0648\,(0.0031) & 1.000 & 0.633\,(0.019) \\
\cmidrule(lr){2-9}
 & \multirow{4}{*}{40{,}000} & \texttt{PGBB} & 0.0052\,(0.0004) & 0.970 & 0.027\,($\approx 0$) & 0.0053\,(0.0004) & 0.970 & 0.030\,($\approx 0$) \\
 &  & \texttt{DP-SGLD} & 0.0049\,(0.0003) & 1.000 & 0.068\,($\approx 0$) & 0.0040\,(0.0003) & 1.000 & 0.069\,($\approx 0$) \\
 &  & \texttt{NoisyGibbs} & 0.0363\,(0.0014) & 0.360 & 0.066\,($\approx 0$) & 0.0114\,(0.0008) & 0.950 & 0.061\,($\approx 0$) \\
 &  & \texttt{DP-Bootstrap} & 0.0607\,(0.0010) & 0.840 & 0.173\,(0.005) & 0.0634\,(0.0010) & 0.830 & 0.173\,(0.005) \\
\bottomrule
\end{tabular}
\end{table}

\clearpage
\begin{table}[ht!]
\centering
\scriptsize
\setlength{\tabcolsep}{3pt}
\renewcommand{\arraystretch}{0.95}
\caption{Results of $\beta_3$ in the regression example. Refer to Table~\ref{tab:linreg-summary} for exposition.}
\label{tab:linreg-b3}
\begin{tabular}{c c l c c c c c c}
\toprule
\multicolumn{3}{c}{} & \multicolumn{3}{c}{Contam. normal} & \multicolumn{3}{c}{Student-$t$} \\
\cmidrule(lr){4-6} \cmidrule(lr){7-9}
$\epsilon$ & $n$ & Method & MAE (SE) & Cov & Width (SE) & MAE (SE) & Cov & Width (SE) \\
\midrule
\multirow{8}{*}{1} & \multirow{4}{*}{10{,}000} & \texttt{PGBB} & 0.0131\,(0.0010) & 0.990 & 0.093\,(0.001) & 0.0156\,(0.0011) & 0.980 & 0.094\,(0.001) \\
 &  & \texttt{DP-SGLD} & 0.0101\,(0.0007) & 1.000 & 0.737\,(0.004) & 0.0101\,(0.0008) & 1.000 & 0.748\,(0.004) \\
 &  & \texttt{NoisyGibbs} & 0.1246\,(0.0095) & 0.940 & 0.574\,(0.002) & 0.1102\,(0.0084) & 0.940 & 0.544\,(0.003) \\
 &  & \texttt{DP-Bootstrap} & 0.0817\,(0.0057) & 1.000 & 2.028\,(0.064) & 0.0822\,(0.0059) & 1.000 & 2.027\,(0.064) \\
\cmidrule(lr){2-9}
 & \multirow{4}{*}{40{,}000} & \texttt{PGBB} & 0.0059\,(0.0004) & 1.000 & 0.042\,(0.001) & 0.0067\,(0.0005) & 1.000 & 0.043\,(0.001) \\
 &  & \texttt{DP-SGLD} & 0.0041\,(0.0003) & 1.000 & 0.200\,($\approx 0$) & 0.0053\,(0.0003) & 1.000 & 0.205\,($\approx 0$) \\
 &  & \texttt{NoisyGibbs} & 0.0373\,(0.0026) & 0.950 & 0.161\,($\approx 0$) & 0.0325\,(0.0026) & 0.950 & 0.158\,($\approx 0$) \\
 &  & \texttt{DP-Bootstrap} & 0.0387\,(0.0022) & 0.990 & 0.552\,(0.020) & 0.0405\,(0.0024) & 1.000 & 0.553\,(0.020) \\
\midrule
\multirow{8}{*}{2} & \multirow{4}{*}{10{,}000} & \texttt{PGBB} & 0.0112\,(0.0009) & 0.960 & 0.069\,(0.001) & 0.0127\,(0.0009) & 0.960 & 0.070\,(0.001) \\
 &  & \texttt{DP-SGLD} & 0.0094\,(0.0007) & 1.000 & 0.399\,(0.001) & 0.0091\,(0.0007) & 1.000 & 0.407\,(0.001) \\
 &  & \texttt{NoisyGibbs} & 0.0660\,(0.0048) & 0.960 & 0.305\,(0.001) & 0.0590\,(0.0045) & 0.960 & 0.295\,(0.001) \\
 &  & \texttt{DP-Bootstrap} & 0.0521\,(0.0036) & 1.000 & 1.105\,(0.035) & 0.0529\,(0.0039) & 1.000 & 1.104\,(0.035) \\
\cmidrule(lr){2-9}
 & \multirow{4}{*}{40{,}000} & \texttt{PGBB} & 0.0052\,(0.0004) & 0.980 & 0.032\,($\approx 0$) & 0.0061\,(0.0004) & 1.000 & 0.032\,(0.001) \\
 &  & \texttt{DP-SGLD} & 0.0039\,(0.0003) & 1.000 & 0.111\,($\approx 0$) & 0.0053\,(0.0003) & 1.000 & 0.114\,($\approx 0$) \\
 &  & \texttt{NoisyGibbs} & 0.0246\,(0.0017) & 0.900 & 0.094\,($\approx 0$) & 0.0173\,(0.0013) & 0.970 & 0.091\,($\approx 0$) \\
 &  & \texttt{DP-Bootstrap} & 0.0382\,(0.0014) & 0.970 & 0.298\,(0.011) & 0.0408\,(0.0015) & 0.970 & 0.298\,(0.011) \\
\midrule
\multirow{8}{*}{4} & \multirow{4}{*}{10{,}000} & \texttt{PGBB} & 0.0107\,(0.0009) & 0.920 & 0.055\,(0.001) & 0.0113\,(0.0008) & 0.940 & 0.055\,(0.001) \\
 &  & \texttt{DP-SGLD} & 0.0092\,(0.0007) & 1.000 & 0.225\,($\approx 0$) & 0.0089\,(0.0007) & 1.000 & 0.231\,($\approx 0$) \\
 &  & \texttt{NoisyGibbs} & 0.0386\,(0.0027) & 0.950 & 0.181\,($\approx 0$) & 0.0314\,(0.0024) & 0.980 & 0.173\,($\approx 0$) \\
 &  & \texttt{DP-Bootstrap} & 0.0413\,(0.0026) & 0.990 & 0.613\,(0.019) & 0.0447\,(0.0026) & 0.980 & 0.613\,(0.019) \\
\cmidrule(lr){2-9}
 & \multirow{4}{*}{40{,}000} & \texttt{PGBB} & 0.0048\,(0.0004) & 0.980 & 0.027\,($\approx 0$) & 0.0057\,(0.0004) & 0.970 & 0.027\,($\approx 0$) \\
 &  & \texttt{DP-SGLD} & 0.0039\,(0.0003) & 1.000 & 0.068\,($\approx 0$) & 0.0053\,(0.0003) & 1.000 & 0.069\,($\approx 0$) \\
 &  & \texttt{NoisyGibbs} & 0.0213\,(0.0013) & 0.800 & 0.066\,($\approx 0$) & 0.0108\,(0.0008) & 0.970 & 0.061\,($\approx 0$) \\
 &  & \texttt{DP-Bootstrap} & 0.0390\,(0.0009) & 0.900 & 0.165\,(0.006) & 0.0416\,(0.0010) & 0.880 & 0.165\,(0.006) \\
\bottomrule
\end{tabular}
\end{table}

\subsubsection{Effect of the Monte Carlo Replicate Count $K$}
\label{app:mc-K-effect}

On the linear-regression setup with $S=20$ and $\gamma=\gamma^\star(n,S)$, we compare the results of $K=10$ and $K=1$ over $\epsilon\in\{1,2,4\}$, $n\in\{10{\text{k}},40{\text{k}}\}$, contaminated-normal and Student-$t$ noise. The algorithm runs for $100$ replicates independently. Table~\ref{tab:K-comparison} summarizes the evaluation metrics for each combination. As supported in Theorem~\ref{thm:conv-gbb}, increasing $K$ hardly degrades the estimation performance even though the $TK$ number of compositions magnifies $\sigma^2\propto TK$.
\begin{table}[ht!]
\centering
\footnotesize
\setlength{\tabcolsep}{4pt}
\setlength{\tabcolsep}{3pt}
\caption{PGBB linear regression at $S=20$, $\gamma=\gamma^\star$, both at $T=10\text{k}$: $K=1$ versus $K=10$ Monte Carlo replicates, averaged over $100$ replicates and $p=3$ coordinates per cell. MAE is the mean absolute deviation of the posterior average from the true coefficient, Coverage is the $95\%$ inclusion rate, and Width is the credible-interval length. Standard errors across the $100$ replicates are in parentheses for MAE and Width.}
\label{tab:K-comparison}
\begin{tabular}{ll c cc cc cc}
\toprule
 & & & \multicolumn{2}{c}{MAE} & \multicolumn{2}{c}{Coverage} & \multicolumn{2}{c}{Width} \\
\cmidrule(lr){4-5}\cmidrule(lr){6-7}\cmidrule(lr){8-9}
Noise & $n$ & $\epsilon$ & $K{=}1$ & $K{=}10$ & $K{=}1$ & $K{=}10$ & $K{=}1$ & $K{=}10$ \\
\midrule
\multirow{6}{*}{Contam.\ normal} & \multirow{3}{*}{10k} & 1 & 0.0142 (0.0006) & 0.0129 (0.0006) & 0.983 & 0.993 & 0.104 (0.001) & 0.103 (0.001) \\
& & 2 & 0.0110 (0.0005) & 0.0104 (0.0005) & 0.983 & 0.987 & 0.074 (0.001) & 0.073 (0.001) \\
& & 4 & 0.0098 (0.0005) & 0.0094 (0.0004) & 0.967 & 0.973 & 0.058 ($\approx 0$) & 0.057 ($\approx 0$) \\
\cmidrule(lr){2-9}
& \multirow{3}{*}{40k} & 1 & 0.0062 (0.0003) & 0.0060 (0.0003) & 0.993 & 0.983 & 0.044 ($\approx 0$) & 0.044 ($\approx 0$) \\
& & 2 & 0.0054 (0.0002) & 0.0054 (0.0002) & 0.980 & 0.980 & 0.033 ($\approx 0$) & 0.032 ($\approx 0$) \\
& & 4 & 0.0051 (0.0002) & 0.0051 (0.0002) & 0.967 & 0.973 & 0.027 ($\approx 0$) & 0.027 ($\approx 0$) \\
\midrule
\multirow{6}{*}{Student-$t$} & \multirow{3}{*}{10k} & 1 & 0.0160 (0.0006) & 0.0142 (0.0006) & 0.987 & 0.993 & 0.107 (0.001) & 0.105 (0.001) \\
& & 2 & 0.0123 (0.0004) & 0.0116 (0.0005) & 0.983 & 0.987 & 0.076 (0.001) & 0.075 (0.001) \\
& & 4 & 0.0108 (0.0004) & 0.0105 (0.0004) & 0.973 & 0.973 & 0.059 (0.001) & 0.059 (0.001) \\
\cmidrule(lr){2-9}
& \multirow{3}{*}{40k} & 1 & 0.0068 (0.0003) & 0.0063 (0.0003) & 0.997 & 0.987 & 0.046 ($\approx 0$) & 0.045 ($\approx 0$) \\
& & 2 & 0.0060 (0.0002) & 0.0055 (0.0002) & 0.993 & 0.977 & 0.034 ($\approx 0$) & 0.034 ($\approx 0$) \\
& & 4 & 0.0055 (0.0002) & 0.0052 (0.0002) & 0.970 & 0.953 & 0.029 ($\approx 0$) & 0.028 ($\approx 0$) \\
\bottomrule
\end{tabular}
\end{table}

\subsubsection{Effective Sample Size of Competing Methods}
\label{supp:ess-comp}

Table~\ref{tab:ess-expo} reports the effective sample sizes (ESS) for both linear-regression and exponential-family examples. Both methods discard the first half as burn-in. The total number of iterations is tuned to reach an effective sample size at least $500$. 

\begin{table}[h]
    \centering
    \scriptsize
    \caption{Effective sample size (ESS) of the two MCMC samples for $\epsilon\in\{1,2,4\}$. The number in parentheses is the ESS per $1000$ post-burn-in iterations. The ``(A/B)'' column reports the \emph{total} number of MCMC iterations for each one. For both methods, the first half is burned. The chain lengths are chosen so that DP-SGLD attains $\mathrm{ESS}\ge 500$ in every cell.}
    \begin{tabular}{ll c ccc ccc}
    \toprule
    & & & \multicolumn{3}{c}{DP-SGLD} & \multicolumn{3}{c}{NoisyGibbs} \\
    \cmidrule(lr){4-6}\cmidrule(lr){7-9}
    Example & Case & (SGLD/NG) & $\varepsilon{=}1$ & $\varepsilon{=}2$ & $\varepsilon{=}4$ & $\varepsilon{=}1$ & $\varepsilon{=}2$ & $\varepsilon{=}4$ \\
    \midrule
\multirow{4}{*}{\shortstack[l]{ExpFam.}} & Exp, $n{=}500$ & 40k / 4k & 522\,(26) & 562\,(28) & 678\,(34) & 644\,(322) & 1292\,(646) & 1716\,(858) \\
 & Exp $n=\text{5k}$ & 40k / 4k & 688\,(34) & 666\,(33) & 660\,(33) & 1527\,(763) & 1812\,(906) & 1902\,(951) \\
 & Gauss, $n{=}500$ & 40k / 4k & 709\,(35) & 676\,(34) & 676\,(34) & 98\,(49) & 325\,(163) & 862\,(431) \\
 & Gauss, $n=\text{5k}$ & 40k / 4k & 668\,(33) & 669\,(33) & 671\,(34) & 521\,(261) & 1154\,(577) & 1653\,(826) \\
\midrule
\multirow{2}{*}{\shortstack[l]{LinearReg}} & $n=\text{10k}$ & 20k / 400k & 991\,(99) & 1000\,(100) & 1018\,(102) & 761\,(4) & 2485\,(12) & 7150\,(36) \\
 & $n=\text{40k}$ & 20k / 400k & 1017\,(102) & 1020\,(102) & 1018\,(102) & 2164\,(11) & 6564\,(33) & 12518\,(63) \\
    \bottomrule
    \end{tabular}
    \label{tab:ess-expo}
\end{table}

\subsubsection{Further Results of Real Data Application}
\label{app:further-result-realdata}

Tables~\ref{tab:supp-acfv} and~\ref{tab:supp-natality} extend Table~\ref{tab:main_bias_unc} to all covariates. Across both applications the two methods attain comparable MAE, while PGBB produces the narrower credible interval for every coefficient. On the U.S.\ Census (Table~\ref{tab:supp-acfv}), PGBB's intervals are several times tighter than DP-SGLD's at similar MAE. Its coverage, the inclusion rate of the non-private point estimate, is close to nominal for \texttt{black} but lower for the returns to schooling, \texttt{educ} and \texttt{exper}, whose non-private signal is small relative to the privacy noise. DP-SGLD is much wider and over-covers throughout. Figure~\ref{fig:acfv_allvars} plots the released coefficients across $\tau$ for all covariates and both years. On the natality data (Table~\ref{tab:supp-natality}), the same pattern holds. PGBB gives the tightest intervals for all thirteen covariates and near-nominal coverage for the well-identified ones such as \texttt{BLACK}, \texttt{MARRIED}, and \texttt{BOY}. Its coverage is lower for \texttt{NOSMOKE} and \texttt{AGE\_s}, improving as $\epsilon$ grows, and lowest for the weakly identified \texttt{CIGAR\_s}. DP-SGLD attains slightly smaller MAE on several coefficients but at two-to-three times the interval width, and it over-covers almost everywhere. Figure~\ref{fig:natality_allvars} plots the released coefficients across $\tau$ for all thirteen covariates.

The two datasets show the same trade-off. From a single quantile loss with no scale parameter, PGBB tends to give the smaller and better-calibrated intervals at MAE competitive with DP-SGLD. Its coverage is near nominal where the coefficient is well identified and lower where the non-private signal is small relative to the privacy noise. DP-SGLD is governed by $\omega$. Lowering it to $\omega=0.1$ widens the intervals and raises coverage on the well-identified coefficients, but even at $\omega=1$ its intervals stay wider than PGBB's while over-covering. The weakly identified coefficients, such as \texttt{CIGAR\_s} and \texttt{exper}, carry little non-private signal, so the privacy noise dominates them.

\clearpage
\begin{table}[H]
\centering
\tiny
\setlength{\tabcolsep}{4pt}
\caption{Numerics for the U.S.\ Census returns-to-schooling application, over $\tau$ and $10$ replicates, for all covariates by year. MAE is the absolute deviation of the posterior average from the non-private reference, Coverage is the fraction of $\tau$ whose non-private reference lies in the $95\%$ interval, and Width is the $95\%$ interval length. Standard errors across the replicates are in parentheses for MAE and Width. The two DP-SGLD rows use asymmetric-Laplace scale $\omega\in\{1,0.1\}$.}
\label{tab:supp-acfv}
\resizebox{\textwidth}{!}{%
\begin{tabular}{lll ccc ccc}
\toprule
& & & \multicolumn{3}{c}{$\epsilon=1$} & \multicolumn{3}{c}{$\epsilon=4$} \\
\cmidrule(lr){4-6}\cmidrule(lr){7-9}
Year & Variable & Method & MAE (SE) & Cov & Width (SE) & MAE (SE) & Cov & Width (SE) \\
\midrule
\multirow{9}{*}{1980} & \multirow{3}{*}{\texttt{educ}} & PGBB & 0.52\,(0.09) & 0.700 & 1.29\,(0.09) & 0.17\,(0.04) & 0.840 & 0.77\,(0.02) \\
& & $\omega{=}1$ & 0.98\,(0.20) & 1.000 & 10.91\,(0.44) & 0.30\,(0.05) & 1.000 & 3.16\,(0.12) \\
& & $\omega{=}0.1$ & 1.05\,(0.19) & 1.000 & 11.11\,(0.45) & 0.37\,(0.06) & 1.000 & 4.30\,(0.13) \\
\cmidrule(lr){2-9}
& \multirow{3}{*}{\texttt{exper}} & PGBB & 0.0038\,(0.0007) & 0.620 & 0.0074\,(0.0009) & 0.0012\,(0.0002) & 0.800 & 0.0050\,(0.0004) \\
& & $\omega{=}1$ & 0.0073\,(0.0013) & 1.000 & 0.0787\,(0.0039) & 0.0021\,(0.0004) & 1.000 & 0.0228\,(0.0010) \\
& & $\omega{=}0.1$ & 0.0077\,(0.0014) & 1.000 & 0.0804\,(0.0038) & 0.0027\,(0.0006) & 1.000 & 0.0302\,(0.0011) \\
\cmidrule(lr){2-9}
& \multirow{3}{*}{\texttt{black}} & PGBB & 0.023\,(0.002) & 0.900 & 0.090\,(0.009) & 0.013\,(0.001) & 0.840 & 0.049\,(0.004) \\
& & $\omega{=}1$ & 0.082\,(0.004) & 1.000 & 1.083\,(0.032) & 0.012\,(0.001) & 1.000 & 0.279\,(0.005) \\
& & $\omega{=}0.1$ & 0.087\,(0.004) & 1.000 & 1.140\,(0.035) & 0.017\,(0.001) & 1.000 & 0.389\,(0.008) \\
\midrule
\multirow{9}{*}{2000} & \multirow{3}{*}{\texttt{educ}} & PGBB & 0.69\,(0.11) & 0.680 & 1.70\,(0.11) & 0.30\,(0.03) & 0.820 & 0.99\,(0.06) \\
& & $\omega{=}1$ & 0.51\,(0.09) & 1.000 & 9.26\,(0.48) & 0.17\,(0.02) & 1.000 & 2.79\,(0.16) \\
& & $\omega{=}0.1$ & 0.49\,(0.10) & 1.000 & 9.59\,(0.54) & 0.22\,(0.05) & 1.000 & 4.04\,(0.22) \\
\cmidrule(lr){2-9}
& \multirow{3}{*}{\texttt{exper}} & PGBB & 0.0034\,(0.0006) & 0.640 & 0.0082\,(0.0012) & 0.0016\,(0.0002) & 0.780 & 0.0048\,(0.0004) \\
& & $\omega{=}1$ & 0.0034\,(0.0008) & 1.000 & 0.0616\,(0.0021) & 0.0011\,(0.0002) & 1.000 & 0.0183\,(0.0008) \\
& & $\omega{=}0.1$ & 0.0033\,(0.0008) & 1.000 & 0.0639\,(0.0028) & 0.0014\,(0.0003) & 1.000 & 0.0267\,(0.0014) \\
\cmidrule(lr){2-9}
& \multirow{3}{*}{\texttt{black}} & PGBB & 0.019\,(0.002) & 0.700 & 0.072\,(0.005) & 0.008\,(0.001) & 0.920 & 0.040\,(0.002) \\
& & $\omega{=}1$ & 0.031\,(0.001) & 1.000 & 0.697\,(0.008) & 0.004\,($\approx 0$) & 1.000 & 0.203\,(0.003) \\
& & $\omega{=}0.1$ & 0.034\,(0.001) & 1.000 & 0.747\,(0.013) & 0.006\,(0.001) & 1.000 & 0.307\,(0.006) \\
\bottomrule
\end{tabular}
}
\end{table}

\clearpage
\begin{table}[H]
\centering
\tiny
\setlength{\tabcolsep}{4pt}
\renewcommand{\arraystretch}{0.9}
\caption{Numerics for the natality (2024 birthweight) application in grams, over $\tau$ and $10$ replicates, for all covariates. MAE is the absolute deviation of the posterior average from the non-private reference, Coverage is the fraction of $\tau$ whose non-private reference lies in the $95\%$ interval, and Width is the $95\%$ interval length. Standard errors across the replicates are in parentheses for MAE and Width. The two DP-SGLD rows use asymmetric-Laplace scale $\omega\in\{1,0.1\}$.}
\label{tab:supp-natality}
\begin{tabular}{ll ccc ccc}
\toprule
& & \multicolumn{3}{c}{$\epsilon=1$} & \multicolumn{3}{c}{$\epsilon=4$} \\
\cmidrule(lr){3-5}\cmidrule(lr){6-8}
Variable & Method & MAE (SE) & Cov & Width (SE) & MAE (SE) & Cov & Width (SE) \\
\midrule
\multirow{3}{*}{\texttt{BLACK}} & PGBB & 2.0\,(0.3) & 0.980 & 17.5\,(1.0) & 1.0\,(0.2) & 1.000 & 11.2\,(0.3) \\
& $\omega{=}1$ & 0.7\,(0.1) & 1.000 & 47.1\,(0.5) & 0.3\,($\approx 0$) & 1.000 & 17.2\,(0.2) \\
& $\omega{=}0.1$ & 0.8\,(0.2) & 1.000 & 57.6\,(0.6) & 0.6\,(0.1) & 1.000 & 37.0\,(0.3) \\
\addlinespace[1pt]
\multirow{3}{*}{\texttt{NOSMOKE}} & PGBB & 16.5\,(2.7) & 0.820 & 58.7\,(6.7) & 6.4\,(1.0) & 0.960 & 38.1\,(2.5) \\
& $\omega{=}1$ & 21.2\,(5.1) & 0.980 & 109.4\,(3.1) & 6.5\,(1.3) & 1.000 & 40.7\,(1.3) \\
& $\omega{=}0.1$ & 22.8\,(4.6) & 1.000 & 137.6\,(4.3) & 11.8\,(1.8) & 1.000 & 91.2\,(3.9) \\
\addlinespace[1pt]
\multirow{3}{*}{\texttt{MARRIED}} & PGBB & 2.5\,(0.2) & 1.000 & 16.1\,(0.5) & 0.9\,(0.1) & 1.000 & 10.7\,(0.3) \\
& $\omega{=}1$ & 0.5\,(0.1) & 1.000 & 43.6\,(0.5) & 0.3\,($\approx 0$) & 1.000 & 15.9\,(0.2) \\
& $\omega{=}0.1$ & 0.7\,(0.1) & 1.000 & 53.3\,(0.7) & 0.6\,(0.1) & 1.000 & 34.4\,(0.3) \\
\addlinespace[1pt]
\multirow{3}{*}{\texttt{BOY}} & PGBB & 1.1\,(0.1) & 1.000 & 12.9\,(0.7) & 0.7\,(0.1) & 1.000 & 8.4\,(0.2) \\
& $\omega{=}1$ & 0.4\,(0.1) & 1.000 & 34.0\,(0.2) & 0.2\,($\approx 0$) & 1.000 & 12.5\,(0.1) \\
& $\omega{=}0.1$ & 0.5\,(0.1) & 1.000 & 41.3\,(0.4) & 0.3\,($\approx 0$) & 1.000 & 26.5\,(0.2) \\
\addlinespace[1pt]
\multirow{3}{*}{\texttt{AGE\_s}} & PGBB & 8.3\,(1.1) & 0.880 & 33.5\,(2.0) & 2.9\,(0.5) & 0.940 & 20.5\,(0.9) \\
& $\omega{=}1$ & 3.6\,(0.8) & 1.000 & 94.8\,(1.9) & 1.4\,(0.2) & 1.000 & 34.6\,(0.6) \\
& $\omega{=}0.1$ & 4.5\,(0.9) & 1.000 & 117.5\,(1.9) & 2.8\,(0.4) & 1.000 & 75.0\,(1.6) \\
\addlinespace[1pt]
\multirow{3}{*}{\texttt{CIGAR\_s}} & PGBB & 98.3\,(9.4) & 0.560 & 204.2\,(18.8) & 43.8\,(6.8) & 0.820 & 168.3\,(17.7) \\
& $\omega{=}1$ & 125.0\,(31.0) & 0.820 & 371.8\,(26.7) & 37.3\,(9.2) & 0.840 & 136.8\,(5.6) \\
& $\omega{=}0.1$ & 130.0\,(31.7) & 0.840 & 463.6\,(19.4) & 68.9\,(9.7) & 0.920 & 332.1\,(19.6) \\
\addlinespace[1pt]
\multirow{3}{*}{\texttt{HSGRAD}} & PGBB & 6.3\,(0.9) & 0.900 & 33.4\,(3.3) & 2.7\,(0.4) & 0.980 & 18.7\,(1.2) \\
& $\omega{=}1$ & 2.5\,(0.6) & 1.000 & 73.8\,(1.6) & 0.9\,(0.2) & 1.000 & 26.3\,(0.5) \\
& $\omega{=}0.1$ & 3.0\,(0.7) & 1.000 & 89.2\,(1.8) & 1.8\,(0.2) & 1.000 & 57.5\,(1.1) \\
\addlinespace[1pt]
\multirow{3}{*}{\texttt{SOMECOLL}} & PGBB & 6.6\,(0.8) & 0.940 & 33.1\,(2.8) & 2.5\,(0.4) & 0.960 & 18.3\,(1.1) \\
& $\omega{=}1$ & 2.5\,(0.6) & 1.000 & 75.3\,(1.5) & 0.9\,(0.2) & 1.000 & 27.1\,(0.5) \\
& $\omega{=}0.1$ & 2.9\,(0.6) & 1.000 & 92.1\,(1.8) & 1.5\,(0.3) & 1.000 & 58.5\,(1.1) \\
\addlinespace[1pt]
\multirow{3}{*}{\texttt{COLLGRAD}} & PGBB & 7.1\,(1.0) & 0.920 & 33.5\,(3.1) & 2.5\,(0.3) & 1.000 & 19.8\,(1.1) \\
& $\omega{=}1$ & 2.6\,(0.6) & 1.000 & 77.3\,(1.4) & 0.9\,(0.2) & 1.000 & 27.8\,(0.5) \\
& $\omega{=}0.1$ & 2.8\,(0.7) & 1.000 & 94.7\,(1.7) & 1.4\,(0.3) & 1.000 & 60.5\,(1.2) \\
\addlinespace[1pt]
\multirow{3}{*}{\texttt{WTGAIN\_s}} & PGBB & 7.5\,(1.0) & 0.980 & 55.8\,(4.3) & 3.3\,(0.3) & 1.000 & 29.3\,(1.2) \\
& $\omega{=}1$ & 3.1\,(0.6) & 1.000 & 102.9\,(1.4) & 1.1\,(0.2) & 1.000 & 37.6\,(0.8) \\
& $\omega{=}0.1$ & 3.1\,(0.7) & 1.000 & 127.8\,(2.5) & 2.2\,(0.5) & 1.000 & 83.8\,(2.1) \\
\addlinespace[1pt]
\multirow{3}{*}{\texttt{NATAL2}} & PGBB & 2.2\,(0.3) & 1.000 & 16.9\,(0.6) & 1.2\,(0.1) & 1.000 & 11.2\,(0.3) \\
& $\omega{=}1$ & 0.9\,(0.1) & 1.000 & 48.0\,(0.6) & 0.4\,($\approx 0$) & 1.000 & 17.4\,(0.2) \\
& $\omega{=}0.1$ & 1.1\,(0.2) & 1.000 & 58.5\,(0.8) & 0.5\,(0.1) & 1.000 & 38.2\,(0.4) \\
\addlinespace[1pt]
\multirow{3}{*}{\texttt{NATAL3}} & PGBB & 4.4\,(0.4) & 0.960 & 31.6\,(1.9) & 2.0\,(0.3) & 1.000 & 18.1\,(0.4) \\
& $\omega{=}1$ & 2.1\,(0.6) & 1.000 & 87.8\,(1.6) & 0.8\,(0.2) & 1.000 & 31.1\,(0.4) \\
& $\omega{=}0.1$ & 2.6\,(0.6) & 1.000 & 106.3\,(1.4) & 1.5\,(0.4) & 1.000 & 67.8\,(1.4) \\
\addlinespace[1pt]
\multirow{3}{*}{\texttt{NOVISIT}} & PGBB & 13.5\,(0.6) & 0.860 & 55.7\,(4.7) & 5.3\,(0.4) & 0.980 & 30.0\,(1.5) \\
& $\omega{=}1$ & 6.9\,(1.4) & 1.000 & 140.5\,(4.0) & 2.7\,(0.4) & 1.000 & 50.6\,(2.0) \\
& $\omega{=}0.1$ & 9.3\,(1.6) & 1.000 & 173.5\,(6.6) & 5.3\,(1.2) & 1.000 & 107.1\,(3.8) \\
\bottomrule
\end{tabular}
\end{table}

\begin{figure}[t]
\centering
\includegraphics[width=\linewidth]{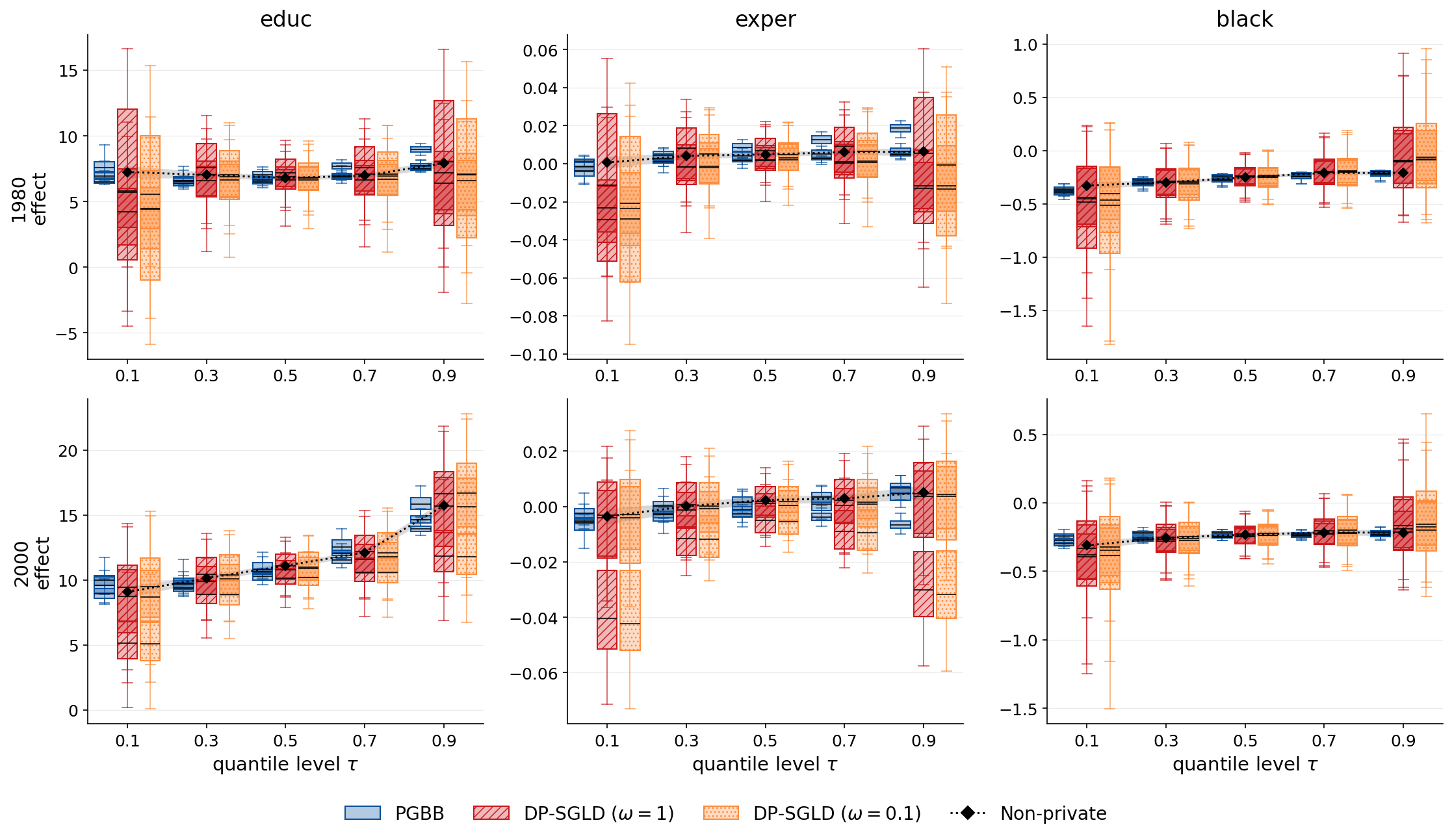}
\caption{Returns-to-schooling coefficients across $\tau$ for all covariates and both years.}
\label{fig:acfv_allvars}
\end{figure}

\begin{figure}[p]
\centering
\includegraphics[width=\linewidth]{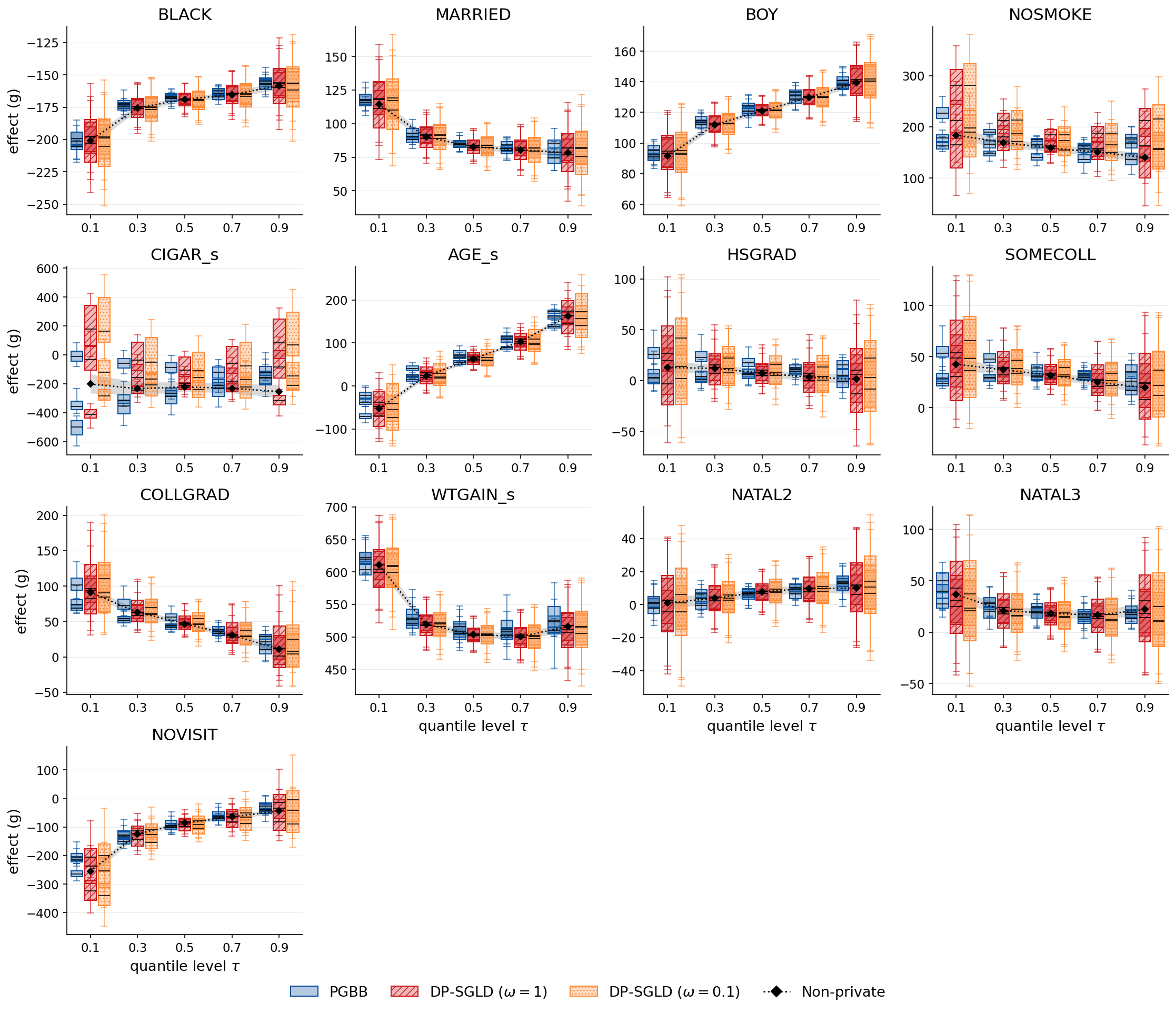}
\caption{Birthweight coefficients (grams) for the Natality application across $\tau$ for all covariates.}
\label{fig:natality_allvars}
\end{figure}

\subsection{Conversion from RDP to DP}
\label{app:conversion-to-dp}

Proposition~\ref{prop:formal-dp} presents $(\epsilon,\delta)$-DP for simplicity. Its derivation, however, exploits the R\'enyi differential privacy (RDP) that measures privacy by the R\'enyi divergence between the two output distributions on neighboring databases \citep{miro:17}. For an integer order $\lambda>1$, the $\lambda$-R\'enyi divergence between $P$ and $Q$ is defined as $D_\lambda(P\|Q)=\tfrac{1}{\lambda-1}\log\E_{Q}[(dP/dQ)^\lambda]$. A mechanism $\cM$ is $(\lambda,\gamma(\lambda))$-RDP if $D_\lambda(\cM(D)\|\cM(D'))\le\gamma(\lambda)$ for every replace-one neighboring pair $D,D'$. Proof of Proposition~\ref{prop:formal-dp} mainly uses the order-$\lambda$ log moment, i.e., $\log\E_{Q}[(dP/dQ)^\lambda]$ and the RDP-to-DP conversion formula suggested in \cite{miro:17} to state $(\epsilon,\delta)$-DP. More specifically, Proposition~\ref{prop:formal-dp} bounds the order-$\lambda$ log-moment of one iteration by $A_{J,\lambda,m}(q)$, so the single release is $\bigl(\lambda,\,A_{J,\lambda,m}(q)/(\lambda-1)\bigr)$-RDP. The generator runs $TK$ such releases, so the total RDP is $\gamma^{\mathrm{tot}}_{M}(\lambda)=TK\,A_{J,\lambda,m}(q)/(\lambda-1)$. The conversion formula of \citet{miro:17} gives $\delta\le\exp\{(\lambda-1)(\gamma(\lambda)-\epsilon)\}=\exp\{TK\,A_{J,\lambda,m}(q)-(\lambda-1)\epsilon\}$. Minimizing this over the integer orders $\lambda\ge2$ recovers $\delta_\star(\epsilon)$ in Proposition~\ref{prop:formal-dp}. Hence $\delta_\star(\epsilon)$ is simply the total R\'enyi cost over the $TK$ steps, converted to $(\epsilon,\delta)$ and tightened over the order $\lambda$.

The Mironov bound underlying $\delta_\star(\epsilon)$ is transparent but not tight enough. Lemma~1 in \citet{asoo:etal:21} characterizes a more efficient conversion formula by adopting an optimization perspective on the relationship between DP and RDP. It applies to any $(\lambda,\gamma(\lambda))$-RDP mechanism at a fixed order $\lambda>1$. With $\zeta_\lambda:=(\lambda-1)^{\lambda-1}/\lambda^\lambda$ and a target $\delta\in(0,1)$, if $\lambda\delta\le 1$ then $\cM$ is $(\epsilon,\delta)$-DP whenever
\begin{align}\label{eqn:asoo-conversion}
    \epsilon \ge \frac{1}{\lambda-1} 
    \min\Bigl\{ \bigl[(\lambda-1)\gamma(\lambda)-\log(\delta/\zeta_\lambda)\bigr]_+, 
    \log \Bigl(1+\tfrac{e^{(\lambda-1)\gamma(\lambda)}-1}{\lambda\delta}\Bigr)\Bigr\},
\end{align}
and if $\lambda\delta\ge 1$ it suffices that $\epsilon\ge\bigl(\gamma(\lambda)+\log(1-\delta)\bigr)_+$. We substitute $\gamma(\lambda)=\gamma^{\mathrm{tot}}_M(\lambda)$ from Section~\ref{sec:bir-step3} and minimize over the integer orders $\lambda\ge2$. This also satisfies the $(\epsilon,\delta)$ guarantee, and it is never weaker than $(\epsilon,\delta_\star(\epsilon))$. We therefore state Proposition~\ref{prop:formal-dp} in the transparent form for better readability and leverage this sharper conversion for simulation.

\subsection{Bisection Scheme to Select the Noise Scale $\sigma$}
\label{app:bisection-sigma}

In the simulation, the noise scale $\sigma$ is determined by the target privacy budget $(\epsilon,\delta)$ and other hyperparameters $(n,S,\alpha,q,C,T,K)$. The per-iteration log-moment $A_{J,\lambda,m}(q)$ depends on $\sigma$ only through the ratio $C/\sigma$, which enters $M_{J,\lambda}$, $R_{J,\lambda}$, and the higher-order remainder $\mathcal{R}_{J,\lambda,m}(q)$, so the total RDP $\gamma^{\mathrm{tot}}_M(\lambda)=TK\,A_{J,\lambda,m}(q)/(\lambda-1)$ is a function of $\sigma$. At the fixed target $\delta$, the conversion \eqref{eqn:asoo-conversion} turns this RDP into a privacy level at each order, and the tightest one over the integer orders is
\begin{align*}
    \epsilon(\sigma) = \min_{\lambda\ge2}\ \epsilon^{\delta}_{\lambda}\bigl(\gamma^{\mathrm{tot}}_M(\lambda)\bigr),
\end{align*}
where $\epsilon^{\delta}_{\lambda}(\gamma)$ is the smallest $\epsilon$ that \eqref{eqn:asoo-conversion} admits at order $\lambda$, that is,
\begin{align*}
    \epsilon^{\delta}_{\lambda}(\gamma)=
    \begin{cases}
    \dfrac{1}{\lambda-1}\min\Bigl\{\bigl[(\lambda-1)\gamma-\log(\delta/\zeta_\lambda)\bigr]_+,\ \log\Bigl(1+\tfrac{e^{(\lambda-1)\gamma}-1}{\lambda\delta}\Bigr)\Bigr\}, & \lambda\delta<1,\\[6pt]
    \bigl(\gamma+\log(1-\delta)\bigr)_+, & \lambda\delta\ge1.
    \end{cases}
\end{align*}
We then solve for the smallest noise that meets the target,
\begin{align*}
    \sigma^\star = \min\{\sigma>0:\ \epsilon(\sigma)\le\epsilon\}.
\end{align*}

The per-iteration log-moment $A_{J,\lambda,m}(q)$ decreases as $\sigma$ grows, i.e., $\epsilon(\sigma)$ decreases monotonically in $\sigma$. The feasible set $\{\sigma:\epsilon(\sigma)\le\epsilon\}$ is then a half-line $[\sigma^\star,\infty)$, and $\sigma^\star$ is its unique boundary. We find it by a bisection scheme on $\log\sigma$. The root is first bracketed by doubling $\sigma_{\mathrm{hi}}$ until $\epsilon(\sigma_{\mathrm{hi}})\le\epsilon$ and halving $\sigma_{\mathrm{lo}}$ until $\epsilon(\sigma_{\mathrm{lo}})>\epsilon$. The geometric midpoint $\sigma_{\mathrm{mid}}=\sqrt{\sigma_{\mathrm{lo}}\sigma_{\mathrm{hi}}}$ then replaces $\sigma_{\mathrm{lo}}$ when it is infeasible and $\sigma_{\mathrm{hi}}$ otherwise, until the relative width $(\sigma_{\mathrm{hi}}-\sigma_{\mathrm{lo}})/\sigma_{\mathrm{hi}}$ falls below a tolerance. The feasible endpoint is returned, so the released $\sigma$ always satisfies $\epsilon(\sigma)\le\epsilon$ and the guarantee is conservative. The same routine applies to the choice of $\sigma$ for DP-SGLD, which shares the replace-one Poisson-subsampled Gaussian accountant.

\section{Proof}
\label{app:mathematical-proof}

\subsection{Proof of Proposition~\ref{prop:choice-alpha} and Theorem~\ref{thm:solution-gap}}
\label{app:newproof}

Throughout, (A1-5) hold. All constants depend only on the population quantities $\mu_\infty,d_\theta,d_x$ and the uniform derivative bounds of $l$ from (A3), never on $n$ or $S$. We develop the common machinery first (Section~\ref{app:common}), then prove the Proposition (Section~\ref{app:proof-prop}) and the Theorem (Section~\ref{app:proof-thm}).

\subsubsection{Common Notation and Lemmas}
\label{app:common}

\paragraph{Notation.}
For a weight vector $v=(v_1,\dots,v_n)$ with $v_i\ge0$ and $\sum_i v_i=1$, write the weighted empirical risk, gradient, and Hessian
\begin{align*}
L(\theta;v):=\sum_{i=1}^n v_i l(X_i,\theta),\qquad
G(\theta;v):=\nabla_\theta L(\theta;v),\qquad
H(\theta;v):=\nabla_\theta^2 L(\theta;v).
\end{align*}
Let $\theta_0^\star:=\argmin_\theta L(\theta;n^{-1}\mathbf 1_n)$ and $\theta_\infty^\star:=\argmin_\theta\E\,l(X,\theta)$, and set $H_{0,n}:=H(\theta_0^\star;n^{-1}\mathbf 1_n)$. At $\theta_0^\star$ we define the scores and the centered Hessian fluctuations
\begin{align*}
\psi_i:=\nabla_\theta l(X_i,\theta_0^\star),\qquad
A_i:=\nabla_\theta^2 l(X_i,\theta_0^\star)-H_{0,n}.
\end{align*}
Note that $\sum_i \psi_i=0$ by the first-order condition and $\sum_{i=1}^n A_i=0$ by construction. The per-record bootstrap draws $\bw\sim\Diri(\mathbf 1_n)$ with weights $w_i$, and the block bootstrap draws $\bu\sim\Diri(\gamma\mathbf 1_S)$ and reallocates block-constantly through $\bar w_i:=u_{\pi(i)}/n_S$ over the balanced partition $\pi$ with groups $I_1,\dots,I_S$ of size $n_S=n/S$. Recall $\theta_\sub^\star=\theta_\sub^\star(\bar \bw):=\argmin_\theta L(\theta;\bar \bw)$ and $\theta_\Diri^\star=\theta_\Diri^\star(\bw):=\argmin_\theta L(\theta;\bw)$. We set the group means $\bar\psi_s:=\frac1{n_S}\sum_{i\in I_s}\psi_i$ and $\bar A_s:=\frac1{n_S}\sum_{i\in I_s}A_i$, and the centered score and Hessian sums
\begin{align*}
&g_\sub:=\sum_{s=1}^S(u_s-\tfrac1S)\bar\psi_s,\qquad
g_\Diri:=\sum_{i=1}^n(w_i-\tfrac1n)\psi_i,\\
&B_\sub:=\sum_{s=1}^S(u_s-\tfrac1S)\bar A_s,\qquad
B_\Diri:=\sum_{i=1}^n(w_i-\tfrac1n)A_i,
\end{align*}
together with the moment summaries
\begin{align*}
\Sigma_{\psi,n}:=\tfrac1n\sum_i\psi_i\psi_i^\top,\quad
\bar\Sigma_{\psi,n}:=\tfrac1S\sum_s\bar\psi_s\bar\psi_s^\top,\quad
\sigma_{\psi,n}^2:=\tr\Sigma_{\psi,n},\quad
\bar\sigma_{\psi,n}^2:=\tr\bar\Sigma_{\psi,n}.
\end{align*}
We also define $\kappa_n^2:=\frac1n\sum_i\|A_i\|_F^2$ and $\bar\kappa_n^2:=\frac1S\sum_s\|\bar A_s\|_F^2$, write $\E_{\bX_{1:n}},\E_\pi,\E_u,\E_w$ for expectation over the respective source, and abbreviate $\bullet\in\{\sub,\Diri\}$.

By (A1-3) the following population constants are finite and used throughout. The gradient bounds are $C_\infty:=\sup_{x}\|\nabla_\theta l(x,\theta_\infty^\star)\|$ and $C:=\sup_{x,\theta}\|\nabla_\theta l(x,\theta)\|$. The Hessian bound is $K_2:=\sup_{x,\theta}\|\nabla_\theta^2 l(x,\theta)\|_\op$, and the Hessian-Lipschitz constant $K_3$ makes $\theta\mapsto\nabla_\theta^2 l(x,\theta)$ Lipschitz uniformly in $x$. We also set $\rho_\infty:=\mathrm{dist}(\theta_\infty^\star,\partial\Theta)>0$ by (A2), $\mu_n:=\mu_\infty/2$, and $D_\Theta:=\diam(\Theta)$, and fix once and for all the localization radius
\begin{align*}
r:=\min\{\tfrac{\rho_\infty}{4},\ \tfrac{\mu_\infty}{8K_3}\}>0
\qquad(\tfrac{\mu_\infty}{8K_3}:=\infty\text{ when }K_3=0),
\end{align*}
which satisfies $r\le\rho_\infty/4$ and $K_3 r\le\mu_\infty/8$. The same symbol $r$ is used in every lemma below.



\begin{lemma}
\label{lem:cm-gamma}
Let $G_1,\dots,G_K\overset{\mathrm{iid}}\sim\mathrm{Gamma}(\alpha,1)$ and $Y:=\sum_{k}G_k$. Set $u_k=G_k/Y$. Then $(u_k)_{k=1}^K\sim\Diri(\alpha\mathbf 1_K)$ with $\E u_k=1/K$ and
\begin{align*}
\Cov(u_k,u_\ell)=\frac{K\delta_{kl}-1}{K^2(K\alpha+1)},
\end{align*}
where $\delta_{kl}=1$ if $k=l$ and $0$ otherwise.
\end{lemma}
\begin{proof}
The Gamma-Dirichlet construction and the independence of $(G_k/Y)_k$ from $Y$ are classical. With $\E u_k=1/K$, $\E u_k^2=\frac{\alpha+1}{K(K\alpha+1)}$, and $\E u_ku_\ell=\frac{\alpha}{K(K\alpha+1)}$ for $k\ne\ell$, subtracting $1/K^2$ gives the stated covariance.
\end{proof}

\begin{lemma}
\label{lem:gamma-indep}
Let $E_1,\dots,E_K\sim\mathrm{Gamma}(\alpha,1)$ i.i.d. with $\alpha>0$, set $V:=\sum_{k=1}^KE_k$ and $\zeta_k:=E_k/V$. Then, for any fixed vectors $v_1,\dots,v_K\in\R^d$ with $\sum_{k}v_k=0$, the centered linear form $h:=\sum_k(\zeta_k-\tfrac1K)v_k$ satisfies $Vh=\sum_k(E_k-\alpha)v_k=:M$ and $h\perp V$, and hence $\E\|M\|^4=\E[V^4]\,\E\|h\|^4$.
\end{lemma}
\begin{proof}
The additivity of the Gamma distribution shows $V\sim \mathrm{Gamma}(K\alpha,1)$. To establish the independence we find the joint density of $(\zeta_1,\dots,\zeta_{K-1},V)$ via the change of variables from $(E_1,\dots,E_K)$, noting that $\zeta_K$ is a deterministic function of $\zeta_1,\dots,\zeta_{K-1}$.

The joint density of $(E_1,\dots,E_K)$ on $(0,\infty)^K$ is $f_E(e)=\Gamma(\alpha)^{-K}(\prod_ke_k^{\alpha-1})e^{-\sum_ke_k}$. Pass to coordinates $(\zeta_1,\dots,\zeta_{K-1},V)$ through the bijection $e_k=V\zeta_k$ for $k<K$ and $e_K=V(1-\sum_{j<K}\zeta_j)=V\zeta_K$, defined on the simplex $\{\zeta_k>0,\ \sum_k\zeta_k=1\}$ times $(0,\infty)$. The derivative matrix $\partial(e_1,\dots,e_K)/\partial(\zeta_1,\dots,\zeta_{K-1},V)$ has $V\,I_{K-1}$ as its top-left $(K-1)\times(K-1)$ block from $\partial e_k/\partial\zeta_j=V\,\mathbf 1\{j=k\}$, the vector $(\zeta_1,\dots,\zeta_{K-1})^\top$ in the first $K-1$ entries of its last column from $\partial e_k/\partial V=\zeta_k$, and last row $(-V,\dots,-V,\zeta_K)$ from $\partial e_K/\partial\zeta_j=-V$ and $\partial e_K/\partial V=\zeta_K$. The determinant is invariant under a row operation, so adding rows $1,\dots,K-1$ to the last row turns it into $(0,\dots,0,1)$ and makes an upper triangular matrix with diagonal $(V,\dots,V,1)$, whose determinant is $V^{K-1}$. Hence
\begin{align*}
f_{\zeta,V}(\zeta,V)=\frac{1}{\Gamma(\alpha)^K}\Big(\prod_k\zeta_k^{\alpha-1}\Big)V^{K\alpha-1}e^{-V}
=\underbrace{\frac{\Gamma(K\alpha)}{\Gamma(\alpha)^K}\prod_k\zeta_k^{\alpha-1}}_{\Diri(\alpha\mathbf 1_K)\text{ density}}\cdot\underbrace{\frac{V^{K\alpha-1}e^{-V}}{\Gamma(K\alpha)}}_{\mathrm{Gamma}(K\alpha,1)\text{ density}}.
\end{align*}
Therefore $\|M\|^4=V^4\|h\|^4$ with $V$ independent of $\|h\|^4$, giving $\E\|M\|^4=\E[V^4]\,\E\|h\|^4$. Finally $\E[V^4]=\Gamma(K\alpha+4)/\Gamma(K\alpha)=K\alpha(K\alpha+1)(K\alpha+2)(K\alpha+3)$ since $V\sim\mathrm{Gamma}(K\alpha,1)$.
\end{proof}

\begin{lemma}
\label{lem:cm-diri}
Conditional on $(\bX_{1:n},\pi)$,
\begin{align*}
\Cov_w(\sqrt n g_\Diri\mid \bX_{1:n})=\frac{n}{n+1}\Sigma_{\psi,n},
\qquad
\Cov_u(\sqrt n g_\sub\mid \bX_{1:n},\pi)=\frac{n}{S\gamma+1}\bar\Sigma_{\psi,n},
\end{align*}
and $\E_w\|B_\Diri\|_F^2=\frac{1}{n+1}\kappa_n^2$, $\E_u\|B_\sub\|_F^2=\frac{1}{S\gamma+1}\bar\kappa_n^2$. Moreover, for a uniformly random balanced partition,
\begin{align*}
\E_\pi[\bar\Sigma_{\psi,n}\mid \bX_{1:n}]=\frac{S-1}{n-1} \Sigma_{\psi,n}.
\end{align*}
\end{lemma}
\begin{proof}
The characterization below uses the first-order condition $\sum_i\psi_i=0$ and its variants $\sum_s\bar\psi_s=\frac1{n_S}\sum_i\psi_i=0$ and $\sum_i A_i=0$.

\paragraph{Per-record covariance.} By Lemma~\ref{lem:cm-gamma} with $(K,\alpha)=(n,1)$, $\Cov(w_i,w_j)=\frac{n\delta_{ij}-1}{n^2(n+1)}$. Since $g_\Diri=\sum_i(w_i-\frac1n)\psi_i$ is linear in $\bw$ and shifting by the constant $-\frac1n$ does not change covariances,
\begin{align*}
\Cov_w(g_\Diri)=\sum_{i,j}\Cov(w_i,w_j) \psi_i\psi_j^\top
=\frac{1}{n^2(n+1)}\Big[n\sum_i\psi_i\psi_i^\top-\Big(\sum_i\psi_i\Big)\Big(\sum_j\psi_j\Big)^\top\Big],
\end{align*}
where the second term vanishes because $\sum_i \psi_i=0$. Hence $\Cov_w(g_\Diri)=\frac1{n(n+1)}\sum_i\psi_i\psi_i^\top=\frac1{n+1}\Sigma_{\psi,n}$, which also gives $\E_w\|g_\Diri\|^2=\frac{1}{n}\tr\Cov_w(\sqrt n g_\Diri)=\frac{\sigma_{\psi,n}^2}{n+1}$.

\paragraph{Block covariance.} The block-constant reallocation $\bar w_i=u_{\pi(i)}/n_S$ gives, after grouping the indices $i\in I_s$, $g_\sub=\sum_i(\bar w_i-\frac1n)\psi_i=\sum_s u_s\bar\psi_s-\frac1S\sum_s\bar\psi_s=\sum_s(u_s-\frac1S)\bar\psi_s$. This is the same linear form as in the per-record case with $(K,\alpha)=(S,\gamma)$ and $\bar\psi_s$ in place of $\psi_i$, and $\sum_s\bar\psi_s=0$ plays the role of the first-order condition. Hence $\Cov_u(\sqrt n g_\sub)=\frac{n}{S\gamma+1}\bar\Sigma_{\psi,n}$ and $\E_u\|g_\sub\|^2=\frac{\bar\sigma_{\psi,n}^2}{S\gamma+1}$.

\paragraph{Hessian second moments.} The sums $B_\sub=\sum_s(u_s-\frac1S)\bar A_s$ and $B_\Diri=\sum_i(w_i-\frac1n)A_i$ have the same linear structure with the matrices $\bar A_s$ and $A_i$ as coefficients. Viewing matrices as vectors under the Frobenius inner product $\langle M,M'\rangle_F=\tr(M^\top M')$ and applying the per-record and block computations coordinatewise gives $\E_u\|B_\sub\|_F^2=\frac{1}{S\gamma+1}\frac1S\sum_s\|\bar A_s\|_F^2=\frac{\bar\kappa_n^2}{S\gamma+1}$ and $\E_w\|B_\Diri\|_F^2=\frac{\kappa_n^2}{n+1}$, using $\sum_s\bar A_s=0$ and $\sum_iA_i=0$.

\paragraph{Partition mean.} For a uniformly random balanced partition $\pi$, each index $i$ is equally likely to occupy any of the $n$ slots, so $\P_\pi(i\in I_s)=n_S/n=\frac1S$. Given $i\in I_s$, a distinct index $j$ occupies one of the remaining $n-1$ slots, exactly $n_S-1$ of which lie in $I_s$, so $\P_\pi(\{i,j\}\subset I_s)=\frac1S\frac{n_S-1}{n-1}$. From $\bar\psi_s\bar\psi_s^\top=\frac1{n_S^2}\sum_{i,j\in I_s}\psi_i\psi_j^\top=\frac{1}{n_S^2}\sum_{i,j}\mathbf 1\{\pi(i)=\pi(j)=s\}\psi_i\psi_j^\top$ and the identity $\sum_{i\ne j}\psi_i\psi_j^\top=\big(\sum_i\psi_i\big)\big(\sum_j\psi_j\big)^\top-\sum_i\psi_i\psi_i^\top=-n\Sigma_{\psi,n}$,
\begin{align*}
\E_\pi[\bar\psi_s\bar\psi_s^\top]
=\frac1{n_S^2}\Big[\frac1S\sum_i\psi_i\psi_i^\top+\frac1S\frac{n_S-1}{n-1}\sum_{i\ne j}\psi_i\psi_j^\top\Big]
=\frac{n}{Sn_S^2}\Big(1-\frac{n_S-1}{n-1}\Big)\Sigma_{\psi,n}
=\frac{S-1}{n-1}\Sigma_{\psi,n},
\end{align*}
where the last step uses $n=Sn_S$ and $1-\frac{n_S-1}{n-1}=\frac{n-n_S}{n-1}=\frac{n_S(S-1)}{n-1}$. Averaging over $s$ leaves the right-hand side unchanged, giving $\E_\pi[\bar\Sigma_{\psi,n}\mid \bX_{1:n}]=\frac{S-1}{n-1}\Sigma_{\psi,n}$.
\end{proof}

\begin{lemma}
\label{lem:fourth-moment}
Let $\xi_1,\dots,\xi_K$ be independent mean-zero random vectors in a Hilbert space with $\E\|\xi_k\|^4<\infty$. Then
\begin{align*}
\E\Big\|\sum_{k=1}^K\xi_k\Big\|^4 \le 3\Big(\sum_{k=1}^K\E\|\xi_k\|^2\Big)^2+\sum_{k=1}^K\E\|\xi_k\|^4.
\end{align*}
\end{lemma}
\begin{proof}
Write $W=\sum_k\xi_k$. From $\|W\|^2=\langle W,W\rangle=\sum_{i,j}\langle\xi_i,\xi_j\rangle$,
\begin{align*}
\E\|W\|^4=\E\Big[\Big(\sum_{i,j}\langle\xi_i,\xi_j\rangle\Big)\Big(\sum_{k,l}\langle\xi_k,\xi_l\rangle\Big)\Big]
=\sum_{i,j,k,l}\E[\langle\xi_i,\xi_j\rangle\langle\xi_k,\xi_l\rangle],
\end{align*}
a sum over the four slots $(i,j,k,l)\in\{1,\dots,K\}^4$.

\paragraph{Singletons vanish.} If some value occupies exactly one slot, say $i$ with $j,k,l\ne i$, then $\xi_i$ is independent of $(\xi_j,\xi_k,\xi_l)$, and by linearity of the inner product in its first argument,
\begin{align*}
\E[\langle\xi_i,\xi_j\rangle\langle\xi_k,\xi_l\rangle]
=\E\big[\langle\xi_k,\xi_l\rangle \langle\E\xi_i,\ \xi_j\rangle\big]=0.
\end{align*}

\paragraph{Diagonal.} The diagonal term is $\langle\xi_i,\xi_i\rangle^2=\|\xi_i\|^4$, contributing $\sum_i\E\|\xi_i\|^4$.

\paragraph{Two-pair patterns.} The four slots split into two equal-index pairs in exactly three ways, and each two-pair quadruple matches one of the following,
\begin{align*}
\begin{array}{lll}
\{1,2\}\{3,4\}: & i=j,\ k=l,\ i\ne k & \displaystyle\sum_{i\ne k}\E\|\xi_i\|^2 \E\|\xi_k\|^2,\\[4pt]
\{1,3\}\{2,4\}: & i=k,\ j=l,\ i\ne j & \displaystyle\sum_{i\ne j}\E\langle\xi_i,\xi_j\rangle^2,\\[4pt]
\{1,4\}\{2,3\}: & i=l,\ j=k,\ i\ne j & \displaystyle\sum_{i\ne j}\E\langle\xi_i,\xi_j\rangle^2,
\end{array}
\end{align*}
where the first line uses independence ($i\ne k$) to factor $\E\|\xi_i\|^2\|\xi_k\|^2$. For $i\ne j$ the Cauchy-Schwarz inequality gives $\E\langle\xi_i,\xi_j\rangle^2\le\E\|\xi_i\|^2\,\E\|\xi_j\|^2$, so each of the three sums is at most $\big(\sum_k\E\|\xi_k\|^2\big)^2$. Adding the diagonal and the three two-pair contributions yields the claim.
\end{proof}

\begin{lemma}
\label{lem:cm-four}
On $\mathcal G_n$, uniformly over $\gamma\in\Gamma=[\underline\gamma,\overline\gamma]\subset(0,\infty)$,
\begin{align*}
&\E_w[\|g_\Diri\|^4\mid \bX_{1:n}]=O(n^{-2}),\quad
\E_w[\|B_\Diri\|_F^4\mid \bX_{1:n}]=O(n^{-2}),\\
&\E_{u,\pi}[\|g_\sub\|^4\mid \bX_{1:n}]=O(\gamma^{-3}n^{-2})+O(\gamma^{-3}S^{-1}n^{-2}),\\
&\E_{u,\pi}[\|B_\sub\|_F^4\mid \bX_{1:n}]=O(\gamma^{-3}n^{-2})+O(\gamma^{-3}S^{-1}n^{-2}),
\end{align*}
with constants depending only on $(C,K_2,d_\theta,\underline\gamma)$ and uniform in $S$.
\end{lemma}
\begin{proof}
Let $c_i$ denote either $\psi_i$ (for $g$) or $A_i$ (for $B$). In both cases $\sum_i c_i=0$ on $\mathcal G_n$ and $\|c_i\|\le C_\star:=\max\{C,2\sqrt{d_\theta}K_2\}$ by (A3), where $A_i=\nabla_\theta^2 l(X_i,\theta_0^\star)-H_{0,n}$ has $\|A_i\|_F\le2\sqrt{d_\theta}K_2$.

Both $g_\Diri$ and $g_\sub$ have the form of Lemma~\ref{lem:gamma-indep}, so we may combine it with Lemma~\ref{lem:fourth-moment}. Since $\E\|\xi_k\|^2=\alpha\|v_k\|^2$ and $\E\|\xi_k\|^4=(3\alpha^2+6\alpha)\|v_k\|^4$ from $\Var(E_k)=\alpha$ and $\E|E_k-\alpha|^4=3\alpha^2+6\alpha$, and using $3\alpha^2\le3\alpha^2+6\alpha$,
\begin{align}\label{eqn:cm-four-M}
\E\|M\|^4\le(3\alpha^2+6\alpha)\Big[\Big(\sum_k\|v_k\|^2\Big)^2+\sum_k\|v_k\|^4\Big].
\end{align}

\paragraph{org-BB case} $(K=n,\ \alpha=1,\ v_i=c_i,\ \zeta=w,\ V=T)$. Here $\sum_i\|c_i\|^2\le nC_\star^2$ and $\sum_i\|c_i\|^4\le C_\star^2\sum_i\|c_i\|^2\le nC_\star^4$, so \eqref{eqn:cm-four-M} with $\alpha=1$, that is $3\alpha^2+6\alpha=9$, gives $\E\|M\|^4\le 9[(nC_\star^2)^2+nC_\star^4]=O(n^2)$, while $\E[T^4]=n(n+1)(n+2)(n+3)\ge n^4$. Hence $\E_w[\|g_\Diri\|^4\mid \bX_{1:n}]=\E\|M\|^4/\E[T^4]=O(n^{-2})$, and likewise $\E_w[\|B_\Diri\|_F^4\mid \bX_{1:n}]=O(n^{-2})$.

\paragraph{blk-BB case} $(K=S,\ \alpha=\gamma,\ v_s=\bar c_s:=\frac1{n_S}\sum_{i\in I_s}c_i,\ \zeta=u,\ V=\sum_sE_s)$, with $\sum_s\bar c_s=\frac1{n_S}\sum_ic_i=0$. Here $g_\sub$ depends on two independent sources, the Dirichlet weights $u$ and the balanced partition $\pi$. Conditionally on $\pi$, \eqref{eqn:cm-four-M} gives
\begin{align}\label{eqn:cm-four-sub1}
\E_u[\|g_\sub\|^4\mid \bX_{1:n},\pi]=\frac{\E\|M\|^4}{\E[V^4]}
\le\frac{3\gamma^2+6\gamma}{\E[V^4]}\Big[\Big(\sum_s\|\bar c_s\|^2\Big)^2+\sum_s\|\bar c_s\|^4\Big].
\end{align}
We invoke Hoeffding's sampling-without-replacement argument \citep[Theorem~4]{hoef:63} only through its Jensen step. Conditional on $\bX_{1:n}$ the vectors $c_1,\dots,c_n$ are fixed, and Jensen applies directly to any convex scalar function of their sum. In particular, for the convex norm power used below,
\begin{align*}
\E_\pi\Big\|\sum_{i\in I_s}c_i\Big\|^4\le\E\Big\|\sum_{j=1}^{n_S}c_{J_j}\Big\|^4,
\qquad J_1,\dots,J_{n_S}\overset{\mathrm{iid}}\sim\mathrm{Unif}\{1,\dots,n\},
\end{align*}
where $z\mapsto\|z\|^4$ is convex because $\|\cdot\|$ is convex and $t\mapsto t^4$ is convex nondecreasing on $[0,\infty)$. The right-hand sum is i.i.d. and mean-zero with $\E\|c_{J_j}\|^2=\frac1n\sum_i\|c_i\|^2\le C_\star^2$ and $\E\|c_{J_j}\|^4\le C_\star^4$, so $\E\|\sum_jc_{J_j}\|^4\le3(n_SC_\star^2)^2+n_SC_\star^4$ and
\begin{align*}
\E_\pi\|\bar c_s\|^4=n_S^{-4}\E_\pi\Big\|\sum_{i\in I_s}c_i\Big\|^4
\le n_S^{-4}[3(n_SC_\star^2)^2+n_SC_\star^4]=O(n_S^{-2})=O(S^2/n^2).
\end{align*}
By symmetry over $s$,
\begin{align*}
\begin{gathered}
\E_\pi\Big[\sum_s\|\bar c_s\|^4\Big]=S \E_\pi\|\bar c_s\|^4=O(S^3/n^2),\\[2pt]
\E_\pi\Big[\Big(\sum_s\|\bar c_s\|^2\Big)^2\Big]\le\Big(\sum_s(\E_\pi\|\bar c_s\|^4)^{1/2}\Big)^2=O(S^4/n^2),
\end{gathered}
\end{align*}
where the second bound follows by Minkowski's inequality applied to $Y_s=\|\bar c_s\|^2$ in $L^2(\pi)$, that is $\|\sum_s Y_s\|_{L^2(\pi)}\le\sum_s\|Y_s\|_{L^2(\pi)}$, and squaring. Substituting into \eqref{eqn:cm-four-sub1} and taking $\E_\pi$ (with $\E[V^4]$ constant in $\pi$),
\begin{align*}
\E_{u,\pi}[\|g_\sub\|^4\mid \bX_{1:n}]
\le\frac{3\gamma^2+6\gamma}{\E[V^4]}[O(S^4/n^2)+O(S^3/n^2)].
\end{align*}
Finally $\E[V^4]=S\gamma(S\gamma+1)(S\gamma+2)(S\gamma+3)\ge(S\gamma)^4$, so
\begin{align*}
\E_{u,\pi}[\|g_\sub\|^4\mid \bX_{1:n}]
\le\frac{3\gamma^2+6\gamma}{(S\gamma)^4}[O(S^4/n^2)+O(S^3/n^2)]
=O(\gamma^{-3}n^{-2})+O(\gamma^{-3}S^{-1}n^{-2})
\end{align*}
uniformly over $\gamma\in\Gamma$. The identical computation with $c_i=A_i$ gives $\E_{u,\pi}[\|B_\sub\|_F^4\mid \bX_{1:n}]$ the same rate.
\end{proof}

\begin{lemma}
\label{lem:cm-reg}
Let $\theta_0^\star:=\arg\min_{\theta\in\Theta}M_n(\theta)$ with $M_n(\theta):=L(\theta;n^{-1}\mathbf 1_n)=\frac1n\sum_i l(X_i,\theta)$ and $\theta^{\star}_{\infty}:=\arg\min_{\theta \in \Theta} \E[l(X,\theta)]$. Define $\mathcal G_n^{(1)}:=\{\|\theta_0^\star-\theta_\infty^\star\|\le r\}$ with $r=\min\{\frac{\rho_\infty}{4},\ \frac{\mu_\infty}{8K_3}\}$, $\mathcal G_n^{(2)}:=\{\|H_n(\theta_\infty^\star)-H_\infty\|_\op\le\mu_\infty/4\}$, and the good event $\mathcal G_n=\mathcal G_n^{(1)}\cap\mathcal G_n^{(2)}$. Then the following hold.
\begin{enumerate}
\item[\textnormal{(i)}] $\P(\mathcal G_n^c)=O(e^{-c'n})$ for some constant $c'>0$.
\item[\textnormal{(ii)}] On $\mathcal G_n$, $M_n$ is $\mu_\infty/2$-strongly convex on $\overline{B(\theta_\infty^\star,r)}$, and $\theta_0^\star\in\overline{B(\theta_\infty^\star,r)}$. Moreover, $\theta_0^\star$ is the unique minimizer of $M_n$, lies in $\mathrm{int}(\Theta)$ with $\mathrm{dist}(\theta_0^\star,\partial\Theta)\ge\rho_\infty/2$, satisfies the first-order condition $\sum_i\psi_i=0$, and $H_{0,n}=H_n(\theta_0^\star)\succeq\mu_\infty/2\,I$.
\item[\textnormal{(iii)}] $\E_{\bX_{1:n}}\|\theta_0^\star-\theta_\infty^\star\|^2=O(1/n)$.
\end{enumerate}
\end{lemma}
\begin{proof}
Write $\Delta_n:=\theta_0^\star-\theta_\infty^\star$ and $M(\theta):=\E\,l(X,\theta)$, and recall $H_n(\theta)=\frac1n\sum_i\nabla_\theta^2 l(X_i,\theta)$ and the standing constants $K_2,K_3,C_\infty,C,\rho_\infty,D_\Theta,r$ of Section~\ref{app:common}.

\paragraph{Probability of the complement (i).} For $\mathcal G_n^{(2)}$, apply the matrix Bernstein inequality \citep{trop:12} to the i.i.d. $Z_i:=\nabla_\theta^2 l(X_i,\theta_\infty^\star)-H_\infty$ with $\E Z_i=0$, $\|Z_i\|_\op\le 2K_2$, and $\|\sum_i\E Z_i^2\|_\op\le 4K_2^2 n$. This gives $\P(\mathcal G_n^{(2),c})=O(e^{-c'n})$ for some constant $c'>0$. For $\mathcal G_n^{(1)}$, we have $\theta_0^\star\in\Theta$ always, so $\mathcal G_n^{(1),c}=\{\theta_0^\star\in K_0\}$ with $K_0:=\{\theta\in\Theta:\|\theta-\theta_\infty^\star\|\ge r\}$ compact, and $\gamma_0:=\inf_{\theta\in K_0}[M(\theta)-M(\theta_\infty^\star)]>0$ is attained on $K_0$ because $\theta_\infty^\star\notin K_0$ is the unique minimizer of the continuous $M$ on $\Theta$ by (A2). The basic inequality $M_n(\theta_0^\star)\le M_n(\theta_\infty^\star)$ gives $M(\theta_0^\star)-M(\theta_\infty^\star)\le|(M_n-M)(\theta_\infty^\star)-(M_n-M)(\theta_0^\star)|$, so on $\mathcal G_n^{(1),c}$, where the left side is at least $\gamma_0$,
\begin{align*}
\mathcal G_n^{(1),c}\subseteq\Big\{\sup_{\theta\in\Theta}\big|(M_n-M)(\theta_\infty^\star)-(M_n-M)(\theta)\big|\ge\gamma_0\Big\}.
\end{align*}
For fixed $\theta$, the variables $Y_i:=l(X_i,\theta_\infty^\star)-l(X_i,\theta)$ are i.i.d. with $|Y_i|\le CD_\Theta$. Writing $g(\theta)=(M_n-M)(\theta_\infty^\star)-(M_n-M)(\theta)=\frac{1}{n}\sum_{i=1}^n(Y_i-\E Y_i)$, Hoeffding's inequality gives the tail $\P(|g(\theta)|\ge t)\le 2e^{-nt^2/(2C^2D_\Theta^2)}$. Construct an $\epsilon$-net $\{\theta_1,\dots,\theta_N\}$ in which every $\theta\in\Theta$ lies within $\epsilon$ of some $\theta_j$, with $N\le(3D_\Theta/\epsilon)^{d_\theta}$ \citep{vers:18}. Since $\theta\mapsto(M_n-M)(\theta)$ is $2C$-Lipschitz, $|g(\theta)|\le|g(\theta_j)|+2C\epsilon$ and $\sup_\theta|g(\theta)|\le\max_j|g(\theta_j)|+2C\epsilon$. Choosing $\epsilon=\gamma_0/(8C)$ gives $\{\sup|g(\theta)|\ge\gamma_0\}\subseteq\{\max|g(\theta_j)|\ge3\gamma_0/4\}$, and the union bound over the net at $t=3\gamma_0/4$ yields $\P(\mathcal G_n^{(1),c})=O(e^{-c'n})$. Combining the two events, $\P(\mathcal G_n^c)=O(e^{-c'n})$.

\paragraph{Strong convexity and uniqueness (ii).} On $\mathcal G_n$, for every $\theta\in\overline{B(\theta_\infty^\star,r)}$ the $K_3$-Lipschitzness of $H_n(\cdot)$ from (A3) gives
\begin{align*}
\|H_n(\theta)-H_\infty\|_\op\le K_3\|\theta-\theta_\infty^\star\|+\|H_n(\theta_\infty^\star)-H_\infty\|_\op\le K_3 r+\frac{\mu_\infty}4\le\frac{\mu_\infty}8+\frac{\mu_\infty}4<\frac{\mu_\infty}2,
\end{align*}
so by Weyl's inequality and (A5), $H_n(\theta)\succeq\frac{\mu_\infty}2 I$, that is, $M_n$ is $\mu_\infty/2$-strongly convex on $\overline{B(\theta_\infty^\star,r)}$. On $\mathcal G_n$, a minimizer of $M_n$ lies in this convex ball, because the basic inequality applies to any minimizer, and a strongly convex function has at most one minimizer there, so $\theta_0^\star$ is the unique minimizer. As $r\le\rho_\infty/4$, the triangle inequality gives $\mathrm{dist}(\theta_0^\star,\partial\Theta)\ge\rho_\infty-r\ge\frac34\rho_\infty\ge\rho_\infty/2$, so $\theta_0^\star\in\mathrm{int}(\Theta)$ and the interior stationarity $\nabla_\theta L(\theta_0^\star;n^{-1}\mathbf 1_n)=\frac1n\sum_i\psi_i=0$ holds. Finally $H_{0,n}=H_n(\theta_0^\star)\succeq\mu_\infty/2\,I$ since $\theta_0^\star\in\overline{B(\theta_\infty^\star,r)}$.

\paragraph{$L^2$ rate (iii).} Split $\E_{\bX_{1:n}}\|\Delta_n\|^2=\E_{\bX_{1:n}}[\|\Delta_n\|^2\mathbf 1_{\mathcal G_n}]+\E_{\bX_{1:n}}[\|\Delta_n\|^2\mathbf 1_{\mathcal G_n^c}]$. On $\mathcal G_n$, since $r\le\rho_\infty/4<\rho_\infty$, we have $\theta_0^\star\in B(\theta_\infty^\star,r)\subset\mathrm{int}(\Theta)$, so the empirical first-order condition $\nabla M_n(\theta_0^\star)=0$ holds. Applying the fundamental theorem of calculus to $\nabla M_n$ along $[\theta_\infty^\star,\theta_0^\star]$ reads $0=\nabla M_n(\theta_\infty^\star)+\tilde H_n\Delta_n$ with $\tilde H_n:=\int_0^1 H_n(\theta_\infty^\star+t\Delta_n)\,dt$. For each $t$, $\|\theta_\infty^\star+t\Delta_n-\theta_\infty^\star\|=t\|\Delta_n\|\le r$, and $H_n$ is $K_3$-Lipschitz in $\theta$ by (A3), so
\begin{align*}
\|H_n(\theta_\infty^\star+t\Delta_n)-H_\infty\|_\op\le\underbrace{\|H_n(\theta_\infty^\star+t\Delta_n)-H_n(\theta_\infty^\star)\|_\op}_{\le K_3 r\le\mu_\infty/4}+\underbrace{\|H_n(\theta_\infty^\star)-H_\infty\|_\op}_{\le\mu_\infty/4\text{ on }\mathcal G_n^{(2)}}\le\frac{\mu_\infty}2.
\end{align*}
By Weyl's inequality and (A5), $H_n(\theta_\infty^\star+t\Delta_n)\succeq\frac{\mu_\infty}2 I$, hence $\tilde H_n\succeq\frac{\mu_\infty}2 I$ is invertible with $\|\tilde H_n^{-1}\|_\op\le2/\mu_\infty$. Therefore $\|\Delta_n\|\le\frac2{\mu_\infty}\|\nabla M_n(\theta_\infty^\star)\|$. The score $\nabla M_n(\theta_\infty^\star)=\frac1n\sum_i\nabla_\theta l(X_i,\theta_\infty^\star)$ is an average of i.i.d. centered vectors bounded by $C_\infty$, so $\E_{\bX_{1:n}}\|\nabla M_n(\theta_\infty^\star)\|^2=\frac1n\E\|\nabla_\theta l(X,\theta_\infty^\star)\|^2\le\frac{C^2}{n}$, giving
\begin{align*}
\E_{\bX_{1:n}}[\|\Delta_n\|^2\mathbf 1_{\mathcal G_n}]\le\frac{4}{\mu_\infty^2}\E_{\bX_{1:n}}\|\nabla M_n(\theta_\infty^\star)\|^2\le\frac{4C_\infty^2}{\mu_\infty^2 n}.
\end{align*}
Off $\mathcal G_n$, $\|\Delta_n\|\le D_\Theta$ and $\P(\mathcal G_n^c)=O(e^{-c'n})$, so $\E_{\bX_{1:n}}[\|\Delta_n\|^2\mathbf 1_{\mathcal G_n^c}]=O(e^{-c'n})$. Adding the two parts, $\E_{\bX_{1:n}}\|\Delta_n\|^2=O(1/n)$.
\end{proof}

\begin{lemma}
\label{lem:cm-lin}
Let $\bullet\in\{\sub,\Diri\}$ and recall the localization radius $r=\min\{\frac{\rho_\infty}{4},\ \frac{\mu_\infty}{8K_3}\}$. On $\mathcal G_n$, conditionally on $(\bX_{1:n},\pi)$, there is an event $\tilde E_n^\bullet$ on which $\theta_\bullet^\star$ is the unique minimizer of the weighted risk in $\{\|\theta-\theta_0^\star\|\le r\}$ and
\begin{align*}
\sqrt n (\theta_\bullet^\star-\theta_0^\star)=-H_{0,n}^{-1}\sqrt n g_\bullet+r_\bullet,
\end{align*}
with, for the block case,
\begin{align*}
&\P(\tilde E_n^{\sub c}\mid \bX_{1:n},\pi)\le\frac{8^2}{\mu_\infty^2(S\gamma+1)}\Big(\bar\kappa_n^2+\frac{\bar\sigma_{\psi,n}^2}{r^2}\Big)=O\Big(\frac{\bar\sigma_{\psi,n}^2+\bar\kappa_n^2}{S\gamma}\Big),\\
&\E[\|r_\sub\| \mathbf 1_{\tilde E_n^\sub}\mid \bX_{1:n},\pi]\le\Phi_\infty \frac{\sqrt n}{S\gamma}(\bar\sigma_{\psi,n}^2+\bar\kappa_n^2),
\end{align*}
and the per-record analog obtained by replacing $(S,\gamma,\bar\sigma_{\psi,n}^2,\bar\kappa_n^2)$ with $(n,1,\sigma_{\psi,n}^2,\kappa_n^2)$. Here $\Phi_\infty$ is a population constant depending only on $\mu_\infty,K_3,\diam(\Theta),r$.
\end{lemma}
\begin{proof}
We give the block case ($\bullet=\sub$, weights $\bar w$). The per-record case is identical with $(S,\gamma,\bar\psi_s,\bar A_s,\bar\sigma_{\psi,n}^2,\bar\kappa_n^2)$ replaced by $(n,1,\psi_i,A_i,\sigma_{\psi,n}^2,\kappa_n^2)$. Throughout we work on $\mathcal G_n$, where $H_{0,n}\succeq \mu_n I$ with $\mu_n=\mu_\infty/2$, the first-order condition $\sum_i\psi_i=0$ holds, and $\overline{B(\theta_0^\star,r)}\subset\mathrm{int}(\Theta)$ since $r\le\rho_\infty/4\le\frac12\mathrm{dist}(\theta_0^\star,\partial\Theta)$ by Lemma~\ref{lem:cm-reg}. Define
\begin{align*}
\tilde E_n^\sub:=\{\|B_\sub\|_\op\le\tfrac{\mu_n}{4}\}\cap\{\|g_\sub\|\le\tfrac{\mu_n r}{4}\}.
\end{align*}
Because $\sum_i(\bar w_i-\frac1n)=0$ on $\mathcal G_n$, we have the two boundary identities
\begin{align*}
G(\theta_0^\star;\bar \bw)&=G(\theta_0^\star;\bar \bw)-\underbrace{G(\theta_0^\star;n^{-1}\mathbf 1_n)}_{=0}=\sum_i(\bar w_i-\tfrac1n)\psi_i=g_\sub,\\
H(\theta_0^\star;\bar \bw)&=H_{0,n}+\sum_i(\bar w_i-\tfrac1n)A_i=H_{0,n}+B_\sub.
\end{align*}

\paragraph{Positive definiteness on the ball.} For $\theta\in\overline{B(\theta_0^\star,r)}$, the (A3) Lipschitzness gives $\|H(\theta;\bar \bw)-H(\theta_0^\star;\bar \bw)\|_\op\le K_3\|\theta-\theta_0^\star\|\le K_3 r\le\frac{\mu_n}4$, since $r\le\mu_\infty/(8K_3)=\mu_n/(4K_3)$. On $\tilde E_n^\sub$, $H(\theta_0^\star;\bar \bw)=H_{0,n}+B_\sub\succeq\mu_n I-\frac{\mu_n}4 I$, so by Weyl's inequality
\begin{align*}
H(\theta;\bar \bw)\succeq\mu_n I-\tfrac{\mu_n}4 I-\tfrac{\mu_n}4 I=\tfrac{\mu_n}2 I,\qquad\theta\in\overline{B(\theta_0^\star,r)}.
\end{align*}
Thus $L(\cdot;\bar \bw)$ is $\frac{\mu_n}2$-strongly convex on this ball.

\paragraph{The minimizer lies in the open ball.} The function $L(\cdot;\bar \bw)=\sum_{i=1}^n \bar w_i l(X_i,\cdot)$ is continuous on the compact ball, so it attains a minimum there, and we show the minimum cannot occur on $\partial B(\theta_0^\star,r)$. For $\theta\in\partial B(\theta_0^\star,r)$, the fundamental theorem of calculus and the boundary identity for $G(\theta_0^\star;\bar \bw)$ give $G(\theta;\bar \bw)=g_\sub+\bar H_\theta(\theta-\theta_0^\star)$ with $\bar H_\theta:=\int_0^1 H(\theta_0^\star+t(\theta-\theta_0^\star);\bar \bw)\,dt\succeq\frac{\mu_n}2 I$ by the positive-definiteness bound above. Taking the inner product with $\theta-\theta_0^\star$ (of norm $r$) and using Cauchy-Schwarz and $\|g_\sub\|\le\frac{\mu_n r}4$,
\begin{align*}
(\theta-\theta_0^\star)^\top G(\theta;\bar \bw)\ge -r\|g_\sub\|+\tfrac{\mu_n}2 r^2\ge -\tfrac{\mu_n r^2}4+\tfrac{\mu_n}2 r^2=\tfrac{\mu_n r^2}4>0.
\end{align*}
The outward radial derivative is strictly positive on the whole sphere, so the minimum lies in the open ball. Call it $\theta_\sub^\star$, a stationary point with $G(\theta_\sub^\star;\bar \bw)=0$. By (A4) each $l(X_i,\cdot)$ is convex, so $L(\cdot;\bar \bw)$ is convex on $\Theta$ and this interior stationary point is a global minimizer.

\paragraph{Linearization.} Define $\Delta_\sub:=\theta_\sub^\star-\theta_0^\star$ and $H_\sub^\sharp:=\int_0^1 H(\theta_0^\star+t\Delta_\sub;\bar \bw)\,dt$. Recall from the boundary identity that $G(\theta_0^\star;\bar \bw)=\sum_i(\bar w_i-\frac1n)\psi_i=\sum_s(u_s-\frac1S)\bar\psi_s=g_\sub$. By the fundamental theorem of calculus,
\begin{align*}
0=G(\theta_\sub^\star;\bar \bw)=G(\theta_0^\star;\bar\bw)+H_\sub^\sharp \Delta_\sub=g_\sub+H_\sub^\sharp\Delta_\sub.
\end{align*}
Thus $H_\sub^\sharp\Delta_\sub=-g_\sub$, and since $\lambda_{\min}(H_\sub^\sharp)\ge\frac{\mu_n}2$ by integrating the positive-definiteness bound, $\|\Delta_\sub\|\le\frac{2}{\mu_n}\|g_\sub\|\le\frac r2$. Subtracting $-H_{0,n}^{-1}\sqrt n g_\sub$ from $\sqrt n\Delta_\sub=-(H_\sub^\sharp)^{-1}\sqrt n g_\sub$ and using the resolvent identity $(H_\sub^\sharp)^{-1}-H_{0,n}^{-1}=-H_{0,n}^{-1}(H_\sub^\sharp-H_{0,n})(H_\sub^\sharp)^{-1}$,
\begin{align*}
\sqrt n \Delta_\sub=-H_{0,n}^{-1}\sqrt n g_\sub+r_\sub,
\qquad
r_\sub=\sqrt n H_{0,n}^{-1}(H_\sub^\sharp-H_{0,n})(H_\sub^\sharp)^{-1}g_\sub.
\end{align*}
Writing $H_\sub^\sharp-H_{0,n}=(H_\sub^\sharp-H(\theta_0^\star;\bar w))+B_\sub$, the integral form and (A3) give
\begin{align*}
\|H_\sub^\sharp-H(\theta_0^\star;\bar w)\|_\op
&=\Big\|\int_0^1 [H(\theta_0^\star+t\Delta_\sub;\bar \bw)-H(\theta_0^\star;\bar w)] dt\Big\|_\op\\
&\le K_3\|\Delta_\sub\|\int_0^1 t dt=\tfrac{K_3}2\|\Delta_\sub\|.
\end{align*}
Since $\frac{K_3}2\|\Delta_\sub\|\le\frac{K_3}{\mu_n}\|g_\sub\|$ and $\|H_{0,n}^{-1}\|_\op,\|(H_\sub^\sharp)^{-1}\|_\op\le\frac{2}{\mu_n}$,
\begin{align*}
\|r_\sub\|\le\frac{4\sqrt n}{\mu_n^2}\Big(\|B_\sub\|_\op\|g_\sub\|+\tfrac{K_3}{\mu_n}\|g_\sub\|^2\Big).
\end{align*}

\paragraph{Remainder expectation.} Taking $\E_u[\,\cdot\,\mathbf 1_{\tilde E_n^\sub}\mid \bX_{1:n},\pi]$, dropping the nonnegative indicator on the right, and applying Cauchy-Schwarz with $\|B_\sub\|_\op\le\|B_\sub\|_F$ together with Lemma~\ref{lem:cm-diri} ($\E_u\|g_\sub\|^2=\frac{\bar\sigma_{\psi,n}^2}{S\gamma+1}$ and $\E_u\|B_\sub\|_F^2=\frac{\bar\kappa_n^2}{S\gamma+1}$),
\begin{align*}
\E_u[\|r_\sub\|\mathbf 1_{\tilde E_n^\sub}\mid \bX_{1:n},\pi]
\le\frac{4\sqrt n}{\mu_n^2}\Big(\frac{\bar\kappa_n\bar\sigma_{\psi,n}}{S\gamma+1}+\frac{K_3}{\mu_n}\frac{\bar\sigma_{\psi,n}^2}{S\gamma+1}\Big)
\le\Phi_\infty \frac{\sqrt n}{S\gamma}(\bar\sigma_{\psi,n}^2+\bar\kappa_n^2),
\end{align*}
using $\bar\kappa_n\bar\sigma_{\psi,n}\le\frac12(\bar\kappa_n^2+\bar\sigma_{\psi,n}^2)$, $S\gamma+1\ge S\gamma$, and $\Phi_\infty:=\frac{4}{\mu_n^2}(\frac12+\frac{K_3}{\mu_n})$.

\paragraph{Localization probability.} By the union bound and conditional Markov,
\begin{align*}
\P(\tilde E_n^{\sub c}\mid \bX_{1:n},\pi)
&\le\P(\|B_\sub\|_F>\tfrac{\mu_n}4)+\P(\|g_\sub\|>\tfrac{\mu_n r}4)\\
&\le\frac{16 \bar\kappa_n^2}{\mu_n^2(S\gamma+1)}+\frac{16 \bar\sigma_{\psi,n}^2}{\mu_n^2 r^2(S\gamma+1)}=O\Big(\frac{\bar\sigma_{\psi,n}^2+\bar\kappa_n^2}{S\gamma}\Big),
\end{align*}
where $\mu_n=\mu_\infty/2$, again by Lemma~\ref{lem:cm-diri} and the boundedness of $\bar\kappa_n^2,\bar\sigma_{\psi,n}^2$ on $\mathcal G_n$.
\end{proof}

\subsubsection{Proof of Proposition~\ref{prop:choice-alpha}}
\label{app:proof-prop}

On $\mathcal G_n$, we define the (unscaled) linearized characterization.
\begin{align}\label{eqn:vm-lin}
r_\bullet:=\theta_\bullet^\star-\theta_0^\star+H_{0,n}^{-1}g_\bullet,
\qquad\text{so}\qquad
\theta_\bullet^\star-\theta_0^\star=-H_{0,n}^{-1}g_\bullet+r_\bullet,\quad\bullet\in\{\Diri,\sub\},
\end{align}
which is an identity by definition, so $\sqrt n r_\bullet$ is the remainder of Lemma~\ref{lem:cm-lin} (here $r$ without a subscript continues to denote the localization radius). By that lemma, dividing its remainder bound by $\sqrt n$, on the localization event $\tilde E_n^\bullet$ where $\|B_\bullet\|_\op\le\mu_n/4$ and $\|g_\bullet\|\le\mu_n r/4$,
\begin{align}\label{eqn:vm-Rptwise}
\|r_\bullet\|\le\frac{4}{\mu_n^2}(\|B_\bullet\|_\op+\tfrac{K_3}{\mu_n}\|g_\bullet\|)\|g_\bullet\|.
\end{align}

\paragraph{Marginal moment rates.} By the partition mean of Lemma~\ref{lem:cm-diri} and the tower property, $\E_\pi[\bar\sigma_{\psi,n}^2\mid \bX_{1:n}]=\tr\E_\pi[\bar\Sigma_{\psi,n}\mid \bX_{1:n}]=\frac{S-1}{n-1}\sigma_{\psi,n}^2$, and the same computation with $\psi_i\mapsto A_i$ in the Frobenius geometry gives $\E_\pi[\bar\kappa_n^2\mid \bX_{1:n}]=\frac{S-1}{n-1}\kappa_n^2$. Since $\sigma_{\psi,n}^2\le(C)^2$ and $\kappa_n^2\le 4d_\theta K_2^2$ deterministically by (A3),
\begin{align}\label{eqn:vm-margrate}
\E_{\bX_{1:n},\pi}\bar\sigma_{\psi,n}^2=\frac{S-1}{n-1}\E_{\bX_{1:n}}\sigma_{\psi,n}^2=O(S/n),
\qquad
\E_{\bX_{1:n},\pi}\bar\kappa_n^2=O(S/n).
\end{align}

\begin{lemma}
\label{lem:cm-rem2}
For $\bw\sim\Diri(\mathbf 1_n)$ and $\bu\sim\Diri(\gamma\mathbf 1_S)$, $\gamma\in\Gamma=[\underline\gamma,\overline\gamma]\subset(0,\infty)$, with $r_\bullet$ the linear approximation \eqref{eqn:vm-lin} defined on $\mathcal G_n$ (so $\E\|r_\bullet\|^2$ abbreviates $\E[\|r_\bullet\|^2\mathbf 1_{\mathcal G_n}]$), uniformly over $\gamma\in\Gamma$,
\begin{align*}
\E\|r_\Diri\|^2=O(n^{-2}),\qquad \E\|r_\sub\|^2=O(\gamma^{-3}n^{-2})+O(\gamma^{-3}S^{-1}n^{-2}).
\end{align*}
\end{lemma}
\begin{proof}
We prove the block bound. The per-record one is identical with $(S,\gamma,g_\sub,B_\sub)\mapsto(n,1,g_\Diri,B_\Diri)$ and no partition averaging. From \eqref{eqn:vm-Rptwise},
\begin{align*}
\|r_\sub\|\le\frac{4}{\mu_n^2}(\|B_\sub\|_\op+\tfrac{K_3}{\mu_n}\|g_\sub\|)\|g_\sub\|,
\end{align*}
on $\tilde E_n^\sub$, which is second order in the pair $(g_\sub,B_\sub)$. Hence on $\tilde E_n^\sub$, with $\|B_\sub\|_\op\le\|B_\sub\|_F$,
\begin{align*}
\|r_\sub\|^2\le\frac{32}{\mu_n^4} \|B_\sub\|_F^2\|g_\sub\|^2+\frac{32K_3^2}{\mu_n^6} \|g_\sub\|^4,
\end{align*}
and Cauchy-Schwarz with the fourth moments of Lemma~\ref{lem:cm-four} give
\begin{align*}
    \E[\|r_\sub\|^2\mathbf 1_{\tilde E_n^\sub}]
&\le\frac{32}{\mu_n^4}(\E\|B_\sub\|_F^4)^{1/2}(\E\|g_\sub\|^4)^{1/2}
+\frac{32K_3^2}{\mu_n^6} \E\|g_\sub\|^4, \\ 
&=O(\gamma^{-3}n^{-2})+O(\gamma^{-3}S^{-1}n^{-2}),
\end{align*}
uniformly over $\gamma\in\Gamma$. Off the localization event, $\tilde E_n^{\sub c}\subseteq\{\|B_\sub\|_F>\mu_n/4\}\cup\{\|g_\sub\|>\mu_n r/4\}$, the Markov inequality and the observation $\|r_\sub\|\le D'$ leads to
\begin{align*}
    \E_{\bX_{1:n},\pi} \P(\tilde E_n^{\sub c}\mid \bX_{1:n},\pi)
&\le(\tfrac{4}{\mu_n})^4\E\|B_\sub\|_F^4+(\tfrac{4}{\mu_n r})^4\E\|g_\sub\|^4,\\ 
&=O(\gamma^{-3}n^{-2})+O(\gamma^{-3}S^{-1}n^{-2}),
\end{align*}
whence $\E[\|r_\sub\|^2\mathbf 1_{\tilde E_n^{\sub c}}]\le(D')^2 \E_{\bX_{1:n},\pi}\P(\tilde E_n^{\sub c}\mid \bX_{1:n},\pi)=O(\gamma^{-3}n^{-2})+O(\gamma^{-3}S^{-1}n^{-2})$. Adding the two parts proves the claim.
\end{proof}

\paragraph{Leading conditional second moments.} For any deterministic $M$ and mean-zero $g$, $\E\|Mg\|^2=\tr(M^\top M \E[gg^\top])$. Applying this with $M=H_{0,n}^{-1}$ and Lemma~\ref{lem:cm-diri} ($\E_w[g_\Diri g_\Diri^\top\mid \bX_{1:n}]=\Sigma_{\psi,n}/(n+1)$, $\E_u[g_\sub g_\sub^\top\mid \bX_{1:n},\pi]=\bar\Sigma_{\psi,n}/(S\gamma+1)$),
\begin{align*}
\E_w[\|H_{0,n}^{-1}g_\Diri\|^2\mid \bX_{1:n}]=\frac{\tr(H_{0,n}^{-2}\Sigma_{\psi,n})}{n+1},
\qquad
\E_u[\|H_{0,n}^{-1}g_\sub\|^2\mid \bX_{1:n},\pi]=\frac{\tr(H_{0,n}^{-2}\bar\Sigma_{\psi,n})}{S\gamma+1}.
\end{align*}
In addition, taking $\E_\pi[\cdot\mid \bX_{1:n}]$ in the second and using the partition mean of Lemma~\ref{lem:cm-diri} yields $\E_{u,\pi}[\|H_{0,n}^{-1}g_\sub\|^2\mid \bX_{1:n}]=\frac{S-1}{(S\gamma+1)(n-1)}\tr(H_{0,n}^{-2}\Sigma_{\psi,n})$. Writing $a_n:=\E_{\bX_{1:n}}[\mathbf 1_{\mathcal G_n}\tr(H_{0,n}^{-2}\Sigma_{\psi,n})]$, which satisfies $0\le a_n\le\mu_n^{-2}(C)^2=O(1)$ with $\mu_n=\mu_{\infty}/2$, since $\tr(H_{0,n}^{-2}\Sigma_{\psi,n})\le\|H_{0,n}^{-2}\|_\op\tr\Sigma_{\psi,n}\le\mu_n^{-2}(C)^2$ on $\mathcal G_n$,
\begin{align}\label{eqn:vm-lead}
\E\|H_{0,n}^{-1}g_\Diri\|^2=\frac{a_n}{n+1},\qquad
\E\|H_{0,n}^{-1}g_\sub\|^2=\frac{(S-1)a_n}{(S\gamma+1)(n-1)}.
\end{align}

\paragraph{Conditional variance and matching.}  In \eqref{eqn:vm-lin}, the observation $\E_\bullet[g_\bullet\mid \bX_{1:n}]=0$ gives $\theta_\bullet^\star-\E_\bullet[\theta_\bullet^\star\mid \bX_{1:n}]=-H_{0,n}^{-1}g_\bullet+(r_\bullet-\E_\bullet[r_\bullet\mid \bX_{1:n}])$ and thus 
\begin{align*}
\Var_\bullet(\theta_\bullet^\star\mid \bX_{1:n})=\E_\bullet[\|H_{0,n}^{-1}g_\bullet\|^2\mid \bX_{1:n}]+\Var_\bullet(r_\bullet\mid \bX_{1:n})-2 \E_\bullet[\langle H_{0,n}^{-1}g_\bullet, r_\bullet-\E_\bullet[r_\bullet\mid \bX_{1:n}]\rangle\mid \bX_{1:n}].
\end{align*}
We then take $\E_{\bX_{1:n}}$ to both sides. The remainder variance is controlled by Lemma~\ref{lem:cm-rem2} via $\E_{\bX_{1:n}}\Var_\bullet(r_\bullet\mid \bX_{1:n})\le\E\|r_\bullet\|^2$, and the cross term by Cauchy-Schwarz with \eqref{eqn:vm-lead} and Lemma~\ref{lem:cm-rem2}, $|\E\langle H_{0,n}^{-1}g_\bullet,r_\bullet-\E_\bullet[r_\bullet\mid \bX_{1:n}]\rangle|\le\sqrt{\E\|H_{0,n}^{-1}g_\bullet\|^2}\sqrt{\E\|r_\bullet\|^2}$. For the per-record law $\E\|r_\Diri\|^2=O(n^{-2})$ and $\E\|H_{0,n}^{-1}g_\Diri\|^2=\frac{a_n}{n+1}=O(1/n)$, so both error terms are $O(n^{-3/2})$. For the block law $\E\|r_\sub\|^2=O(\gamma^{-3}n^{-2})$, the cross term is $\sqrt{\E\|H_{0,n}^{-1}g_\sub\|^2} O(\gamma^{-3/2}n^{-1})$, and the remainder variance is $O(\gamma^{-3}n^{-2})$. Hence
\begin{align*}
    &\E_{\bX_{1:n}}[\Var_\Diri(\theta_\Diri^\star\mid \bX_{1:n})]=\frac{a_n}{n+1}+O(n^{-3/2}),
\\
&\E_{\bX_{1:n}}[\Var_\sub(\theta_\sub^\star\mid \bX_{1:n})]=\frac{(S-1)a_n}{(S\gamma+1)(n-1)}+R_\sub(\gamma),
\end{align*}
where $\Var_\sub(\theta_\sub^\star\mid \bX_{1:n})=\Var_{\bu,\pi}(\theta_\sub^\star\mid \bX_{1:n})$ is the joint $(\bu,\pi)$ variance, and $R_\sub(\gamma):=\sqrt{\E\|H_{0,n}^{-1}g_\sub\|^2} O(\gamma^{-3/2}n^{-1})+O(\gamma^{-3}n^{-2})$ with $\E\|H_{0,n}^{-1}g_\sub\|^2=\frac{(S-1)a_n}{(S\gamma+1)(n-1)}$. For the {partition-conditional} variance of the statement, we use the law of total variance over $(\bu,\pi)$ given $\bX_{1:n}$,
\begin{align*}
\E_{\bX_{1:n},\pi}\Var_\bu[\theta^\star_\sub\mid \bX_{1:n},\pi]
=\E_{\bX_{1:n}}\Var_{\bu,\pi}[\theta^\star_\sub\mid \bX_{1:n}]
-\E_{\bX_{1:n}}\Var_\pi(\E_\bu[\theta^\star_\sub\mid \bX_{1:n},\pi]).
\end{align*}
By the linearization~\eqref{eqn:vm-lin} and the centered score $\E_\bu[g_\sub\mid \bX_{1:n},\pi]=0$, $\E_\bu[\theta^\star_\sub\mid \bX_{1:n},\pi]-\theta^\star_0(\bX_{1:n})=\E_\bu[r_\sub\mid \bX_{1:n},\pi]$, and since $\theta^\star_0(\bX_{1:n})$ is $\pi$-free,
\begin{align*}
\E_{\bX_{1:n}}\Var_\pi(\E_\bu[\theta^\star_\sub\mid \bX_{1:n},\pi])
=\E_{\bX_{1:n}}\Var_\pi(\E_\bu[r_\sub\mid \bX_{1:n},\pi])
\le \E_{X}\E_{\pi,\bu}\|r_\sub\|^2=O(\gamma^{-3}n^{-2}),
\end{align*}
by Jensen and Lemma~\ref{lem:cm-rem2}. This partition-mean correction is of the order of the $O(\gamma^{-3}n^{-2})$ term already in $R_\sub(\gamma)$ and is absorbed into it, so $\E_{\bX_{1:n},\pi}\Var_\bu[\theta^\star_\sub\mid \bX_{1:n},\pi]=\frac{(S-1)a_n}{(S\gamma+1)(n-1)}+R_\sub(\gamma)$, the same leading term. Consequently,
\begin{align*}
\E_{\bX_{1:n},\pi}\Var_\bu[\theta^\star_\sub\mid \bX_{1:n},\pi]-\E_{\bX_{1:n}}\Var_\bw[\theta^\star_\Diri\mid \bX_{1:n}]
=a_{S,n}(\gamma)a_n+R_\sub(\gamma)+O(n^{-3/2}).
\end{align*}
Using the exact value $\E\|H_{0,n}^{-1}g_\sub\|^2=\frac{(S-1)a_n}{(S\gamma+1)(n-1)}$, the cross term $\sqrt{\E\|H_{0,n}^{-1}g_\sub\|^2}\,O(\gamma^{-3/2}n^{-1})=O\!\Big(\gamma^{-3/2}\sqrt{\tfrac{S-1}{S\gamma+1}}\,n^{-3/2}\Big)$, so $R_\sub(\gamma)=O(\gamma^{-3/2}\sqrt{\tfrac{S-1}{S\gamma+1}}\,n^{-3/2}+\gamma^{-3}n^{-2})$. This gives the stated expansion
\begin{align*}
    \E_{\bX_{1:n},\pi}\Var_\bu[\theta^\star_\sub\mid \bX_{1:n},\pi]&-\E_{\bX_{1:n}}\Var_\bw[\theta^\star_\Diri\mid \bX_{1:n}]
\\
&=a_{S,n}(\gamma)a_n+O((1+\sqrt{\tfrac{S-1}{S\gamma+1}}\gamma^{-3/2})n^{-3/2}+\gamma^{-3}n^{-2}).
\end{align*}

At $\gamma=\gamma^\star(n,S)$ one has $S\gamma^\star+1=\frac{(S-1)(n+1)}{n-1}$, so $\frac{S-1}{(S\gamma^\star+1)(n-1)}=\frac{1}{n+1}$ and the leading terms coincide exactly, i.e., $\E\|H_{0,n}^{-1}g_\sub\|^2=\frac{a_n}{n+1}=O(1/n)$ at $\gamma^\star$. Thus,
\begin{align*}
R_\sub(\gamma^\star)=\sqrt{\tfrac{a_n}{n+1}} O((\gamma^\star)^{-3/2}n^{-1})+O((\gamma^\star)^{-3}n^{-2})=O((\gamma^\star)^{-3/2}n^{-3/2}).
\end{align*}
Since $\gamma^\star=\{S(n+1)-2n\}/\{S(n-1)\}\ge\frac{S-2}{S}$ for $S\ge 3$, $(\gamma^\star)^{-3/2}\le(\tfrac{S}{S-2})^{3/2}$, and therefore
\begin{align*}
&\E_{\bX_{1:n},\pi}[\Var_{\bu}(\theta_\sub^\star\mid \bX_{1:n},\pi)]-\E_{\bX_{1:n}}[\Var_{\bw}(\theta_\Diri^\star\mid \bX_{1:n})]\\
&=R_\sub(\gamma^\star)+O(n^{-3/2})=O((S/(S-2))^{3/2}n^{-3/2}),
\end{align*}
which is the variance-matching claim. For fixed $S\ge3$ this is $o(n^{-1})$, so the scaled variances agree in the limit, $n \{\E_{\bX_{1:n},\pi}\Var_\bu[\theta^\star_\sub\mid \bX_{1:n},\pi]-\E_{\bX_{1:n}}\Var_\bw[\theta^\star_\Diri\mid \bX_{1:n}]\}\to0$.

\paragraph{Mean matching (any $\gamma>0$).} On $\mathcal G_n$, $\E_\bullet[H_{0,n}^{-1}g_\bullet\mid \bX_{1:n}]=H_{0,n}^{-1}\E_\bullet[g_\bullet\mid \bX_{1:n}]=0$, so by \eqref{eqn:vm-lin}, $\E[(\theta_\bullet^\star-\theta_0^\star)\mathbf 1_{\mathcal G_n}]=\E[r_\bullet\mathbf 1_{\mathcal G_n}]$ and, adding the $O(e^{-c'n})$ contribution of $\mathcal G_n^c$, $\|\E[\theta_\bullet^\star-\theta_0^\star]\|\le\E\|r_\bullet\|+O(e^{-c'n})$. The mean is the expectation of the quadratic defect, so we bound the first moment $\E\|r_\bullet\|$ directly off \eqref{eqn:vm-Rptwise}, $\|r_\bullet\|\le\frac{4}{\mu_n^2}(\|B_\bullet\|_\op+\frac{K_3}{\mu_n}\|g_\bullet\|)\|g_\bullet\|$ on $\tilde E_n^\bullet$, using only the second moments of Lemma~\ref{lem:cm-diri}. After averaging the partition, $\E\|g_\bullet\|^2=c_\bullet \sigma_{\psi,n}^2$ and $\E\|B_\bullet\|_F^2=c_\bullet \kappa_n^2$ with
\begin{align*}
c_\Diri=\tfrac1{n+1},\qquad c_\sub=\tfrac{S-1}{(S\gamma+1)(n-1)} .
\end{align*}
By Cauchy-Schwarz, $\E[\|r_\bullet\|\mathbf 1_{\tilde E_n^\bullet}]\le\frac{4}{\mu_n^2}(\sqrt{\E\|B_\bullet\|_F^2} \sqrt{\E\|g_\bullet\|^2}+\frac{K_3}{\mu_n}\E\|g_\bullet\|^2)=O(c_\bullet)$. Off $\tilde E_n^\bullet$, $\|r_\bullet\|\le D'$ and $\E_{\bX_{1:n},\pi}\P(\tilde E_n^{\bullet c})=O(c_\bullet)$ (the second-moment localization bound of Lemma~\ref{lem:cm-lin}), so $\E\|r_\bullet\|=O(c_\bullet)$. Hence $\E\|r_\Diri\|=O(\tfrac1{n+1})$ and $\E\|r_\sub\|=O(c_\sub)$, and by the tower property and the triangle inequality
\begin{align*}
\|\E_{\bX_{1:n},\pi}[\E_{\bu}[\theta_\sub^\star\mid \bX_{1:n},\pi]-\E_{\bw}[\theta_\Diri^\star\mid \bX_{1:n}]]\|
&=\|\E\theta_\sub^\star-\E\theta_\Diri^\star\|\\
&\le\E\|r_\sub\|+\E\|r_\Diri\|, \\ 
&=O\Big(\tfrac{S-1}{(S\gamma+1)(n-1)}+\tfrac1{n+1}\Big).
\end{align*}
Since $\frac{S-1}{S\gamma+1}\le\frac1\gamma$, this is $O(n^{-1})$ for every fixed $\gamma>0$, uniformly in $S$. \qed

\subsubsection{Proof of Theorem~\ref{thm:solution-gap}}
\label{app:proof-thm}

We first introduce a known theoretical result of \cite{chen:etal:10} that quantifies the finite approximation error between a sum-type random variable and a standard Gaussian variable in distribution.
\begin{theorem}
\label{thm:cgs}
Let $\zeta_1,\dots,\zeta_K$ be independent, mean-zero random variables with $\sum_{k=1}^K\Var(\zeta_k)=1$. Let $F$ be the distribution function of $\sum_{k=1}^K\zeta_k$ and $\Phi$ that of $\mathcal N(0,1)$. Then
\begin{align*}
\|F-\Phi\|_1\ \le\ \sum_{k=1}^K B(\zeta_k) \E|\zeta_k|^3,
\end{align*}
where the constant factor $B(X)=\tfrac{2\Var(X)\|{\cal L}(X^*)-{\cal L}(X)\|_1}{\E|X^3|}$ and $X^*$ is the zero-bias version of $X$. Note $B(X)$ is scale-invariant, i.e.\ $B(a\zeta)=B(\zeta)$ for $a\ne0$. In one dimension $\|F-\Phi\|_1=W_1(F,\Phi)$, so the left-hand side is the 1-Wasserstein distance.
\end{theorem}

We apply this bound to the projected linearized bootstrap laws.

\paragraph{Notation and scalar variances.} Fix $a$ with $\|a\|=1$, and on $\mathcal G_n$ (Lemma~\ref{lem:cm-reg}(ii)) define
\begin{align*}
b_a:=H_{0,n}^{-1}a,\qquad \xi_i^a:=b_a^\top\psi_i,\qquad \bar\xi_s^a:=b_a^\top\bar\psi_s=\tfrac1{n_S}\sum_{i\in I_s}\xi_i^a,\qquad \hat v_a:=\tfrac1n\sum_i(\xi_i^a)^2.
\end{align*}
Note $\sum_i\xi_i^a=b_a^\top\sum_i\psi_i=0$ because of the first-order condition. Recall $g_\sub=\sum_{s=1}^S(u_s-\tfrac1S)\bar\psi_s$ and $g_\Diri=\sum_{i=1}^n(w_i-\tfrac1n)\psi_i$. Denote by $V_\bullet:=H_{0,n}^{-1}\Cov(\sqrt n g_\bullet)H_{0,n}^{-1}$ the covariance of the linearized terms. With the projection vector $a$, Lemma~\ref{lem:cm-diri} implies
\begin{align*}
v_\Diri^a:=a^\top V_\Diri a=\frac{1}{n+1}\sum_{i=1}^n(\xi_i^a)^2,
\qquad
v_\sub^a:=a^\top V_\sub a=\frac{n}{S\gamma+1}\cdot\frac1S\sum_{s=1}^S(\bar\xi_s^a)^2,
\end{align*}
using $a^\top H_{0,n}^{-1}=b_a^\top$ (symmetry) and $b_a^\top\Sigma_{\psi,n}b_a=\frac1n\sum_i(\xi_i^a)^2$, $b_a^\top\bar\Sigma_{\psi,n}b_a=\frac1S\sum_s(\bar\xi_s^a)^2$. We also define
\begin{align*}
\rho_\Diri^a:=\frac{\sum_i|\xi_i^a|^3}{\Bigl(\sum_i(\xi_i^a)^2\Bigr)^{3/2}},\qquad
\rho_\sub^a:=\frac{\sum_s|\bar\xi_s^a|^3}{\Bigl(\sum_s(\bar\xi_s^a)^2\Bigr)^{3/2}}.
\end{align*}

By Lemma~\ref{lem:cm-diri} and the partition mean, for every $\gamma>0$,
\begin{align}\label{eqn:thm-varmatch}
\E_\pi[v_\sub^a\mid \bX_{1:n}] =\frac{n}{S\gamma+1}\,\frac{S-1}{n-1}\,b_a^\top\Sigma_{\psi,n}b_a
=R_{S,n}(\gamma)\,v_\Diri^a,
\end{align}
where $R_{S,n}(\gamma):=\frac{(S-1)(n+1)}{(S\gamma+1)(n-1)}$, so that, with $a_{S,n}(\gamma)=\tfrac{S-1}{(S\gamma+1)(n-1)}-\tfrac1{n+1}$,
\begin{align}\label{eqn:thm-varbias}
\E_\pi[v_\sub^a\mid \bX_{1:n}]-v_\Diri^a=\pm n \hat v_a\, |a_{S,n}(\gamma)|.
\end{align}
At $\gamma=\gamma^\star=\{S(n+1)-2n\}/\{S(n-1)\}$, we have $S\gamma^\star+1=\frac{(S-1)(n+1)}{n-1}$, hence $R_{S,n}(\gamma^\star)=1$ and $\E_\pi[v_\sub^a\mid\bX_{1:n}]=v_\Diri^a$.

\paragraph{Proof scheme.} Lemma~\ref{lem:cm-lin} linearizes each bootstrap solution as $a^\top\sqrt n(\theta_\bullet^\star-\theta_0^\star)=-\sqrt n b_a^\top g_\bullet+R_a^\bullet$, where the projected remainder $R_a^\bullet:=a^\top r_\bullet$ obeys $|R_a^\bullet|\le\|r_\bullet\|$. Replacing each projected law by the law of its linear part, therefore, costs only the remainder. Conditionally on $(\bX_{1:n},\pi)$, we observe, by definition of the 1-Wasserstein distance, 
\begin{align}\label{eqn:thm-w1remainder}
W_1(\mathcal L_\bullet^{a,\bX_{1:n},\pi},\ \mathrm{Law}(-\sqrt n b_a^\top g_\bullet\mid \bX_{1:n},\pi))\le\E[ |R_a^\bullet|\mid \bX_{1:n},\pi ]\le\E[ \|r_\bullet\|\mid \bX_{1:n},\pi ].
\end{align}
Lemma~\ref{lem:cm-lin} and \eqref{eqn:vm-margrate} show that the remainder decays by $O(n^{-1/2})$ eventually in marginalization.

By the Gamma representation $w_i=G_i/Y$, $u_s=\Gamma_s/Y_S$ (Lemma~\ref{lem:cm-gamma}) and the centrings $\sum_i\xi_i^a=0$, $\sum_s\bar\xi_s^a=0$, each linear part can be represented as
\begin{align}\label{eqn:thm-gammascalar}
-\sqrt n b_a^\top g_\Diri=\frac nY\Bigl(-\tfrac1{\sqrt n}\sum_i(G_i-1)\xi_i^a\Bigr), 
-\sqrt n b_a^\top g_\sub=\frac{1}{Y_S}\Bigl(-\sum_s(\Gamma_s-\gamma)\sqrt n \bar\xi_s^a\Bigr).
\end{align}
The sum-type variables in the parentheses consist of independent terms, and we apply Theorem~\ref{thm:cgs} to each one.

\begin{lemma}
\label{lem:cm-be}
Conditional on $(\bX_{1:n},\pi)$, with $v_\Diri^a,v_\sub^a>0$ and $S\gamma>2$,
\begin{align*}
&W_1(\mathrm{Law}(-\sqrt n b_a^\top g_\Diri\mid \bX_{1:n}), \mathcal N(0,v_\Diri^a))\lesssim\sqrt{v_\Diri^a} (\rho_\Diri^a+n^{-1/2}),    \\ 
&W_1(\mathrm{Law}(-\sqrt n b_a^\top g_\sub\mid \bX_{1:n},\pi), \mathcal N(0,v_\sub^a))\lesssim\sqrt{v_\sub^a} (\rho_\sub^a+(S\gamma)^{-1/2}),
\end{align*}
the suppressed constants being the centered-Gamma third-moment factors of Theorem~\ref{thm:cgs}.
\end{lemma}
\begin{proof}
We show the block case. The per-record one is identical after replacing $(\Gamma_s,\gamma,S)$ by $(G_i,1,n)$. By the Gamma representation \eqref{eqn:thm-gammascalar}, let's define $-\sqrt n b_a^\top g_\sub=T/Y_S$ where $T:=-\sum_s(\Gamma_s-\gamma)\sqrt n \bar\xi_s^a$ is a sum of $S$ independent mean-zero terms with variance $\tau^2:=\Var(T)=\gamma n\sum_s(\bar\xi_s^a)^2$ for $\Gamma_s \sim \mathrm{Gamma}(\gamma,1)$. 
Define $\zeta_s:=(\Gamma_s-\gamma)\sqrt n \bar\xi_s^a/\tau$ so that $\sum_s\Var\zeta_s=1$. Theorem~\ref{thm:cgs} applies with the common scale-invariant factor $B(\zeta_s)=B(\Gamma_1-\gamma)$ and yields
\begin{align*}
W_1(T/\tau,\mathcal N(0,1))\le B(\Gamma_1-\gamma)\E|\Gamma_1-\gamma|^3 \frac{\sum_s|\bar\xi_s^a|^3}{\Bigl(\gamma\sum_s(\bar\xi_s^a)^2\Bigr)^{3/2}}=\frac{B(\Gamma_1-\gamma)\E|\Gamma_1-\gamma|^3}{\gamma^{3/2}} \rho_\sub^a,
\end{align*}
For the normalizer, write $-\sqrt n b_a^\top g_\sub=\frac{S\gamma}{Y_S}(-\frac{T}{S\gamma})$. As $Y_S\sim\mathrm{Gamma}(S\gamma,1)$, the multiplicative error has second moment $\E(\frac{S\gamma}{Y_S}-1)^2=\frac{S\gamma+2}{(S\gamma-1)(S\gamma-2)}=O((S\gamma)^{-1})$  for $S\gamma>2$ and Cauchy-Schwarz turns
\begin{align*}
    W_1(-\sqrt n b_a^\top g_\sub, T/(S\gamma))\leq \E\Big|\frac{S \gamma}{Y_S} - 1\Big| \Big|\frac{T}{S\gamma}\Big|&\leq \sqrt{\E\Bigl(\frac{S\gamma}{Y_S}-1\Bigr)^2}\sqrt{\E[T^2]} (S\gamma)^{-1}, \\ 
    &=O(\sqrt{v^a_\sub}(S\gamma)^{-1/2}), 
\end{align*}
since $\E[T]^2 = S\gamma(S\gamma+1)v^a_\sub$. Finally $-T/(S\gamma)$ has variance $v_\sub^a(1+\frac1{S\gamma})$, so
\begin{align*}
    &W_1(T/(S\gamma), \mathcal N(0,v^a_\sub))\\ 
    &\leq W_1(T/(S\gamma), \mathcal N(0,v^a_\sub(1+(S\gamma)^{-1}))) + W_1(N(0,v^a_\sub),N(0,v^a_\sub(1+(S\gamma)^{-1}))),
\end{align*}
where the second term reduces to $\sqrt{v^a_\sub}(S\gamma)^{-1}$. Now rescaling
\begin{align*}
\frac{S\gamma}{\tau}\times W_1(T/(S\gamma), \mathcal N(0,v^a_\sub(1+(S\gamma)^{-1})))&\leq W_1(T/\tau, \mathcal N\Bigl(0,v^a_\sub(1+(S\gamma)^{-1})\frac{(S\gamma)^2}{\tau^2}\Bigr)), \\ 
&=W_1(T/\tau, \mathcal N(0,1)).
\end{align*}
Since $\frac{\tau}{S\gamma}=\sqrt{v_\sub^a(1+\frac1{S\gamma})}$ and $|v_\sub^a|<\infty$, we have
\begin{align*}
    W_1(-\sqrt n b_a^\top g_\sub,\mathcal N(0,v^a_\sub))\lesssim \sqrt{v_\sub^a} (\rho_\sub^a+(S\gamma)^{-1/2}).
\end{align*}
The per-record case differs only in the normalizer moment, $\E(\frac nY-1)^2=\frac{n+2}{(n-1)(n-2)}=O(n^{-1})$.
\end{proof}

\begin{lemma}
\label{lem:cm-sdgap}
Recall $\hat v_a=\tfrac1n\sum_i(\xi_i^a)^2$. On $\mathcal G_n$ with $|\xi_i^a|\le B_a$ and
$\hat v_a>0$, for every $\gamma$ with $S\gamma\ge3$,
\begin{align*}
\E_\pi\bigl[\,|\sqrt{v_\sub^a}-\sqrt{v_\Diri^a}|\ \big|\ \bX_{1:n}\bigr]
\;\lesssim\;\sqrt{\hat v_a}\,\Bigl(\bigl|1-\sqrt{R_{S,n}(\gamma)}\bigr|+(S\gamma+1)^{-1/2}\Bigr).
\end{align*}
At $\gamma=\gamma^\star$ the first summand vanishes and $S\gamma^\star+1\asymp S$, so the bound
reduces to $\lesssim\sqrt{\hat v_a}\,S^{-1/2}$.
\end{lemma}
 
\begin{proof}
Write $\bar v:=\E_\pi[v_\sub^a\mid\bX_{1:n}]$ and split
\begin{align*}
\E_\pi|\sqrt{v_\sub^a}-\sqrt{v_\Diri^a}|
\;\le\;\underbrace{\bigl|\sqrt{\bar v}-\sqrt{v_\Diri^a}\bigr|}_{\text{bias}}
\;+\;\underbrace{\E_\pi\bigl|\sqrt{v_\sub^a}-\sqrt{\bar v}\bigr|}_{\text{fluctuation}}.
\end{align*}
By \eqref{eqn:thm-varmatch} and \eqref{eqn:thm-varbias} the bias equals
$\sqrt{v_\Diri^a}\,\bigl|1-\sqrt{R_{S,n}(\gamma)}\bigr|
\le\sqrt{\hat v_a}\,\bigl|1-\sqrt{R_{S,n}(\gamma)}\bigr|$,
since $v_\Diri^a=\tfrac{n}{n+1}\hat v_a\le\hat v_a$.
For the fluctuation, $|\sqrt x-\sqrt c|\le|x-c|/\sqrt c$ with $x=v_\sub^a$, $c=\bar v$ gives
$\E_\pi|\sqrt{v_\sub^a}-\sqrt{\bar v}|\le(\Var_\pi v_\sub^a)^{1/2}/\sqrt{\bar v}$, so it remains to
bound $\Var_\pi(v_\sub^a)$.

Put $x_i:=\xi_i^a$ (centered, $|x_i|\le B_a$), $V:=\sum_ix_i^2=n\hat v_a$, $m:=n_S$, $\delta_{ij}:=\mathbf 1[\pi(i)=\pi(j)]$. Since $\sum_s(\sum_{i\in I_s}x_i)^2=V+2A$ with $A:=\sum_{i<j}\delta_{ij}x_ix_j$ deterministic up to $A$, and $v_\sub^a=\frac{n}{S\gamma+1}M_2$, $M_2:=\frac1S\sum_s(\bar\xi_s^a)^2=\frac{V+2A}{Sm^2}$,
\begin{align*}
\Var_\pi(M_2)=\frac{4}{S^2m^4}\Var_\pi(A),\qquad
\Var_\pi(A)=\sum_{e,f}\Cov_\pi(\delta_e,\delta_f) w_ew_f\quad(w_e:=x_ix_j,\ e=\{i,j\}).
\end{align*}
By exchangeability of the uniform partition, $\Cov_\pi(\delta_e,\delta_f)$ depends only on the overlap $|e\cap f|\in\{0,1,2\}$ and is therefore {constant within each overlap class}, with
\begin{align*}
 |\Cov_\pi(\delta_e,\delta_f)|\leq |\P(\delta_e=1)+\P(\delta_e=1)\P(\delta_f=1)|   \le\frac1S+\frac1{S^2}\le\frac2S.
\end{align*}
Pulling that constant out of each class,
\begin{align*}
|\Var_\pi(A)|\le\frac2S(|\Sigma_2|+|\Sigma_1|+|\Sigma_0|),\qquad
\Sigma_2:=\sum_{e}w_e^2,\quad
\Sigma_1:=\sum_{|e\cap f|=1}w_ew_f,\quad
\Sigma_0:=\sum_{e\cap f=\varnothing}w_ew_f,
\end{align*}
where each $\Sigma_\ell$ is a signed weight sum (the $w_e=x_ix_j$ can be negative). Write $p_4:=\sum_i x_i^4\le(\sum_i x_i^2)^2=V^2$; the centring is $\sum_i x_i=0$, equivalently $\sum_{j\ne i}x_j=-x_i$. Each class is $O(V^2)$. Specifically, we see
\paragraph{Diagonal ($e=f$).} $\Sigma_2=\sum_{i<j}x_i^2x_j^2=\tfrac12(V^2-p_4)$, so $|\Sigma_2|\le\tfrac12V^2$.

\paragraph{Shared index ($e=\{i,j\}$, $f=\{i,k\}$, $j\ne k$).} Here $w_ew_f=x_i^2x_jx_k$, so summing over the shared index $i$,
\begin{align*}
\Sigma_1=\sum_i x_i^2\Bigl[\Bigl(\sum_{j\ne i}x_j\Bigr)^2-\sum_{j\ne i}x_j^2\Bigr]
=\sum_i x_i^2[x_i^2-(V-x_i^2)]=\sum_i x_i^2(2x_i^2-V)=2p_4-V^2,
\end{align*}
using $\sum_{j\ne i}x_j=-x_i$ (centring) and $\sum_{j\ne i}x_j^2=V-x_i^2$, so $|\Sigma_1|\le V^2$.

\paragraph{Disjoint ($e\cap f=\varnothing$).} From $\sum_e w_e=\tfrac12((\sum_i x_i)^2-V)=-\tfrac V2$ (centring) we get $(\sum_e w_e)^2=\tfrac{V^2}4$, and since $\sum_{e,f}w_ew_f=(\sum_e w_e)^2$ splits over the three classes, $\Sigma_0=\tfrac{V^2}4-\Sigma_2-\Sigma_1=\tfrac34V^2-\tfrac32p_4$, so $|\Sigma_0|\le\tfrac34V^2$.
Thus, $|\Var_\pi(A)|\le\frac2S(\tfrac12+1+\tfrac34)V^2=\frac{9V^2}{2S}=O(V^2/S)$, and as $S^3m^4=n^4/S$,
\begin{align*}
    &\Var_\pi(M_2)=\frac{4}{S^2m^4}O(V^2/S)=O\Bigl(\frac{V^2 S}{n^4}\Bigr)=O\Bigl(\frac{\hat v_a^2 S}{n^2}\Bigr), \\
&\Var_\pi(v_\sub^a)=\Bigl(\tfrac{n}{S\gamma+1}\Bigr)^2 O\Bigl(\frac{\hat v_a^2 S}{n^2}\Bigr)
=O\Bigl(\frac{\hat v_a^2 S}{(S\gamma+1)^2}\Bigr).
\end{align*}
Since $\bar v=\frac{n}{S\gamma+1}\frac{S-1}{n-1}\hat v_a\ge\frac{S-1}{S\gamma+1}\hat v_a$, the
fluctuation obeys
\begin{align*}
\frac{(\Var_\pi v_\sub^a)^{1/2}}{\sqrt{\bar v}}
=O\Bigl(\frac{\hat v_a\sqrt S}{S\gamma+1}\Bigr)\Big/
\sqrt{\frac{(S-1)\hat v_a}{S\gamma+1}}
\;\lesssim\;\sqrt{\hat v_a}\,(S\gamma+1)^{-1/2},
\end{align*}
using $S/(S-1)\le2$ for $S\ge2$. Adding the bias completes the proof.
\end{proof}

\begin{lemma}
\label{lem:cm-floor}
Set $B_a:=2C/\mu_\infty$, $c_a:=\tfrac12 v_\infty^a$, $\kappa_0:=\frac{K_2}{\mu_\infty}\vee\frac{2C}{\mu_\infty^2}$, $\delta:=\frac{v_\infty^a}{8B_a\kappa_0}$, and let $\xi_i^{a,\infty}:=(H_\infty^{-1}a)^\top\nabla l(X_i,\theta_\infty^\star)$, $D_n:=\|\theta_0^\star-\theta_\infty^\star\|+\|H_{0,n}-H_\infty\|_\op$. Define the $\bX_{1:n}$-measurable event
\begin{align}\label{eqn:thm-floorset}
\mathcal G_n':=\mathcal G_n\cap\{\tfrac1n\sum_i(\xi_i^{a,\infty})^2\ge\tfrac34 v_\infty^a\}\cap\{D_n\le\delta\}.
\end{align}
Under (A1-5) and $v_\infty^a>0$, $\P(\mathcal G_n'^c)=O(n^{-1})$, and on $\mathcal G_n'$ the contrast scores are bounded and the empirical variance is floored, i.e., $|\xi_i^a|,|\bar\xi_s^a|\le B_a$ and $c_a\le\hat v_a\le B_a^2$.
\end{lemma}
\begin{proof}
The score bounds hold throughout $\mathcal G_n$: $\|b_a\|=\|H_{0,n}^{-1}a\|\le2/\mu_\infty$ gives $|\xi_i^a|\le\|b_a\|C\le B_a$ and $|\bar\xi_s^a|\le B_a$, whence $\hat v_a\le B_a^2$. For the floor, decompose $\hat v_a=\frac1n\sum_i(\xi_i^{a,\infty})^2+e_n$ with $e_n:=\hat v_a-\frac1n\sum_i(\xi_i^{a,\infty})^2$. The $\xi_i^{a,\infty}$ are i.i.d.\ with $\E(\xi_i^{a,\infty})^2=v_\infty^a$ and $|\xi_i^{a,\infty}|\le C/\mu_\infty\le B_a$. On $\mathcal G_n'$ the leading term is $\ge\frac34 v_\infty^a$ by \eqref{eqn:thm-floorset}. For the plug-in error, $x\mapsto x^2$ is $2B_a$-Lipschitz on $[-B_a,B_a]$, so $|e_n|\le2B_a\kappa_0 D_n$, using $\|\psi_i-\psi_i^\infty\|\le K_2\|\theta_0^\star-\theta_\infty^\star\|$ and $\|b_a-H_\infty^{-1}a\|\le\mu_\infty^{-2}\|H_{0,n}-H_\infty\|_\op$. On $\mathcal G_n'$, $D_n\le\delta$, so $|e_n|\le2B_a\kappa_0\delta=\frac14 v_\infty^a$ by the choice of $\delta$. Hence $\hat v_a\ge\frac34 v_\infty^a-\frac14 v_\infty^a=c_a$. For the probability, each of the three events in \eqref{eqn:thm-floorset} fails with probability $O(n^{-1})$. The event $\mathcal G_n^c$ has $\P(\mathcal G_n^c)=O(e^{-c'n})$ by Lemma~\ref{lem:cm-reg}(i). The mean-square event fails with probability $O(e^{-cn})$ by Hoeffding's inequality applied to the i.i.d.\ bounded $(\xi_i^{a,\infty})^2$. Finally $\P(D_n>\delta)\le\P(\|\theta_0^\star-\theta_\infty^\star\|>\delta/2)+\P(\|H_{0,n}-H_\infty\|_\op>\delta/2)=O(n^{-1})$ by Lemma~\ref{lem:cm-reg}(iii) with Markov and by (A3)-Lipschitzness with matrix Bernstein. Summing gives $\P(\mathcal G_n'^c)=O(n^{-1})$.
\end{proof}

\paragraph{The conditional comparison.} Combining the linearization error with Lemma~\ref{lem:cm-be}, conditional on $(\bX_{1:n},\pi)$,
\begin{align}
W_1(\mathcal L_\sub^{a,\bX_{1:n},\pi},\mathcal L_\Diri^{a,\bX_{1:n}})&\lesssim 
\underbrace{\sqrt{v_\sub^a} (\rho_\sub^a +(S\gamma)^{-1/2})}_{\text{(BE-blk)}} \label{eqn:thm-cond} \\
&+\underbrace{|\sqrt{v_\sub^a}-\sqrt{v_\Diri^a}|}_{\text{(gap)}} \label{eqn:thm-cond2}\\
&+\underbrace{\sqrt{v_\Diri^a} (\rho_\Diri^a+n^{-1/2})}_{\text{(BE-dir)}} \label{eqn:thm-cond3}\\
&+ 
\E[\|r_\sub\|\mid \bX_{1:n},\pi]+\E[\|r_\Diri\|\mid \bX_{1:n}]. \label{eqn:thm-cond4}
\end{align}
The three bracketed terms are, in order, the block Berry-Esseen error, the standard-deviation gap between the two Gaussians, and the per-record Berry-Esseen error. The last two expectations are the linearization remainders. 

\paragraph{Taking expectations.} It remains to integrate \eqref{eqn:thm-cond}-\eqref{eqn:thm-cond4} over $(\bX_{1:n},\pi)$. We do so on the floor event $\mathcal G_n'$ of Lemma~\ref{lem:cm-floor}, bounding the four contributions one at a time and handling the complement $\mathcal G_n'^c$ at the very end. 
We work under $S\gamma\ge3$, which is what Lemma~\ref{lem:cm-be} requires.  Note that $S\gamma^\star+1=\frac{(S-1)(n+1)}{n-1}$ gives $S\gamma^\star>3$ for every $S\ge5$ and $n>1$, so the calibrated choice always satisfies it.

\paragraph{Bounding \eqref{eqn:thm-cond3}.} Observe  $\sum_i|\xi_i^a|^3\le B_a\sum_i(\xi_i^a)^2$ and $\sqrt{v_\Diri^a}\le B_a$, whence on $\mathcal G_n'$
\begin{align*}
\sqrt{v_\Diri^a} \rho_\Diri^a=\frac{\sum_i|\xi_i^a|^3}{\sqrt{n+1} \sum_i(\xi_i^a)^2}\le\frac{B_a}{\sqrt{n+1}},\qquad
\sqrt{v_\Diri^a} n^{-1/2}\le B_a n^{-1/2},
\end{align*}
and the per-record Berry-Esseen contribution $\sqrt{v_\Diri^a}(\rho_\Diri^a+n^{-1/2})$ is $\lesssim n^{-1/2}$.

\paragraph{Bounding \eqref{eqn:thm-cond}.} Fix $\bX_{1:n}$ on $\mathcal G_n'$, so the scores $\xi_i^a$ are fixed, centered, and bounded by $B_a$, and view $\bar\xi_1^a=\frac1{n_S}\sum_{i\in I_1}\xi_i^a$ as the mean of a size-$n_S$ sample drawn without replacement from $\{\xi_i^a\}$. By Theorem~4 of \cite{hoef:63} and its MGF argument therein, $\E_\pi e^{\lambda\bar\xi_1^a}\le e^{\lambda^2\tau^2/2}$ with $\tau:=B_a/\sqrt{n_S}=B_a\sqrt{S/n}$, so $\bar\xi_1^a$ is sub-Gaussian in $\pi$ with $\|\bar\xi_1^a\|_{\psi_2}\le c \tau$ (\citealp[Prop.~2.6.1]{vers:18}) for some constant $c>0$. 

Recall $M_2=\frac{1}{S}\sum_{s=1}^S (\bar \xi_s^a)^2$ with $v_\sub^a=\frac{n}{S\gamma+1}M_2$. Define $\sigma_g^2:=\E_\pi M_2=\frac{S-1}{n-1}\hat v_a$ (Lemma~\ref{lem:cm-diri}), and on $\mathcal G_n'$ the floor $c_a\le\hat v_a\le B_a^2$ (Lemma~\ref{lem:cm-floor}) together with $\frac{S}{2n}\le\frac{S-1}{n-1}\le\frac Sn$ implies the bounds
\begin{align}\label{eqn:thm-sigmag}
\tfrac{c_a S}{2n}\le\sigma_g^2\le\tfrac{B_a^2 S}{n}=\tau^2\le\tfrac{2B_a^2}{c_a} \sigma_g^2.
\end{align}
Writing $M_k:=\frac1S\sum_s|\bar\xi_s^a|^k$, the term to control is $\sqrt{v_\sub^a}\rho_\sub^a=\sqrt{\frac{n}{(S\gamma+1)S}} \frac{M_3}{M_2}$. We split on the typical event $\mathcal T:=\{M_2\ge\frac12\sigma_g^2\}$.

On $\mathcal T$ we have $\frac{M_3}{M_2}\le\frac{2M_3}{\sigma_g^2}$, and exchangeability with the sub-Gaussian third-moment bound (\citealp[Prop.~2.6.1]{vers:18}) and \eqref{eqn:thm-sigmag} give $\E_\pi[\frac{M_3}{M_2}\mathbf 1_\mathcal T]\le\frac2{\sigma_g^2}\E_\pi|\bar\xi_1^a|^3\lesssim_a\sigma_g$. On the rare event $\mathcal T^c$ we use $\frac{M_3}{M_2}\le\max_s|\bar\xi_s^a|$ and apply the Cauchy-Schwarz inequality to $\E_\pi[\frac{M_3}{M_2}\mathbf 1_{\mathcal T^c}]$. First, we observe the maximal inequality $\E_\pi\max_s(\bar\xi_s^a)^2\lesssim\sigma_g^2\log S$ (\citealp[Prop.~2.7.6]{vers:18}) and the Chebyshev bound $\P_\pi(\mathcal T^c)\le4\Var_\pi(M_2)/\sigma_g^4=O(1/S)$ (Lemma~\ref{lem:cm-sdgap}). These lead to the bound $\E_\pi[\frac{M_3}{M_2}\mathbf 1_{\mathcal T^c}]\lesssim_a\sigma_g\sqrt{(\log S)/S}$. 
Adding the two, $\E_\pi[\frac{M_3}{M_2}\mathbf 1_{\mathcal G_n'}]\lesssim_a\sigma_g\le\tau=B_a\sqrt{S/n}$, so that
\begin{align*}
\E_\pi[\sqrt{v_\sub^a} \rho_\sub^a \mathbf 1_{\mathcal G_n'}] \lesssim_a\sqrt{\frac{n}{(S\gamma+1)S}}\cdot B_a\sqrt{S/n} =B_a\,(S\gamma+1)^{-1/2}.
\end{align*}
For the self-normalizer part, by Jensen and \eqref{eqn:thm-varmatch},
$\E_\pi\sqrt{v_\sub^a}\le\sqrt{\E_\pi v_\sub^a}=\sqrt{R_{S,n}(\gamma)}\sqrt{v_\Diri^a}$, whence
\begin{align*}
\E_\pi\bigl[\sqrt{v_\sub^a}\,(S\gamma)^{-1/2}\bigr]
\le B_a\sqrt{\frac{R_{S,n}(\gamma)}{S\gamma}}
=B_a\sqrt{\frac{(S-1)(n+1)}{(n-1)\,S\gamma}}\,(S\gamma+1)^{-1/2}
\le\frac{2B_a}{\sqrt{\underline\gamma}}\,(S\gamma+1)^{-1/2}
\end{align*}
for $\gamma\ge\underline\gamma$, using $(S-1)/S\le1$ and $(n+1)/(n-1)\le2$ for $n\ge3$. Combined with the preceding display, the block term is $\lesssim_a(S\gamma+1)^{-1/2}$ in $\pi$-mean.

\paragraph{Bounding \eqref{eqn:thm-cond2}.} This is exactly the content of
Lemma~\ref{lem:cm-sdgap}: with the floor $\hat v_a\ge c_a$, the gap averages to
$\E_\pi[|\sqrt{v_\sub^a}-\sqrt{v_\Diri^a}|\mid\bX_{1:n}]
\lesssim_a\bigl|1-\sqrt{R_{S,n}(\gamma)}\bigr|+(S\gamma+1)^{-1/2}$ on $\mathcal G_n'$.

\paragraph{Bounding \eqref{eqn:thm-cond4}.} By the linearization bound \eqref{eqn:thm-w1remainder} and Lemma~\ref{lem:cm-lin}, together with the partition-mean rate $\E_\pi[\bar\sigma_{\psi,n}^2+\bar\kappa_n^2\mid \bX_{1:n}]=O(S/n)$ of \eqref{eqn:vm-margrate}, the block remainder splits as in Section~\ref{app:proof-prop}. On $\tilde E_n^\sub$ the small-remainder bound gives $\E_{\bX_{1:n},\pi}[\|r_\sub\|\mathbf 1_{\tilde E_n^\sub}\mathbf 1_{\mathcal G_n'}]\le\Phi_\infty\frac{\sqrt n}{S\gamma} O(S/n)=O(n^{-1/2})$; off $\tilde E_n^\sub$ the crude bound $\|r_\sub\|\le\sqrt n D'$ and $\E_{\bX_{1:n}}\P(\tilde E_n^{\sub c}\mid \bX_{1:n},\pi)=O(\frac{S/n}{S\gamma})=O(1/(n\gamma))$ give $\sqrt n D'\cdot O(1/(n\gamma))=O(\underline\gamma^{-1}n^{-1/2})$ for
$\gamma\ge\underline\gamma$. The per-record remainder is $O(n^{-1/2})$ by the same argument, so the two remainders together contribute $O(n^{-1/2})$.

\paragraph{The complement of the floor event.} Finally, both projected laws are supported in $[-\sqrt n D_\Theta,\sqrt n D_\Theta]$, so $W_1\le2\sqrt n D_\Theta$ deterministically. Since $\P(\mathcal G_n'^c)=O(n^{-1})$, the discarded part contributes $\E[W_1\mathbf 1_{\mathcal G_n'^c}]\le2\sqrt n D_\Theta\cdot O(n^{-1})=O(n^{-1/2})$.

Collecting the four bounds gives
\begin{align*}
\E_{\bX_{1:n},\pi}\bigl[W_1(\mathcal L_\sub^{a,\bX_{1:n},\pi},\mathcal L_\Diri^{a,\bX_{1:n}})\bigr]
\lesssim (S\gamma+1)^{-1/2}+n^{-1/2}+\bigl|1-\sqrt{R_{S,n}(\gamma)}\bigr|,
\end{align*}
which at $\gamma=\gamma^\star$ reduces to $S^{-1/2}+n^{-1/2}$. \qed

\newpage 
\subsection{Proposition~\ref{prop:formal-dp}}
\label{sec:proof-formal-dp}

The proof has three steps. First, we bound the one-iteration R\'enyi divergence between $\cM_t(D)$ and $\cM_t(D')$, conditional on a realization $\bar w_n=w$ of the $n$th instance. This follows the replace-one analysis of Theorem C.9 in \cite{birr:etal:24}. Secondly, we then integrate the $w$-dependent one-iteration bound over the law of $\bar w_n$. Throughout, we work with the scaled R\'enyi divergence
\begin{align*}
    A_{\cM}(\lambda) := (\lambda-1) D_\lambda(P_{\cM(D)} \| P_{\cM(D')})
     = \log \E_{o_{1:T}\sim P_{\cM(D')}} [(P_{\cM(D)}(o_{1:T})/P_{\cM(D')}(o_{1:T}))^{\lambda}]
\end{align*}
at integer order $\lambda\ge 2$. Letting $A_{\cM_t}(\lambda):=\sup_{D,D',o_{<t}}A_{\cM_t|<t}(\lambda)$ where $A_{\cM_t|<t}(\lambda)$ is the conditional logarithmic moment given the prior outputs $o_{<t}$ for $K=1$ and, applying the standard additive argument \citep{abad:etal:16},
\begin{align}\label{eqn:bir-chain}
    A_{\cM}(\lambda) \le \sum_{t=1}^{T} A_{\cM_{t}}(\lambda)\leq T \sup_t A_{\cM_t}(\lambda).
\end{align}
If we run $K$ independent Monte Carlo replicates of an identical release,  the bound becomes $A_{\cM}(\lambda)\le TK\sup_t A_{\cM_t}(\lambda)$. Hence, it suffices to characterize $A_{\cM_{t}}$.

Note the bounded-gradient assumption makes the sup over $o_{<t}$ well-defined. Also, the Dirichlet allocation has nothing to do with the data. Throughout the privacy analysis we work with the general Dirichlet $(\bu,\tilde\bu)\sim\Diri(\gamma{\bf 1}_S,\alpha{\bf 1}_{\tilde S}/\tilde S)$ used in the posterior bootstrap setup (SM~\ref{app:pseudo-sample-alpha}). The marginal of the differing record is $\bar w_n=u_{s(n)}/n_S\sim\Beta(\gamma,(S-1)\gamma+\alpha)/n_S$. We also set $q$ is fixed.


\subsubsection{Step 1: one-iteration R\'enyi divergence conditional on $\bar w_n=w$}
\label{sec:bir-step1}

This part follows \citet[Appendix~C.1, Eqs.~(54)-(57), (89)-(90)]{birr:etal:24}. For the replace-one scheme, we assume $x_n\in D$ and $x_n'\in D'$ are differing records. Define $T_{\mathrm{com}}=B_t\cap\{1,\dots,n-1\}$ of the common indices where $B_t$ is the set of indices selected from Poisson sampling. The differing records come into play for privacy analysis by the Bernoulli inclusion indicator $B_n\sim\mathrm{Ber}(q)$ of the differing record. The per-step mechanism can be expressed by
\begin{align*}
    \cM_t(D)
    =\sum_{i\in T_{\mathrm{com}}}\bar w_i l'(x_i,g_t)
    +B_n \bar w_n l'(x_n,g_t)+Z,
\end{align*}
where $\mu_{\mathrm{com}}(T_{\mathrm{com}},\bw_{-n}):=\sum_{i\in T_{\mathrm{com}}}\bar w_i l'(x_i,g_t)$. Note $g_t$ is a function of the previous iterate. Centering the means of the conditional Gaussian releases by the common-sample shift $\mu_{\mathrm{com}}(T_{\mathrm{com}},\bar \bw)$, the conditional one-iteration R\'enyi divergence at order $\lambda$ is upper bounded by
\begin{align}\label{eqn:bir-Eq56-ours}
    D_\lambda^{|<t,\bar w=w}:=&D_\lambda (P_{\cM_t(D)|\bar w=w}\|P_{\cM_t(D')|\bar w=w})
    \\
    &\le 
    D_\lambda (q N_{\Delta,\sigma^{2}I_d}+(1-q) N_{0,\sigma^{2}I_d}\|
                     q N_{\Delta',\sigma^{2}I_d}+(1-q) N_{0,\sigma^{2}I_d}),
\end{align}
where
\begin{align*}
    \Delta := \bar w_n l'(x_n,g) = w l'(x_n,g),\qquad
    \Delta' := \bar w_n l'(x_n',g) = w l'(x_n',g),
\end{align*}
are the differing-record contributions on the event $\{B_n=1\}$ under $D$ and $D'$ respectively. Under the bounded gradient norm assumption, it follows that
\begin{align}\label{eqn:bir-Eq57-ours}
    \|\Delta\| \le wC,\qquad
    \|\Delta'\| \le wC,\qquad
    \|\Delta-\Delta'\| \le 2wC.
\end{align}
Birrell's Theorem~C.9 provides a non-asymptotic Taylor expansion in $q$, with explicit closed-form remainder, of the right-hand side of \eqref{eqn:bir-Eq56-ours} subject to \eqref{eqn:bir-Eq57-ours}. For any integer $m\ge 3$ and integer $\lambda\ge 2$, the conditional divergence $D_\lambda^{|<t,\bar w=w}$ in \eqref{eqn:bir-Eq56-ours} is bounded above
\begin{align}\label{eqn:bir-c9-ours}
    \exp [(\lambda-1) D_\lambda^{|<t,\bar w=w}]
     \le 
    1 + q^{2} \Phi_\lambda(w)
     + \sum_{k=3}^{m-1}\frac{q^{k}}{k!} \tilde F_{\lambda,k}^{\mathrm{Pois,ro}}(w)
     + \tilde E_{\lambda,m}^{\mathrm{Pois,ro}}(q,w),
\end{align}
where
\begin{align*}
    \Phi_\lambda(w)
     := \lambda(\lambda-1)(e^{w^{2}C^{2}/\sigma^{2}}-e^{-w^{2}C^{2}/\sigma^{2}})
     = 2\lambda(\lambda-1) \sinh \Bigl(\frac{w^{2}C^{2}}{\sigma^{2}}\Bigr).
\end{align*}

\paragraph{Explicit forms of $\tilde F$ and $\tilde E$.}
For convenience, we introduce the explicit Birrell-Theorem-C.9 forms with the Dirichlet weight. Let
\begin{align*}
    \tau_w := \frac{w^{2}C^{2}}{2\sigma^{2}}
\end{align*}
Define the base Gaussian-moment quantity
\begin{align}\label{eqn:bir-Mk}
    M_k(w)
     := \sum_{\ell=2}^{k}(-1)^{k-\ell}\binom{k}{\ell} e^{\ell(\ell-1)\tau_w}
     + (-1)^{k-1}(k-1),
    \qquad k\ge 2,
\end{align}
with $M_0(w)=M_1(w)=0$, and
\begin{align*}
    \tilde B_k(w)
     := \begin{cases}
        M_k(w), & k\text{ even},\\
        \sqrt{M_{k-1}(w) M_{k+1}(w)}, & k\text{ odd}.
    \end{cases}
\end{align*}
With these, for $2\le k\le m-1$,
\begin{align}\label{eqn:bir-Ftilde}
    \tilde F_{\lambda,k}^{\mathrm{Pois,ro}}(w)
    & = (\lambda-1)\lambda^{k-1}[ 4 M_k(w) \mathbf{1}\{k\text{ even}\}
         + 3 \sqrt{M_{k-1}(w) M_{k+1}(w)} \mathbf{1}\{k\text{ odd}\} ]\nonumber\\
    &\quad+ (\lambda-1)\lambda^{k-1} \tilde B_k(w) \Lambda_k(\lambda),\\[2pt]
\Lambda_k(\lambda)
    & := \sum_{j=0}^{k}\binom{k}{j} \left| \frac{\lambda}{\lambda-1} 
        \prod_{\ell=0}^{j-1} (1-\tfrac{\ell}{\lambda}) 
        \prod_{\ell=0}^{k-j-1} (1+\tfrac{\ell-1}{\lambda}) - 1 \right|, \notag
\end{align}
and the integral remainder reads
\begin{align}\label{eqn:bir-Etilde}
    \tilde E_{\lambda,m}^{\mathrm{Pois,ro}}(q,w)
     = \frac{q^{m}}{m!}\sum_{j=0}^{m}\binom{m}{j} (1-q)^{-(\lambda+m-j-1)} 
        \prod_{\ell=0}^{j-1}|\lambda-\ell|\prod_{\ell=0}^{m-j-1}(\lambda+\ell-1) 
        \Theta_{\lambda,m,j}(q,w),
\end{align}
where the inner remainder factor splits according to the sign of $\lambda-j$:
\begin{align}\label{eqn:bir-Theta}
    \Theta_{\lambda,m,j}(q,w)
     := \begin{cases}
        (1-q)^{\lambda-j} \tilde B_m(w), & j\ge\lambda,\\[4pt]
        \displaystyle\tilde B_m(w)+\sum_{\ell=0}^{\lambda-j}q^{\ell} 
            \frac{(\lambda-j)! m!}{(\lambda-j-\ell)! (m+\ell)!} \tilde B_{m+\ell}(w),
            & j<\lambda.
    \end{cases}
\end{align}
The vanishing factor $|\lambda-\ell|$ at $\ell=\lambda$ truncates the effective range of $j$ in \eqref{eqn:bir-Etilde} to $\{0,1,\dots,\lambda\}$. Empty products equal $1$ by convention. Note $w$ is mainly activated via $M_k(w)$. In \eqref{eqn:bir-c9-ours}, the $k=2$ Taylor contribution is written separately as the leading term $q^2\Phi_\lambda(w)$, so the displayed Taylor sum starts at $k=3$ even though $\tilde F_{\lambda,k}^{\mathrm{Pois,ro}}$ is defined from $k=2$ onward. This is the same bookkeeping as \citet[Appendix~C.1, Eqs.~(92)-(95)]{birr:etal:24}: the raw Taylor expansion includes $k=2,\ldots,m-1$, the $k=2$ coefficient is evaluated separately, and Theorem~C.9 displays the remainder sum from $k=3$.

\subsubsection{Step 2: handling the joint randomness of $\bar w$}
\label{sec:bir-step2}

The bound \eqref{eqn:bir-c9-ours} controls the conditional R\'enyi divergence at fixed $\bar w_n=w$. We now integrate it over the law of $\bar w_n$ under $\Diri(\gamma{\bf 1}_S,\alpha{\bf 1}_{\tilde S}/\tilde S)/n_S$ and its marginal
$\Beta(\gamma,(S-1)\gamma+\alpha)/n_S$. 

\paragraph{Joint-convexity Jensen step.}
The Hessian of $f_\lambda(x,y)=x^{\lambda}y^{1-\lambda}$ on $(0,\infty)^{2}$ has determinant zero and positive trace, so it is positive semidefinite. Pointwise Jensen and Fubini \citep[Theorem~13]{van:harr:14} give, for any kernels $w\mapsto p_w,q_w$ and any law $\mu$ on $w$,
\begin{align*}
    e^{(\lambda-1) D_\lambda(\E_w p_w \| \E_w q_w)}
     \le \E_{w\sim\mu} [ e^{(\lambda-1) D_\lambda(p_w \| q_w)} ].
\end{align*}
Since $P_{\cM_t}(o)=\int P_{\cM_t|\bar \bw}(o)dP_{\bar\bw}$, the marginal logarithmic R\'enyi moment of one release satisfies
\begin{align}\label{eqn:bir-marginal}
    \exp\{A_{\cM_t}(\lambda)\}
    &\le\E_{\bar w} [\exp \{(\lambda-1)D_\lambda^{|<t,\bar w}\}]
     \notag \\ 
    &\le 1+q^{2}\E_{\bar w} [\Phi_\lambda(\bar w)]
    +\E_{\bar w} \Bigl[\sum_{k=3}^{m-1}\frac{q^{k}}{k!} \tilde F_{\lambda,k}^{\mathrm{Pois,ro}}(\bar w)
    +\tilde E_{\lambda,m}^{\mathrm{Pois,ro}}(q,\bar w)\Bigr].
\end{align}

\paragraph{Beta-moment Taylor building block.}
Every $w$-dependence in \eqref{eqn:bir-marginal} reduces to expectations of the form
\begin{align}\label{eqn:bir-g}
    g_\ell := \E_{\bar w} [\exp(\beta_\ell \bar w^{2})],
    \qquad \beta_\ell := \frac{\ell(\ell-1) C^{2}}{2\sigma^{2}},
    \qquad \ell\ge 2,
\end{align}
either directly (through $\Phi_\lambda$ and the $M_k$'s in $\tilde F, \tilde E$) or after Cauchy-Schwarz for odd $k$. Expanding the exponential, applying termwise the Beta-moment formula $\E[\bar w^{2j}]=(\gamma)_{2j}/((S\gamma+\alpha)_{2j}\,n_S^{2j})$ (dominated convergence is satisfied because $\bar w\le 1/n_S$), and truncating at level $J$,
\begin{align*}
    {\cal T}_{J,\ell} := \sum_{j=0}^{J}\frac{\beta_\ell^{j}}{j!} 
        \frac{(\gamma)_{2j}}{(S\gamma+\alpha)_{2j} n_S^{2j}}
     \le  g_\ell  \le {\cal T}_{J,\ell} + \mathcal{R}_{J,\ell},
    \quad
    \mathcal{R}_{J,\ell} := e^{\beta_\ell/n_S^{2}}-\sum_{j=0}^{J}\frac{(\beta_\ell/n_S^{2})^{j}}{j!},
\end{align*}
where the lower bound is the truncated series (every term non-negative) and the upper bound uses $\bar w\le 1/n_S$ to dominate the tail. The two quantities ${\cal T}_{J,\ell}$ and $\mathcal{R}_{J,\ell}$ are the only pieces of arithmetic that depend on the ratio $n_S=n/S$ between the sample size and the number of blocks. Everything else in Step~2 is purely combinatorial.

\paragraph{Dirichlet integration of the leading term.}
For $\Phi_\lambda(w)=2\lambda(\lambda-1)\sinh(w^{2}C^{2}/\sigma^{2})$, expanding $\sinh$ as $\sum_{j\ge 0}x^{2j+1}/(2j+1)!$ at $x=\bar w^{2}C^{2}/\sigma^{2}$ and applying the Beta-moment formula gives the closed-form series
\begin{align*}
    \E_{\bar w} [\Phi_\lambda(\bar w)]
     = 2\lambda(\lambda-1)\sum_{j=0}^{\infty}\frac{(C/\sigma)^{4j+2}}{(2j+1)!}\cdot\frac{(\gamma)_{4j+2}}{(S\gamma+\alpha)_{4j+2} n_S^{4j+2}}.
\end{align*}
Truncating at level $J$ and using $\bar w\le 1/n_S$ on the tail,
\begin{align}\label{eqn:bir-MR-def}
    M_{J,\lambda}
    &:= 2\lambda(\lambda-1)\sum_{j=0}^{J}\frac{(C/\sigma)^{4j+2}}{(2j+1)!}\cdot\frac{(\gamma)_{4j+2}}{(S\gamma+\alpha)_{4j+2} n_S^{4j+2}},\nonumber\\
    R_{J,\lambda}
    &:= 2\lambda(\lambda-1)\Bigl[\sinh (C^{2}/(\sigma^{2}n_S^{2}))-\sum_{j=0}^{J}\frac{(C/(\sigma n_S))^{4j+2}}{(2j+1)!}\Bigr],
\end{align}
so that $\E_{\bar w}[\Phi_\lambda(\bar w)]\le M_{J,\lambda}+R_{J,\lambda}$.

\paragraph{Dirichlet integration of the higher-order terms.}
For $M_k(\bar w)$ in \eqref{eqn:bir-Mk}, $\E[M_k(\bar w)]=\sum_{\ell=2}^{k}(-1)^{k-\ell}\binom{k}{\ell}g_\ell+(-1)^{k-1}(k-1)$; upper-bounding $g_\ell$ on the positive-sign terms ($\ell\equiv k  \mod 2$) and lower-bounding on the negative-sign terms via the building-block inequality of \eqref{eqn:bir-g} gives
\begin{align}\label{eqn:bir-Mbar}
    \E[M_k(\bar w)]
     \le \overline{M}_{J,k}^{\mathrm{Dir}}
     := \sum_{\ell=2}^{k}(-1)^{k-\ell}\binom{k}{\ell} {\cal T}_{J,\ell}
     + \sum_{\substack{\ell=2\\ \ell\equiv k (\mathrm{mod} 2)}}^{k}\binom{k}{\ell} \mathcal{R}_{J,\ell}
     + (-1)^{k-1}(k-1).
\end{align}
For the odd-$k$ envelope, Cauchy-Schwarz on $\sqrt{M_{k-1}}\sqrt{M_{k+1}}$ yields
\begin{align}\label{eqn:bir-Bbar}
    \E[\tilde B_k(\bar w)]
     \le \overline{B}_{J,k}^{\mathrm{Dir}}
     := \begin{cases}
        \overline{M}_{J,k}^{\mathrm{Dir}}, & k\text{ even},\\
        \sqrt{\overline{M}_{J,k-1}^{\mathrm{Dir}} \overline{M}_{J,k+1}^{\mathrm{Dir}}}, & k\text{ odd}.
    \end{cases}
\end{align}
Substituting \eqref{eqn:bir-Mbar} and \eqref{eqn:bir-Bbar} into the explicit forms \eqref{eqn:bir-Ftilde}-\eqref{eqn:bir-Theta} of $\tilde F$ and $\tilde E$ (every $w$-dependence enters through either $M_k(w)$ or $\tilde B_k(w)$),
\begin{align*}
    &\E[\tilde F_{\lambda,k}^{\mathrm{Pois,ro}}(\bar w)]
     \le \overline{F}_{J,\lambda,k}^{\mathrm{Dir}}
     \\ 
    &:= (\lambda-1)\lambda^{k-1}[
        4 \overline{M}_{J,k}^{\mathrm{Dir}} \mathbf{1}\{k\text{ even}\}
        +3 \sqrt{\overline{M}_{J,k-1}^{\mathrm{Dir}} \overline{M}_{J,k+1}^{\mathrm{Dir}}} \mathbf{1}\{k\text{ odd}\}]
        +(\lambda-1)\lambda^{k-1} \overline{B}_{J,k}^{\mathrm{Dir}} \Lambda_k(\lambda),
\end{align*}
and, with $\overline{\Theta}_{J,\lambda,m,j}(q)$ obtained from \eqref{eqn:bir-Theta} by the substitution $\tilde B_p(w)\to\overline{B}_{J,p}^{\mathrm{Dir}}$,
\begin{align*}
    &\E[\tilde E_{\lambda,m}^{\mathrm{Pois,ro}}(q,\bar w)]
     \le \overline{E}_{J,\lambda,m}^{\mathrm{Dir}}(q)
     \\ 
    &:= \frac{q^{m}}{m!}\sum_{j=0}^{m}\binom{m}{j} (1-q)^{-(\lambda+m-j-1)}
        \prod_{\ell=0}^{j-1}|\lambda-\ell|\prod_{\ell=0}^{m-j-1}(\lambda+\ell-1) 
        \overline{\Theta}_{J,\lambda,m,j}(q).
\end{align*}
Bundling the two contributions into the non-asymptotic Dirichlet-integrated upper bound denoted by $\mathcal R_{J,\lambda,m}$ in Proposition~\ref{prop:formal-dp},
\begin{align}\label{eqn:bir-Rcal-bar}
    \mathcal{R}_{J,\lambda,m}(q)
     := \sum_{k=3}^{m-1}\frac{q^{k}}{k!} \overline{F}_{J,\lambda,k}^{\mathrm{Dir}}
     + \overline{E}_{J,\lambda,m}^{\mathrm{Dir}}(q),
\end{align}
which by construction satisfies \( \E_{\bar w} [\sum_{k=3}^{m-1} \frac{q^{k}}{k!}\tilde F_{\lambda,k}^{\mathrm{Pois,ro}}(\bar w)+\tilde E_{\lambda,m}^{\mathrm{Pois,ro}}(q,\bar w)] \;\le\;\mathcal{R}_{J,\lambda,m}(q). \) Substituting \eqref{eqn:bir-MR-def} and \eqref{eqn:bir-Rcal-bar} into \eqref{eqn:bir-marginal} gives the per-release accountant
\begin{align}\label{eqn:bir-Aprim}
    A_{\cM_t}(\lambda)
    \leq \log [ 1 + q^{2}(M_{J,\lambda}+R_{J,\lambda}) + \mathcal{R}_{J,\lambda,m}(q) ]
     =: A_{J,\lambda,m}(q),
\end{align}
uniformly in $t$ and in the prior outputs $o_{<t}$. 

\subsubsection{Step 3: composition over $TK$ releases and conversion to $(\epsilon,\delta)$-DP}
\label{sec:bir-step3}

By the chain rule \eqref{eqn:bir-chain} together with the $K$-fold identical Monte Carlo replication of each iteration,
\begin{align}\label{eqn:bir-total-A}
    A_{\cM}(\lambda)\leq TK A_{J,\lambda,m}(q),\qquad
    \gamma_M^{\mathrm{tot}}(\lambda) := \frac{A_{\cM}(\lambda)}{\lambda-1} \le \frac{TK A_{J,\lambda,m}(q)}{\lambda-1},
\end{align}
so $\cM$ is $(\lambda,\gamma_M^{\mathrm{tot}}(\lambda))$-RDP under replace-one adjacency for every integer $\lambda\ge 2$. Converting this RDP guarantee to $(\epsilon,\delta)$-DP by the bound recalled in Section~\ref{app:conversion-to-dp} shows that $\cM$ is $(\epsilon,\delta_\star(\epsilon))$-DP with the transparent budget
\begin{align*}
    \delta_\star(\epsilon) = \inf_{\lambda\ge 2}\exp\{ TK A_{J,\lambda,m}(q)-(\lambda-1)\epsilon \},
\end{align*}
which is the one in Proposition~\ref{prop:formal-dp}. The optimal conversion \eqref{eqn:asoo-conversion} of \citet{asoo:etal:21}, stated in Section~\ref{app:conversion-to-dp}, only sharpens it.

\newpage

\subsection{Theorem~\ref{thm:asymp-dp}}
\label{sec:bir-step4}

Refer to the setup in the proof of Proposition~\ref{sec:proof-formal-dp}. Note $(\bu,\tilde\bu)\sim\Diri(\gamma\mathbf 1_S,\alpha{\bf 1}_{\tilde S}/\tilde S)$, i.e., $u_{s(n)}\sim\Beta(\gamma,(S-1)\gamma+\alpha)$ and $\bar w_n=u_{s(n)}/n_S$, with Beta moments
\begin{align}\label{eqn:step4-Beta-moments}
    \E[\bar w_n^{ 2j}]
     = \frac{(\gamma)_{2j}}{(S\gamma+\alpha)_{2j} n_S^{2j}},
    \qquad
    \E[\bar w_n^{ 2}]
     = \frac{\gamma(\gamma+1)}{(S\gamma+\alpha)(S\gamma+\alpha+1) n_S^{2}}.
\end{align}

\paragraph{Linearization of the per-release accountant.}
We expand the per-release bound \eqref{eqn:bir-c9-ours} in the regime $C/(\sigma n_S)\ll 1$. Three groups of terms appear there: the leading $q^{2}\Phi_\lambda(\bar w)$, the Taylor sum $\sum_{k=3}^{m-1}\tfrac{q^{k}}{k!}\tilde F_{\lambda,k}^{\mathrm{Pois,ro}}(\bar w)$, and the integral remainder $\tilde E_{\lambda,m}^{\mathrm{Pois,ro}}(q,\bar w)$.

Using the Taylor series $\sinh(z)=\sum_{j=0}^{\infty}z^{2j+1}/(2j+1)!$ at $z=\bar w^{2}C^{2}/\sigma^{2}$ and DCT with  $\bar w \le 1/n_S$ on the truncated Beta marginal of \eqref{eqn:step4-Beta-moments},
\begin{align*}
    \E_{\bar w} [\Phi_\lambda(\bar w)]
     = 2\lambda(\lambda-1)\sum_{j=0}^{\infty}\frac{(C/\sigma)^{4j+2}}{(2j+1)!} \E[\bar w^{4j+2}].
\end{align*}
Splitting off the $j=0$ slot, the relative correction is
\begin{align}\label{eqn:step4-Phi-relcorr}
    R_\Phi(\sigma,n_S;C)
     := \sum_{j=1}^{\infty}\frac{(C/\sigma)^{4j}}{(2j+1)!}\cdot\frac{\E[\bar w^{4j+2}]}{\E[\bar w^{2}]},
\end{align}
so that $\E_{\bar w}[\Phi_\lambda(\bar w)]=2\lambda(\lambda-1)(C/\sigma)^{2}\E[\bar w^{2}](1+R_\Phi(\sigma,n_S;C))$. The fixed Dirichlet-parameter dependence of $R_\Phi$ through the law of $\bar w$ is suppressed in this notation. The deterministic bound $\bar w\le 1/n_S$ gives the moment-ratio inequality
\begin{align*}
    \frac{\E[\bar w^{4j+2}]}{\E[\bar w^{2}]}
     = \frac{\E[\bar w^{2}\cdot\bar w^{4j}]}{\E[\bar w^{2}]}
    \le \frac{1}{n_S^{4j}}, \qquad j\ge 1,
\end{align*}
and substituting into \eqref{eqn:step4-Phi-relcorr},
\begin{align}\label{eqn:step4-Phi-RsharpUB}
    R_\Phi(\sigma,n_S;C)
     \le \sum_{j=1}^{\infty}\frac{1}{(2j+1)!} \Bigl(\frac{C}{\sigma n_S}\Bigr)^{4j}.
\end{align}
The geometric-tail sum on the right of \eqref{eqn:step4-Phi-RsharpUB} admits the closed form
\begin{align*}
    \sum_{j=1}^{\infty}\frac{1}{(2j+1)!}\Bigl(\frac{C}{\sigma n_S}\Bigr)^{4j}
     = \sum_{j=1}^{\infty}\frac{\eta^{2j}}{(2j+1)!}
     = \frac{\sinh(\eta)}{\eta}-1
    \quad\text{with}\quad
    \eta := \frac{C^{2}}{\sigma^{2}n_S^{2}},
\end{align*}
obtained from $\sinh(z)/z=\sum_{k\ge 0}z^{2k}/(2k+1)!$ at $z=\eta$ after subtracting the $k=0$ slot. Substituting into the definition of $R_\Phi$,
\begin{align}\label{eqn:step4-Phi-leading}
\begin{aligned}
    \E_{\bar w} [\Phi_\lambda(\bar w)]
    & = 2\lambda(\lambda-1) \frac{C^{2}}{\sigma^{2}} \E[\bar w_n^{2}](1+R_\Phi(\sigma,n_S;C))\\
    & = \frac{2\lambda(\lambda-1)\gamma(\gamma+1) C^{2}}{\sigma^{2} (S\gamma+\alpha)(S\gamma+\alpha+1) n_S^{2}}(1+R_\Phi(\sigma,n_S;C)),
\end{aligned}
\end{align}
with the explicit pointwise residual
\begin{align*}
    0 \le R_\Phi(\sigma,n_S;C) \le \frac{\sinh(\eta)}{\eta}-1
     \le \tfrac{1}{6} \eta^{2} e^{\eta}.
\end{align*}
The last inequality follows from
\begin{align*}
    \frac{\sinh(\eta)}{\eta}-1
    =\eta^{2}\sum_{k=0}^{\infty}\frac{\eta^{2k}}{(2k+3)!}
    \le \frac{\eta^{2}}{6}\sum_{k=0}^{\infty}\frac{\eta^{2k}}{(2k)!}
    =\frac{\eta^{2}}{6}\cosh(\eta)
    \le\frac{\eta^{2}}{6}e^\eta,
    \qquad \eta\ge0.
\end{align*}

Thus, the leading-slot residual is controlled by the local expansion parameter itself:
\begin{align*}
    R_\Phi(\sigma,n_S;C)
     = O \left(\Bigl(\frac{C}{\sigma n_S}\Bigr)^4\right)
     = o(1)
    \qquad\text{whenever } C/(\sigma n_S)\to0,
\end{align*}


\paragraph{Bounding the higher Taylor terms and the integral remainder.}
We now show the contributions of the third and fourth terms in \eqref{eqn:bir-c9-ours}.

\paragraph{Cancellation of $M_k(w)$ to second order.}
Recall from \eqref{eqn:bir-Mk} that $M_k(w)=\sum_{\ell=2}^{k}(-1)^{k-\ell}\binom{k}{\ell}e^{\ell(\ell-1)\tau_w}+(-1)^{k-1}(k-1)$ with $\tau_w=w^{2}C^{2}/(2\sigma^{2})$. Since $M_k$ is real-analytic in $\tau_w$, expanding each exponential as $e^{\ell(\ell-1)\tau_w}=\sum_{p\ge 0}[\ell(\ell-1)]^{p}\tau_w^{p}/p!$ and collecting powers gives $M_k(w)=\sum_{p=0}^{\infty}\mathsf c_{k,p}\,\tau_w^{p}$ with $\mathsf c_{k,0}=\sum_{\ell=2}^{k}(-1)^{k-\ell}\binom{k}{\ell}+(-1)^{k-1}(k-1)$ and, for $p\ge 1$,
\begin{align}\label{eqn:step4-Mk-coef}
    \mathsf c_{k,p} := \frac{1}{p!}\sum_{\ell=2}^{k}(-1)^{k-\ell}\binom{k}{\ell}[\ell(\ell-1)]^{p}.
\end{align}
For the two-order cancellation claimed in this paragraph, it is enough to show $\mathsf c_{k,0}=0$ and $\mathsf c_{k,1}=0$ for $k\ge 3$, because then $M_k(w)=O(\tau_w^2)=O\{(wC/\sigma)^4\}$. The sharper envelope used below, $M_k(w)=O\{(wC/\sigma)^{2\lceil k/2\rceil}\}$, follows from the same finite-difference identity applied to higher coefficients.

\paragraph{Constant coefficient ($p=0$).}
The binomial identity $\sum_{\ell=0}^{k}(-1)^{k-\ell}\binom{k}{\ell}=(1-1)^{k}=0$ ($k\ge 1$) split into $\ell\in\{0,1\}$ and $\ell\in\{2,\ldots,k\}$ gives $\sum_{\ell=2}^{k}(-1)^{k-\ell}\binom{k}{\ell}=(-1)^{k}(k-1)$, which combined with $(-1)^{k-1}(k-1)=-(-1)^{k}(k-1)$ yields
\begin{align*}
    \mathsf c_{k,0} = (-1)^{k}(k-1)+(-1)^{k-1}(k-1) = 0,\qquad k\ge 2.
\end{align*}

\paragraph{Linear coefficient ($p=1$).}
The factor $\ell(\ell-1)$ vanishes at $\ell=0$ and $\ell=1$, so the sum in \eqref{eqn:step4-Mk-coef} at $p=1$ is unchanged by extending the range to $\ell\in\{0,\ldots,k\}$:
\begin{align*}
    \mathsf c_{k,1}
     = \sum_{\ell=0}^{k}(-1)^{k-\ell}\binom{k}{\ell}\ell(\ell-1)
     = \sum_{\ell=0}^{k}(-1)^{k-\ell}\binom{k}{\ell}\ell^{2}
     - \sum_{\ell=0}^{k}(-1)^{k-\ell}\binom{k}{\ell}\ell.
\end{align*}
The finite-difference identity $\sum_{\ell=0}^{k}(-1)^{k-\ell}\binom{k}{\ell}\ell^{q}=k!\,S(q,k)$, with $S(q,k)$ the Stirling number of the second kind, then gives
\begin{align*}
    \mathsf c_{k,1}
     = k![S(2,k)-S(1,k)]
     = 0,\qquad k\ge 3,
\end{align*}
since $S(1,k)=0$ for $k\ge 2$ and $S(2,k)=0$ for $k\ge 3$.


\paragraph{Asymptotic order and the unified envelope $\tilde B_k$.}
The Stirling-number argument above gives the sharper cancellation: $\mathsf c_{k,p}=0$ for every $p$ with $2p<k$ (because $[\ell(\ell-1)]^{p}$ is a polynomial of degree $2p$ in $\ell$ whose finite-difference image $\sum_{\ell=0}^{k}(-1)^{k-\ell}\binom{k}{\ell}\ell^{q}=k!\,S(q,k)$ vanishes for $q<k$). Hence all powers below $p_k:=\lceil k/2\rceil$ vanish. Taylor's theorem then gives a uniform local bound over the fixed truncation range: for every fixed $m$ and every $\tau_0<\infty$, there exists $C_{m,M}<\infty$ such that, whenever $0\le\tau_w\le\tau_0$ and $2\le k\le m$,
\begin{align}\label{eqn:step4-Mk-asymp}
    |M_k(w)|
     \le C_{m,M} \tau_w^{p_k}
     \le C_{m,M}\left(\frac{wC}{\sigma}\right)^{2p_k}
     = \begin{cases}
        C_{m,M}(wC/\sigma)^{k},& k\text{ even,}\\[2pt]
        C_{m,M}(wC/\sigma)^{k+1},& k\text{ odd.}
    \end{cases}
\end{align}
In the asymptotic regime used here, $\tau_w\le C^2/(2\sigma^2 n_S^2)=o(1)$ uniformly over $0\le w\le1/n_S$, so \eqref{eqn:step4-Mk-asymp} applies uniformly over any fixed finite range of $k$. In particular, $M_k=O\{(wC/\sigma)^4\}$ already at $k=3,4$. Thus, for each fixed $k$,
\begin{align}\label{eqn:step4-Btilde-uniform}
    \tilde B_k(w) = O_k \left(\Bigl(\frac{wC}{\sigma}\Bigr)^k\right),
    \qquad 0\le w\le1/n_S,
\end{align}
i.e., $\tilde B_k$ vanishes to order $k$ in the local parameter $wC/\sigma$. On the finite Taylor-sum range $2\le k\le m-1$, the constants in \eqref{eqn:step4-Btilde-uniform} are absorbed into constants depending only on $m$.

\paragraph{Ratio bound between higher-$k$ terms and the leading slot.}
We first bound $\tilde F_{\lambda,k}^{\mathrm{Pois,ro}}(w)$ from \eqref{eqn:bir-Ftilde}, which splits into two pieces:
\begin{align*}
    &\underbrace{(\lambda-1)\lambda^{k-1}[4M_k(w)\mathbf 1\{k\text{ even}\}+3\sqrt{M_{k-1}(w) M_{k+1}(w)} \mathbf 1\{k\text{ odd}\}]}_{\text{Term A}}\\
    & + \underbrace{(\lambda-1)\lambda^{k-1}\tilde B_k(w) \Lambda_k(\lambda)}_{\text{Term B}}.
\end{align*}
Term A is bounded directly by the uniform envelope \eqref{eqn:step4-Btilde-uniform}: on the finite range $2\le k\le m-1$, $\tilde B_k(w)\le C_{m,B}(wC/\sigma)^{k}$ covers both $M_k$ for even $k$ and $\sqrt{M_{k-1}M_{k+1}}$ for odd $k$, so Term A $=O\{\lambda^{k}(wC/\sigma)^{k}\}$ uniformly over this finite range.

For Term B we must control the $\Lambda_k(\lambda)$ factor of \eqref{eqn:bir-Ftilde}, written as $\Lambda_k(\lambda)=\sum_{j=0}^{k}\binom{k}{j}\frac{\lambda}{\lambda-1}[P_j(\lambda)-1]$ with $P_j(\lambda):=\prod_{\ell=0}^{j-1}(1-\ell/\lambda)\prod_{\ell=0}^{k-j-1}(1+(\ell-1)/\lambda)$.

We provide a uniform bound of this under the fixed-truncation regime. The function $\tilde F_{\lambda,k}^{\mathrm{Pois,ro}}$ is defined for $2\le k\le m-1$, while only $k=3,\ldots,m-1$ appears in the Taylor sum because the $k=2$ term is the leading $q^2\Phi_\lambda$ term. For every integer $\lambda\ge2$ and $0\le \ell\le k-1$, with $2\le k\le m-1$, the two product factors are bounded by constants depending only on $m$: $|1-\ell/\lambda|\le 1+(m-2)/2$ and $|1+(\ell-1)/\lambda|\le 1+(m-2)/2$. Hence $|P_j(\lambda)|\le C_{m,\Lambda}$ and $|P_j(\lambda)-1|\le C_{m,\Lambda}$ uniformly over $j$, $k$, and $\lambda$. With $\binom{k}{j}$ summing to $2^{k}\le 2^m$ and $\lambda/(\lambda-1)\le 2$,
\begin{align*}
    |\Lambda_k(\lambda)| \le C_{m,\Lambda}\qquad\text{uniformly for integer }\lambda\ge2\text{ and }2\le k\le m-1,
\end{align*}
which is a finite constant at fixed truncation order.

Combining the two terms gives the pointwise uniform bound
\begin{align*}
    |\tilde F_{\lambda,k}^{\mathrm{Pois,ro}}(w)|
     \le C_{m,F} \lambda^k\Bigl(\frac{wC}{\sigma}\Bigr)^k,
    \qquad 0\le w\le 1/n_S,\quad 3\le k\le m-1,
\end{align*}
for a constant $C_{m,F}$ depending only on the fixed truncation order. Therefore the expectation is bounded by integrating this pointwise inequality:
\begin{align*}
    \E_{\bar w} [\tilde F_{\lambda,k}^{\mathrm{Pois,ro}}(\bar w)]
    & \le C_{m,F} \lambda^k(C/\sigma)^k \E[\bar w^k]\nonumber\\
    & \le C_{m,F} \lambda^{k} (C/\sigma)^{2} \E[\bar w^{2}] (C/(\sigma n_S))^{k-2},
    \qquad k\ge 3.
\end{align*}
Comparing to the leading slot $q^{2}\E[\Phi_\lambda(\bar w)]$ from \eqref{eqn:step4-Phi-leading},
\begin{align}\label{eqn:step4-ratio-bound}
    \frac{q^{k} \E[\tilde F_{\lambda,k}^{\mathrm{Pois,ro}}(\bar w)] / k!}
         {q^{2} \E[\Phi_\lambda(\bar w)]}
     = O (\tfrac{1}{k!}(q\lambda\cdot C/(\sigma n_S))^{k-2}),
    \qquad k\ge 3.
\end{align}
The suppression in \eqref{eqn:step4-ratio-bound} is controlled, at the current R\'enyi order, by the single small parameter $\varrho=\varrho(\lambda;\sigma,n_S):=q\lambda\cdot C/(\sigma n_S)$; below we write simply $\varrho$ while $\lambda$ is fixed. Each Taylor term has the explicit form $\varrho^{k-2}/k!$. Closing the Taylor sum,
\begin{align}\label{eqn:step4-ratio-sum}
    \sum_{k=3}^{m-1}\frac{q^{k} \E[\tilde F_{\lambda,k}^{\mathrm{Pois,ro}}(\bar w)]/k!}{q^{2}\E[\Phi_\lambda(\bar w)]}
     \le C_{m,F}\sum_{k=3}^{m-1}\frac{\varrho^{k-2}}{k!}
     \le C_{m,F}\sum_{k=3}^{\infty}\frac{\varrho^{k-2}}{k!}
     = C_{m,F} \frac{e^{\varrho}-1-\varrho-\tfrac{1}{2}\varrho^{2}}{\varrho^{2}},
\end{align}
where $C_{m,F}<\infty$ depends only on the fixed truncation order $m$. The closed form follows by writing $\sum_{k\ge 3}\varrho^{k}/k!=e^{\varrho}-1-\varrho-\tfrac{1}{2}\varrho^{2}$ and dividing by $\varrho^{2}$. As $\varrho\to 0$, the Taylor expansion $e^{\varrho}-1-\varrho-\tfrac{1}{2}\varrho^{2}=\tfrac{1}{6}\varrho^{3}+O(\varrho^{4})$ gives $C_{m,F}\sum_{k=3}^{\infty}\frac{\varrho^{k-2}}{k!}=C_{m,F}\{\frac{\varrho}{6}+O(\varrho^{2})\}=O(\varrho)$ as $\varrho\to 0$, dominated by the $k=3$ slot.

\paragraph{Separate bound on the integral remainder $\tilde E_{\lambda,m}^{\mathrm{Pois,ro}}$.}
The Lagrange-type remainder of \eqref{eqn:bir-Etilde}-\eqref{eqn:bir-Theta} is structurally distinct from the Taylor terms and admits its own bound. Writing $A_{j}:=\prod_{\ell=0}^{j-1}|\lambda-\ell|$ and $B_{m-j}:=\prod_{\ell=0}^{m-j-1}(\lambda+\ell-1)$ for the two $\lambda$-factor products in \eqref{eqn:bir-Etilde}, we bound each piece in turn.

\paragraph{Effective range of $j$.} The factor $A_{j}$ vanishes at
$\ell=\lambda$, so $j\ge\lambda+1$ contributes zero, and $\binom{m}{j}=0$ for $j>m$. Hence the sum runs effectively over $j\in\{0,1,\ldots,\min(\lambda,m)\}$.

\paragraph{Per-term $\lambda$-factor bounds.}
$A_{j}=\lambda(\lambda-1)\cdots(\lambda-j+1)\le \lambda^{j}$ and $B_{m-j}=(\lambda-1)\lambda\cdots(\lambda+m-j-2)\le (\lambda+m)^{m-j}$, each by bounding every factor by its largest member.

\paragraph{Binomial-theorem summation over $j$.} Summing the per-term
bound above weighted by $\binom{m}{j}$,
\begin{align*}
    \sum_{j=0}^{m}\binom{m}{j} A_{j} B_{m-j}
     \le \sum_{j=0}^{m}\binom{m}{j} \lambda^{j} (\lambda+m)^{m-j}
     = (\lambda+\lambda+m)^{m} = (2\lambda+m)^{m},
\end{align*}
by the binomial theorem applied to $(a+b)^{m}$ with $a=\lambda$ and $b=\lambda+m$. For $\lambda\gg m$, $(2\lambda+m)^{m}=(2\lambda)^{m}(1+m/(2\lambda))^{m}\le (2\lambda)^{m}\,e^{m^{2}/(2\lambda)}=O(\lambda^{m})$.

\paragraph{The $(1-q)$ factor.} This is the only place where the
sampling-ratio condition is used. Uniformly in $j$, $(1-q)^{-(\lambda+m-j-1)}\le (1-q)^{-(\lambda+m-1)}$ and, more precisely, the integral remainder carries the multiplicative factor
\begin{align*}
    (1-q)^{-(\lambda+m-1)}
    \le \exp \left\{\frac{(\lambda+m-1)q}{1-q}\right\}.
\end{align*}
At the optimizing order used below, and at the leading boundary scale, $\lambda^\star\asymp L/\epsilon$ with $L=\log(1/\delta)$. Therefore this factor is uniformly absorbed into the constants when $qL/\{(1-q)\epsilon\}=O(1)$. If $q$ is bounded away from one, this reduces to the simpler condition $qL/\epsilon=O(1)$. Without such a condition the same proof would carry an additional $\exp\{qL/((1-q)\epsilon)\}$ multiplier in the integral-remainder term, and the clean leading expansion would no longer follow from the stated small-o conditions alone.

\paragraph{The $\Theta$ factor.} For $j<\lambda$, \eqref{eqn:bir-Theta} reads
$\Theta_{\lambda,m,j}(q,w)=\tilde B_m(w)+\sum_{\ell=0}^{\lambda-j}q^{\ell}c_{\ell}\tilde B_{m+\ell}(w)$ with $c_{\ell}:=\frac{(\lambda-j)!\,m!}{(\lambda-j-\ell)!(m+\ell)!}$. Each factor in the numerator of $c_{\ell}$ is at most $\lambda$, each in the denominator is at least $m+1$, so
\begin{align*}
    c_{\ell} = \frac{(\lambda-j)(\lambda-j-1)\cdots(\lambda-j-\ell+1)}{(m+1)(m+2)\cdots(m+\ell)}
     \le \Bigl(\frac{\lambda}{m+1}\Bigr)^{\ell}.
\end{align*}
Combining with $\tilde B_{m+\ell}(w)=O((wC/\sigma)^{m+\ell}) =O((wC/\sigma)^{m})\cdot(C/(\sigma n_S))^{\ell}$ (via \eqref{eqn:step4-Btilde-uniform} and the a.s.\ bound $w\le 1/n_S$),
\begin{align*}
    |\Theta_{\lambda,m,j}(q,w)|
     \le  \tilde B_m(w) \Bigl[ 1+O(1)\sum_{\ell=0}^{\infty}(\tfrac{q\lambda C}{(m+1)\sigma n_S})^{\ell} \Bigr]
     = O(\tilde B_m(w)),
\end{align*}
where the geometric series converges whenever $\varrho=q\lambda C/(\sigma n_S)<m+1$, which is precisely the same small-$\varrho$ condition used in \eqref{eqn:step4-ratio-sum}. For $j=\lambda$, $\Theta_{\lambda,m,\lambda}=\tilde B_m(w)$ directly. So $|\Theta_{\lambda,m,j}(q,w)|=O(\tilde B_m(w))$ uniformly in $j$.

\paragraph{Assembling.}
Putting the bounds above together,
\begin{align*}
\begin{aligned}
    |\tilde E_{\lambda,m}^{\mathrm{Pois,ro}}(q,w)|
    & \le \frac{q^{m}}{m!}\cdot O(\tilde B_m(w))\cdot O(1)\cdot (2\lambda+m)^{m}\\
    & = O (\tfrac{(q\lambda)^{m}}{m!} \tilde B_m(w))
     = O (\tfrac{(q\lambda)^{m}}{m!} (wC/\sigma)^{m}),
\end{aligned}
\end{align*}
where the second equality absorbs $(2\lambda+m)^{m}/\lambda^{m}=O(1)$ from the binomial summation and the last uses $\tilde B_m(w)=O\{(wC/\sigma)^m\}$ from \eqref{eqn:step4-Btilde-uniform}. Taking expectation against $\bar w$ and applying the a.s. bound $\bar w\le 1/n_S$ via $\E[\bar w^{m}]=\E[\bar w^{2}\cdot\bar w^{m-2}]\le \E[\bar w^{2}]/n_S^{m-2}$,
\begin{align*}
    \E [\tilde E_{\lambda,m}^{\mathrm{Pois,ro}}(q,\bar w)]
     \le O (\tfrac{q^{m} \lambda^{m}}{m!} (C/\sigma)^{m} \E[\bar w^{2}]/n_S^{m-2})
     = O (\tfrac{q^{m} \lambda^{m}}{m!} (C/\sigma)^{2} \E[\bar w^{2}] (C/(\sigma n_S))^{m-2}).
\end{align*}
Dividing by $q^{2}\E[\Phi_\lambda(\bar w)]$ from \eqref{eqn:step4-Phi-leading},
\begin{align}\label{eqn:step4-Etilde-ratio}
    \frac{\E [\tilde E_{\lambda,m}^{\mathrm{Pois,ro}}(q,\bar w)]}{q^{2}\E[\Phi_\lambda(\bar w)]}
     = O (\tfrac{1}{m!} (q\lambda\cdot C/(\sigma n_S))^{m-2})
     = O (\varrho^{m-2}/m!),
    \qquad m\ge 3.
\end{align}

\paragraph{Summary: characterization of $\exp\{A_{\cM_t}(\lambda)\}$.}
Combining the per-iteration bound \eqref{eqn:bir-Aprim} from Step~2 with the linearization \eqref{eqn:step4-Phi-leading} of the leading slot, the bound \eqref{eqn:step4-ratio-bound}-\eqref{eqn:step4-ratio-sum} on the higher-$k$ Taylor terms, and the bound \eqref{eqn:step4-Etilde-ratio} on the integral remainder, the per-iteration moment-generating function admits the explicit characterization
\begin{align*}
    \exp\{A_{\cM_t}(\lambda)\}
     &\le 1 + q^{2} \E_{\bar w} [\Phi_\lambda(\bar w)]
         + \sum_{k=3}^{m-1}\tfrac{q^{k}}{k!} \E_{\bar w} [\tilde F_{\lambda,k}^{\mathrm{Pois,ro}}(\bar w)]
         + \E_{\bar w} [\tilde E_{\lambda,m}^{\mathrm{Pois,ro}}(q,\bar w)]\nonumber\\
     &\le 1 + q^{2} \E_{\bar w} [\Phi_\lambda(\bar w)] (1+O(\varrho))\nonumber\\
     &= 1 + \frac{2 q^{2}\lambda(\lambda-1)\gamma(\gamma+1) C^{2}}{\sigma^{2} (S\gamma+\alpha)(S\gamma+\alpha+1) n_S^{2}} (1+R_\Phi(\sigma,n_S;C)+O(\varrho)),
\end{align*}
where $\varrho$ denotes the same current-order quantity $q\lambda C/(\sigma n_S)$. Taking logarithms and using $\log(1+x)=x+O(x^{2})$,
\begin{align*}
    A_{\cM_t}(\lambda)
     \le \frac{2 q^{2}\lambda(\lambda-1)\gamma(\gamma+1) C^{2}}{\sigma^{2} (S\gamma+\alpha)(S\gamma+\alpha+1) n_S^{2}} (1+O(R_\Phi(\sigma,n_S;C)+\varrho)).
\end{align*}
This is the explicit per-iteration rate ready for composition over the $T$ outer iterations and $K$ inner Monte Carlo replicates. Along sequences with fixed truncation order, $C/(\sigma n_S)\to0$, and $\varrho(\lambda;\sigma,n_S)=o(1)$ at the R\'enyi order under consideration, the resulting expansion gives
\begin{align}\label{eqn:step4-gamma-tot}
\begin{aligned}
    \gamma_M^{\mathrm{tot}}(\lambda;\sigma)
    & := \frac{A_{\cM}(\lambda)}{\lambda-1}
     \le c(\sigma) \lambda (1+O(R_\Phi(\sigma,n_S;C)+\varrho)),\\
    c&=c(\sigma) := \frac{2 TKq^{2}\gamma(\gamma+1) C^{2}}{\sigma^{2} (S\gamma+\alpha)(S\gamma+\alpha+1) n_S^{2}}.
\end{aligned}
\end{align}

\paragraph{Conversion to approximate DP.}
Step~3 leaves $\cM$ certified as $(\lambda,\gamma_M^{\mathrm{tot}}(\lambda;\sigma))$-RDP at every integer $\lambda\ge 2$, with the slope-rate decomposition
\begin{align*}
    \gamma_M^{\mathrm{tot}}(\lambda;\sigma)
     \le c(\sigma) \lambda (1+\Delta_\lambda),
    \qquad \Delta_\lambda= C_{\mathrm{rem}} \{R_\Phi(\sigma,n_S;C)+\varrho\},
\end{align*}
for some constant $C_{\text{rem}}$, where $c(\sigma)$ is
\begin{align}\label{eqn:step4-c-def}
    c(\sigma) := \frac{2TKq^{2}\gamma(\gamma+1) C^{2}}{\sigma^{2}(S\gamma+\alpha)(S\gamma+\alpha+1) n_S^{2}},
\end{align}
and $\Delta_\lambda$ is the relative slack inherited from \eqref{eqn:step4-gamma-tot}. Here $\varrho$ means $\varrho(\lambda;\sigma,n_S)=q\lambda C/(\sigma n_S)$ at the same integer order $\lambda$. Mironov's RDP-to-DP conversion \citep[Prop.~3]{miro:17} bounds the privacy budget at every integer $\lambda\ge 2$ by
\begin{align}\label{eqn:step4-mironov-new}
    \epsilon(\lambda,\delta)
     \le c(\sigma) \lambda (1+\Delta_\lambda) + \frac{\log(1/\delta)}{\lambda-1}.
\end{align}

To extract the leading scale, write $L:=\log(1/\delta)$ and drop $\Delta_\lambda$ from the optimization. The reduced objective $\psi(\lambda):=c(\sigma)\lambda+L/(\lambda-1)$ has continuous minimizer
\begin{align*}
    \lambda^{\star} = 1+\sqrt{L/c(\sigma)},
    \qquad
    \psi(\lambda^{\star}) = c(\sigma)+2\sqrt{c(\sigma) L},
\end{align*}
where the continuous optimizer is used only to read off the leading scale. If Birrell's integer-order restriction is enforced literally, take $\hat\lambda^\star=\max\{2,\lceil\lambda^\star\rceil\}$. Writing $s=\lambda^\star-1=\sqrt{L/c(\sigma)}$ and $d=\hat\lambda^\star-\lambda^\star\in[0,1)$ gives $\psi(\hat\lambda^\star)-\psi(\lambda^\star)=c(\sigma)d^2/(s+d)\le c(\sigma)$, and at the calibrated scale, where $c(\sigma)\asymp\epsilon^2/L$, this is $O(\epsilon^2/L)=\epsilon\,O(\epsilon/L)$. Hence, the integer rounding does not change the leading calibration. The regime condition $c(\sigma)\le L$ is enough to keep the linear term no larger than the square-root term:
\begin{align}\label{eqn:step4-eps-lead-scale-new}
    \varepsilon_{\mathrm{lead}}(\delta)
     = c(\sigma)+2\sqrt{c(\sigma) L}
     \lesssim \sqrt{c(\sigma)L}=  \frac{Cq\sqrt{TK L \gamma(\gamma+1)/[(S\gamma+\alpha)(S\gamma+\alpha+1)]}}{\sigma n_S}.
\end{align}
Requiring $\varepsilon_{\mathrm{lead}}\lesssim \epsilon$ forces the noise scale to obey
\begin{align}\label{eqn:step4-sigma-leading-new}
    \sigma \gtrsim \frac{C \sqrt{TKq^{2} L \gamma(\gamma+1)}}
                         {n_S \sqrt{(S\gamma+\alpha)(S\gamma+\alpha+1)} \epsilon},
\end{align}
which is the leading-order scale of Theorem~\ref{thm:asymp-dp}.

\paragraph{The dropped slack.}
It remains to verify that dropping $\Delta_\lambda$ in the optimization costs nothing at leading order. Substituting $\lambda^{\star}-1=\sqrt{L/c(\sigma)}$ into $\varrho(\lambda;\sigma,n_S)=q\lambda C/(\sigma n_S)$ and then evaluating at $\lambda=\lambda^\star$ splits the relative slack into a $\sigma$-dependent piece and a purely external piece:
\begin{align*}
    \left.\varrho\right|_{\lambda=\lambda^\star}
     = \frac{qC}{\sigma n_S} + \sqrt{L/c(\sigma)}\cdot\frac{qC}{\sigma n_S}
     = \frac{qC}{\sigma n_S} + \varrho^{\star},
\end{align*}
where the $\sigma n_S/C$ implicit in $\sqrt{L/c(\sigma)}$ cancels the $C/(\sigma n_S)$ multiplier to leave the external quantity $\varrho^{\star}=\sqrt{L/(2TK\,\kappa_{S,\gamma,\alpha})}$ with $\kappa_{S,\gamma,\alpha}\;:=\;\frac{\gamma(\gamma+1)}{(S\gamma+\alpha)(S\gamma+\alpha+1)}$. If the integer-order restriction is enforced literally, take $\hat\lambda^\star=\lceil\lambda^\star\rceil$. Then $\varrho(\hat\lambda^\star;\sigma,n_S)=\varrho(\lambda^\star;\sigma,n_S)+q(\hat\lambda^\star-\lambda^\star)C/(\sigma n_S)$, so rounding adds at most the already controlled term $qC/(\sigma n_S)$. The corresponding absolute slack on the right-hand side of \eqref{eqn:step4-mironov-new} is
\begin{align*}
    c(\sigma) \lambda^{\star} \Delta_{\lambda^{\star}}
    = C_{\mathrm{rem}} (c+\sqrt{c(\sigma)L})(R_\Phi(\sigma,n_S;C)+\tfrac{qC}{\sigma n_S}+\varrho^{\star}),
\end{align*}
and dividing by $\varepsilon_{\mathrm{lead}}$,
\begin{align*}
    \frac{c(\sigma) \lambda^{\star} \Delta_{\lambda^{\star}}}{\varepsilon_{\mathrm{lead}}}
     \leq C_{\mathrm{rem}} \Bigl( R_\Phi(\sigma,n_S;C)+\varrho^{\star} + \frac{qC}{\sigma n_S} \Bigr).
\end{align*}
Thus the dropped slack is asymptotically negligible only under the two small-o expansion requirements $\varrho^\star=o(1)$ and $C/(\sigma n_S)=o(1)$, since the latter also gives $R_\Phi(\sigma,n_S;C)=o(1)$ and $qC/(\sigma n_S)=o(1)$. Separately, the reduced Mironov objective is kept in the square-root regime by the constant-order regime condition $c(\sigma)\le L$. We therefore use
\begin{align}\label{eqn:step4-two-regimes}
    \varrho^{\star} \ll 1,
    \qquad
    \frac{C}{\sigma n_S} \ll 1, \qquad c(\sigma)\le L,
\end{align}
When the RDP-to-DP bound is optimized over $\lambda$, the expansion is invoked only at the selected order (and, if needed, at its integer rounding). Thus $\varrho=o(1)$ means the selected-order condition $\varrho(\lambda^\star;\sigma,n_S)=o(1)$, equivalently $\varrho(\hat\lambda^\star;\sigma,n_S)=o(1)$ under the preceding rounding bound. It is not a uniform assertion over all $\lambda$ in the infimum. In \eqref{eqn:step4-two-regimes}, the $\varrho^\star$ condition is external to the noise calibration and is exactly the first regime condition of Theorem~\ref{thm:asymp-dp}. The condition $C/(\sigma n_S)\ll1$ asks the added noise to dominate the per-record scale $C/n_S$. The regime condition is left in the exact form $c(\sigma)\le L$ rather than being converted into a separate noise scale, since that conversion only hides constants. Thus the leading-order calibration is justified for any $\sigma$ satisfying
\begin{align}\label{eqn:step4-safe-sigma}
    \sigma \gtrsim \frac{Cq}{n_S}\frac{\sqrt{TK L \kappa_{S,\gamma,\alpha}}}{\epsilon},
    \qquad
    c(\sigma)\le L,
    \qquad
    \frac{C}{n_S}=o(\sigma).
\end{align}
Since the finite accountant is monotone decreasing in $\sigma$, once these conditions hold at some boundary scale $\sigma_0$, any $\sigma\ge\sigma_0$ also satisfies the privacy bound.


\newpage

\subsubsection{Privacy Amplification from Dirichlet Randomness (Remark~\ref{rem:amplification})}
\label{app:privacy-amplification}

We compare two leading-order calibrations that follow the same privacy accounting chain in Sections~\ref{sec:bir-step1}-\ref{sec:bir-step4}. The difference is how the record-level block weight is handled. The Dirichlet-aware accountant integrates over the actual law of the block weight, whereas the worst-case accountant ignores this law and uses only the deterministic envelope $\bar w_i\le 1/n_S$.

\paragraph{Dirichlet moment versus worst-case envelope.}
Fix the possibly changed record, written as record $n$ without loss of generality. Under the $\alpha=0$ blocked bootstrap used in the main text,
\[
    \bar w_n=\frac{u_{\pi(n)}}{n_S},
    \qquad
    u_{\pi(n)}\sim\Beta(\gamma,(S-1)\gamma).
\]
Hence
\[
    \E_{\mathrm{Dir}}[\bar w_n^2]
    =
    \frac{1}{n_S^2}
    \frac{\gamma(\gamma+1)}{S\gamma(S\gamma+1)}
    =
    \frac{\kappa_{S,\gamma}}{n_S^2},
    \qquad
    \kappa_{S,\gamma}
    :=
    \frac{\gamma+1}{S(S\gamma+1)}.
\]
By contrast, the worst-case calibration does not assign a second moment to $\bar w_n$. It replaces the random weight by the deterministic upper bound
\[
    \bar w_n^2\le \frac{1}{n_S^2}.
\]
Thus the comparison is between the actual Dirichlet second moment $\kappa_{S,\gamma}/n_S^2$ and the worst-case envelope $1/n_S^2$. The leading coefficient is therefore reduced by the factor $\kappa_{S,\gamma}$ when the Dirichlet randomness is used.

\paragraph{Propagation to the leading privacy scale.}
At leading order, \eqref{eqn:step4-Phi-leading} gives
\[
    \E_{\bar w}[\Phi_\lambda(\bar w)]
    =
    2\lambda(\lambda-1)(C^2/\sigma^2)
    \E[\bar w_n^2]\{1+R_\Phi(\sigma,n_S;C)\}.
\]
Accordingly, the slope rate in \eqref{eqn:step4-c-def} takes the form
\[
    c_\bullet(\sigma)
    =
    \frac{2TKq^2C^2 m_\bullet}{\sigma^2 n_S^2},
    \qquad
    m_{\mathrm{Dir}}=\kappa_{S,\gamma},
    \qquad
    m_{\mathrm{worst}}=1.
\]
The RDP-to-DP conversion in \eqref{eqn:step4-mironov-new} gives the reduced leading objective
\[
    \varepsilon_{\mathrm{lead},\bullet}
    =
    c_\bullet(\sigma)+2\sqrt{c_\bullet(\sigma)L},
    \qquad
    L:=\log(1/\delta).
\]
In the square-root regime $c_\bullet(\sigma)\le L$, requiring
$\varepsilon_{\mathrm{lead},\bullet}\lesssim \epsilon$ yields
\[
    \Sigma_{0,\bullet}
    \asymp
    \frac{Cq\sqrt{TKL\,m_\bullet}}{n_S\epsilon}.
\]
Therefore
\[
    \Sigma_{0,\mathrm{Dir}}
    \asymp
    \frac{Cq\sqrt{TKL\,\kappa_{S,\gamma}}}{n_S\epsilon},
    \qquad
    \Sigma_{0,\mathrm{worst}}
    \asymp
    \frac{Cq\sqrt{TKL}}{n_S\epsilon},
\]
and the leading noise-scale ratio is
\[
    \frac{\Sigma_{0,\mathrm{Dir}}}{\Sigma_{0,\mathrm{worst}}}
    \asymp
    \sqrt{\kappa_{S,\gamma}}.
\]
This is the amplification effect stated in Remark~\ref{rem:amplification}.

Notably, the comparison above is a leading-order comparison and is valid only when the same expansion conditions used in Theorem~\ref{thm:asymp-dp} hold for the calibration under consideration. With $m_\bullet\in\{\kappa_{S,\gamma},1\}$, the selected-order slack in Step~4 becomes
\[
    \varrho^\star_\bullet
    =
    \sqrt{\frac{L}{2TK\,m_\bullet}}.
\]
Thus the Dirichlet-aware calibration requires
$\sqrt{L/(2TK\kappa_{S,\gamma})}\ll1$, while the worst-case calibration requires $\sqrt{L/(2TK)}\ll1$. Since $\kappa_{S,\gamma}\le1$, the former is the more restrictive selected-order condition.

The local expansion also requires $C/(\sigma n_S)\ll1$. Evaluated at the boundary scale $\Sigma_{0,\bullet}$,
\[
    \frac{C}{\Sigma_{0,\bullet} n_S}
    \asymp
    \frac{\epsilon}{q\sqrt{TKL\,m_\bullet}}.
\]
Hence the Dirichlet-aware version is larger by the factor $1/\sqrt{\kappa_{S,\gamma}}$ relative to the worst-case version, again because it uses a smaller noise scale. Finally, the square-root regime condition $c_\bullet(\sigma)\le L$ reduces at $\sigma=\Sigma_{0,\bullet}$ to the same requirement $\epsilon\lesssim L$ for both calibrations. The sampling-ratio condition
\[
    \frac{q\log(1/\delta)}{(1-q)\epsilon}=O(1)
\]
does not involve $\kappa_{S,\gamma}$ and is therefore identical in the two comparisons.

\newpage

\subsubsection{Privacy-Side Effect of Pseudo-Samples}
\label{app:pseudo-sample-alpha}

The analysis in Section~\ref{sec:bir-step4} is written for the general posterior-bootstrap setup with $\alpha>0$ \citep{fong:etal:19,ohag:rock:25}. Suppose that, in addition to the private observations $X_1,\dots,X_n$, we have data-independent pseudo-samples $\tilde X_1,\dots,\tilde X_{\tilde n}\sim F_{\mathrm{base}}$ where $F_{\mathrm{base}}$ is a pre-specified distribution. Partition the private observations into $S$ equal-sized blocks and also the pseudo-samples into $\tilde S$ equal-sized blocks, with sizes $n_S=n/S$ and $\tilde n_{\tilde S}=\tilde n/\tilde S$. Draw $(\bu,\tilde\bu)\sim
    \Diri(\gamma{\bf 1}_S,\alpha{\bf 1}_{\tilde S}/\tilde S)$ and set $\bar w_i=u_{\pi(i)}/n_S$ and
$\bar{\tilde w}_j=\tilde u_{\tilde\pi(j)}/\tilde n_{\tilde S}$. The amortized objective for the posterior bootstrap extends \eqref{eqn:nbb} to
\[
    \min_{G\in\mathcal G}
    \E_{(\bu,\tilde\bu)}
    \left[
    \sum_{i=1}^n \bar w_i\,l(G(\bu,\tilde\bu),X_i)
    +\sum_{j=1}^{\tilde n}\bar{\tilde w}_j\,
    l(G(\bu,\tilde\bu),\tilde X_j)
    \right].
\]
For $\alpha=0$, the pseudo-sample part is omitted and this reduces to the blocked Bayesian-bootstrap objective used in the main text.

The pseudo-samples are treated here as fixed independently of the private database. If they were generated from private data, their generation would need its own privacy accounting. Under this data-independent interpretation, the pseudo-samples affect the privacy analysis only through the Dirichlet allocation. For a private record, $u_{\pi(i)}\sim\Beta(\gamma,(S-1)\gamma+\alpha)$, and hence
\[
    \E[\bar w_i^2]
    =
    \frac{\gamma(\gamma+1)}
    {(S\gamma+\alpha)(S\gamma+\alpha+1)n_S^2}
    =
    \frac{\kappa_{S,\gamma,\alpha}}{n_S^2},
    \qquad
    \kappa_{S,\gamma,\alpha}
    :=
    \frac{\gamma(\gamma+1)}
    {(S\gamma+\alpha)(S\gamma+\alpha+1)}.
\]
Thus the safe scale in \eqref{eqn:step4-safe-sigma} carries
$\sigma\propto\sqrt{\kappa_{S,\gamma,\alpha}}$. Since
$\kappa_{S,\gamma,\alpha}$ decreases with $\alpha$, larger prior concentration can reduce the privacy noise required for a fixed $(\epsilon,\delta)$. This privacy-side gain comes with statistical shrinkage toward $F_{\mathrm{base}}$, so $\alpha>0$ trades lower privacy noise against possible bias.

\newpage 

\subsection{Proof of Theorem~\ref{thm:conv-gbb}}
\label{app:thm-conv-bridge}

\subsubsection{Setup}
\label{app:thm-conv-setup}


The data $\bX_{1:n}=(X_1,\dots,X_n)\stackrel{\mathrm{iid}}{\sim}F$ and the random partition $\pi:\{1,\dots,n\}\to\{1,\dots,S\}$ generate the sample and its blocks, and the random initialization $\phi_0$ is drawn from the known distribution (e.g., He/Xavier Gaussian). Training uses the per-step Dirichlet draws $\{\bu_t^{(k)}\}_{t<T,k\le K}\stackrel{\mathrm{iid}}{\sim}\Diri(\gamma\mathbf 1_S)$, the Poisson subsampling indicators $\{I_{j,t}^{(k)}\}\stackrel{\mathrm{iid}}{\sim}\mathrm{Bernoulli}(q)$, and the Gaussian privacy noise $\{\zeta_t^{(k)}\}\stackrel{\mathrm{iid}}{\sim}\mathcal N(0,\sigma^2 I_{d_\theta})$. Collectively we write $\bomega_T=(\phi_0,\{\bu_t^{(k)},I_{j,t}^{(k)},\zeta_t^{(k)}\}_{t<T,\,k\le K})$ for the algorithmic randomness of Section~\ref{sec:theory}, so that $\mathcal F_T=\sigma(\bX_{1:n},\pi,\bomega_T)$. Finally, the fresh test draw $\bu\sim\Diri(\gamma\mathbf 1_S)$ enters Theorem~\ref{thm:conv-gbb} and is independent of training.  Define $\mathcal F^{\rm data}:=\sigma(\bX_{1:n},\pi,\phi_0)$ and $\mathcal F_t := \mathcal F^{\rm data}\vee\sigma(\bu_s^{(k)},I_{j,s}^{(k)},\zeta_s^{(k)}:s<t, k\le K, j\le n)$, i.e., $\mathcal F^{\rm data}\subseteq\mathcal F_0\subseteq\mathcal F_1\subseteq\cdots\subseteq\mathcal F_T$ and $\phi_t$ is $\mathcal F_t$-measurable.  Namely, $\bomega_T=(\phi_0,\{\bu_t^{(k)},I_{j,t}^{(k)},\zeta_t^{(k)}\}_{t<T,\,k\le K})$, so that $\mathcal F_T=\sigma(\bX_{1:n},\pi,\bomega_T)$ and the fresh input $\bu$ is independent of $\mathcal F_T$. The objective
$\mathcal{L}(\phi;\bX_{1:n},\pi):=\E_{\bu}[L(G_\phi(\bu);\bar\bw(\bu))\mid \bX_{1:n},\pi]$ is $\sigma(\bX_{1:n},\pi)$-measurable as a function of $\phi$ (independent of $\phi_0$), i.e., the infimum $\mathcal{L}^\star_{\Phi}$ and the local minimizer $\phi^\star$ (in $\Phi$) are accordingly random functions of $(\bX_{1:n},\pi)$. Note the structural conditional independence $\phi_T\;\perp\!\!\!\perp\;\bu\;\big|\;(\bX_{1:n},\pi,\phi_0)$ holds. Let's define common notations $\psi_j(\bu,\phi)\;:=\;\nabla_\theta l(x_j,G_\phi(\bu)) \nabla_\phi G_\phi(\bu)\;\in\;\R^{d_\phi}$. By dominated convergence due to the bounded gradient and $\beta_1$, it follows that
\begin{align}
\label{eqn:app-conv-nablaF}
\nabla \mathcal{L}(\phi;\bX_{1:n},\pi) = \E_{\bu}\Bigl[\sum_{j=1}^n\bar w_j(\bu) \psi_j(\bu,\phi)\Bigr].
\end{align}
The stochastic gradient $\hat g_t(\phi_t)$ of Algorithm~\ref{alg:bayesian_bootstrap} is, with $I_j^{(k)}:=\mathbf 1\{j\in B_t^{(k)}\}$ and $\zeta_t^{(k)}\sim\mathcal N(0,\sigma^2 {\bf I}_{d_\theta})$ mutually independent across $k$ and of $\mathcal F_t$,
\begin{align*}
\hat g_t(\phi_t)=\frac{1}{K}\sum_{k=1}^K
\Big(\sum_{j=1}^n I_j^{(k)} \bar w_j^{(k)} \psi_j(\bu^{(k)},\phi_t) + \zeta_t^{(k)}\nabla_\phi G_{\phi_t}(\bu^{(k)})\Big).
\end{align*}

\subsubsection{Statements and proofs of the lemmas}
\label{app:thm-conv-lemmas}

\begin{lemma}
\label{lem:appB-doob}
Under (A1-5) and the bounded-weight network class $\Phi$, for any $\mathcal F_t$-measurable iterate $\phi_t\in\Phi$,
\begin{align}
\label{eqn:appB-Egt}
\E[\hat g_t(\phi_t)\mid\mathcal F_t] = q \nabla \mathcal{L}(\phi_t;\bX_{1:n},\pi),
\end{align}
and the centered residual $\xi_t:=\hat g_t(\phi_t)-q\nabla \mathcal{L}(\phi_t;\bX_{1:n},\pi)$ admits the orthogonal decomposition $\xi_t=\xi_t^{\rm MC}+\xi_t^{\rm Pois}+\xi_t^{\rm DP}$ with
\begin{align*}
\xi_t^{\rm MC}
&:=\frac{q}{K}\sum_{k=1}^K\Big(\sum_{j=1}^n\bar w_j^{(k)}\psi_j(\bu^{(k)},\phi_t)-\nabla \mathcal{L}(\phi_t;\bX_{1:n},\pi)\Big),\\
\xi_t^{\rm Pois}
&:=\frac{1}{K}\sum_{k=1}^K\sum_{j=1}^n(I_j^{(k)}-q)\bar w_j^{(k)}\psi_j(\bu^{(k)},\phi_t),\\
\xi_t^{\rm DP}
&:=\frac{1}{K}\sum_{k=1}^K\zeta_t^{(k)}\nabla_\phi G_{\phi_t}(\bu^{(k)}),\notag
\end{align*}
each $\mathcal F_t$-conditionally mean zero and $\mathcal F_t$-conditionally pairwise uncorrelated. Their second moments satisfy
\begin{align}
\label{eqn:appB-vars}
\E[\|\xi_t^{\rm DP}\|^2\mid\mathcal F_t] \le \frac{\beta_1^2 d_\theta \sigma^2}{K},
\quad
\E[\|\xi_t^{\rm Pois}\|^2\mid\mathcal F_t] \le \frac{\beta_1^2 C^2 q(1-q) (\gamma+1) S}{K n (S\gamma+1)},
\end{align}
\begin{align}
\label{eqn:appB-var-MC}
\E[\|\xi_t^{\rm MC}\|^2\mid\mathcal F_t] \le \frac{q^2 \beta_1^2 C^2}{K}.
\end{align}
\end{lemma}
\begin{proof}
For~\eqref{eqn:appB-Egt}, condition on $\mathcal F_t$. The fresh draws $(\bu^{(k)},B_t^{(k)},\zeta_t^{(k)})_{k=1}^K$ are independent of $\mathcal F_t$, so $\E[I_j^{(k)}\mid\mathcal F_t]=q$, $\E[\zeta_t^{(k)}\mid\mathcal F_t]=0$, and $\E[\sum_j\bar w_j^{(k)}\psi_j(\bu^{(k)},\phi_t)\mid\mathcal F_t]=\nabla \mathcal{L}(\phi_t;\bX_{1:n},\pi)$ by~\eqref{eqn:app-conv-nablaF}, and averaging over $k$ yields $q \nabla \mathcal{L}(\phi_t;\bX_{1:n},\pi)$. The decomposition follows from the tower identity $\xi_t=(\hat g_t-\E[\hat g_t\mid\mathcal G_t])+(\E[\hat g_t\mid\mathcal G_t]-q\nabla \mathcal{L}(\phi_t;\bX_{1:n},\pi))$ with $\mathcal G_t:=\sigma(\mathcal F_t,\bu^{(1)},\dots,\bu^{(K)})$. The inner residual is $\xi_t^{\rm Pois}+\xi_t^{\rm DP}$ and the outer is $\xi_t^{\rm MC}$. For pairwise uncorrelatedness, since ${\cal F}_t\subset {\cal G}_t$,
\begin{align*}
    \E[\langle \xi_t^{\text{MC}},\xi_t^{\text{Pois}}\rangle|{\cal F}_t]=\E[\E[\langle \xi_t^{\text{MC}},\xi_t^{\text{Pois}}\rangle |{\cal G}_t] |{\cal F}_t]=\E[\langle \xi_t^{\text{MC}},\E[\xi_t^{\text{Pois}}|{\cal G}_t]\rangle|{\cal F}_t]=0, \because \E[\xi_t^{\text{Pois}}|{\cal G}_t]=0,
\end{align*}
\begin{align*}
    \E[\langle \xi_t^{\text{MC}},\xi_t^{\text{DP}}\rangle|{\cal F}_t]=\E[\E[\langle \xi_t^{\text{MC}},\xi_t^{\text{DP}}\rangle |{\cal G}_t] |{\cal F}_t]=\E[\langle \xi_t^{\text{MC}},\E[\xi_t^{\text{DP}}|{\cal G}_t]\rangle|{\cal F}_t]=0, \quad \because \E[\xi_t^{\text{DP}}|{\cal G}_t]=0.
\end{align*}
Define ${\cal H}_t:=\sigma({\cal G}_t,I_j^{(k)}:j\leq n,k\leq K)$. Then $\xi_{t}^{\text{Pois}}$ is ${\cal H}_t$-measurable and $\E[\xi_t^{\text{DP}}|{\cal H}_t]=0$. Hence, by the iterative rule with respect to $ {\cal F}_t \subset {\cal H}_t$, we also have $ \E[\langle \xi_t^{\text{Pois}},\xi_t^{\text{DP}}\rangle|{\cal F}_t]=0$.

\paragraph{Gaussian privacy variance.} The summands in $\xi_t^{\rm DP}$ are conditionally uncorrelated across $k$.  Also $\|\zeta_t^{(k)}\nabla_\phi G_{\phi_t}(\bu^{(k)})\|^2 \le\beta_1^2\|\zeta_t^{(k)}\|^2$, and $\E\|\zeta_t^{(k)}\|^2=d_\theta\sigma^2$. Therefore,  the $1/K^2$ averaging over $K$ independent replicates gives the first bound in~\eqref{eqn:appB-vars}.

\paragraph{Poisson variance.}
Condition on $\mathcal G_t$, under which the weights $\bar w_j^{(k)}$ and scores $\psi_j(\bu^{(k)},\phi_t)$ are deterministic. The Poisson indicators $\{I_j^{(k)}\}_j$ are independent $\mathrm{Bernoulli}(q)$, so the inner $j$-sum has conditional second moment
\begin{align*} 
\E\Bigl[\|\sum_j(I_j^{(k)}-q)\bar w_j^{(k)}\psi_j\|^2 \big| \mathcal G_t\Bigr] =q(1-q)\sum_{j=1}^n(\bar w_j^{(k)})^2 \|\psi_j(\bu^{(k)},\phi_t)\|^2.
\end{align*}
Using $\|\psi_j\|\le\beta_1 C$ from the Jacobian bound $\beta_1$ and the within-group constancy $\bar w_j=u_{\pi(j)}/n_S$ (so $\sum_j\bar w_j^2=n_S^{-1}\sum_s u_s^2=(S/n)\sum_s u_s^2$),
\begin{align*}
\E_{\bu}\Bigl[\sum_j\bar w_j^2 \|\psi_j\|^2\Bigr]
 \le \beta_1^2 C^2 \E_{\bu}\Bigl[\sum_j\bar w_j^2\Bigr]
 = \beta_1^2 C^2 \frac{S}{n} \E_{\bu}\Bigl[\sum_{s=1}^S u_s^2\Bigr].
\end{align*}
The Dirichlet second-moment identity (for $\bu\sim\Diri(\gamma\mathbf 1_S)$, $\E[u_s^2]=\Var(u_s)+(\E u_s)^2=\tfrac{S-1}{S^2(S\gamma+1)}+\tfrac1{S^2}$, hence $\E[\sum_s u_s^2]=\tfrac{\gamma+1}{S\gamma+1}$) then gives
\begin{align}
\label{eqn:appB-pois-Ew2}
\E_{\bu}\Bigl[\sum_{j=1}^n\bar w_j^2\Bigr] = \frac{(\gamma+1) S}{n (S\gamma+1)}.
\end{align}
Combining the conditional variance with~\eqref{eqn:appB-pois-Ew2} and the $1/K^2\cdot K$-averaging over the $K$ independent replicates yields the Poisson bound in~\eqref{eqn:appB-vars}.

\paragraph{Monte Carlo variance.} Write $Z_k:=\sum_{j=1}^n\bar w_j^{(k)}\psi_j(\bu^{(k)},\phi_t)$, so that $\xi_t^{\rm MC}=\frac{q}{K}\sum_{k=1}^K(Z_k-\nabla\mathcal L(\phi_t;\bX_{1:n},\pi))$ with $\nabla\mathcal L(\phi_t;\bX_{1:n},\pi)=\E_\bu Z_k$ by~\eqref{eqn:app-conv-nablaF}. The $Z_k$ are i.i.d.\ across $k$ (fresh $\bu^{(k)}$), so by independence and centring,
\begin{align*}
\E[\|\xi_t^{\rm MC}\|^2\mid\mathcal F_t]
=\frac{q^2}{K^2}\sum_{k=1}^K\E\|Z_k-\nabla\mathcal L\|^2
=\frac{q^2}{K} \Var_\bu(Z_1).
\end{align*}
Note $\Var_\bu(Z_1)\le\E\|Z_1\|^2\le\beta_1^2 C^2$ because of  $\|Z_1\|\le\sum_j\bar w_j\|\psi_j\|\le\beta_1 C\sum_j\bar w_j=\beta_1 C$.


\end{proof}

\begin{lemma}
\label{lem:appB-descent}
Suppose the conditions of Theorem~\ref{thm:conv-gbb} hold. By Lemma~\ref{lem:appB-doob}, the iterates $\phi_{t+1}=\phi_t-\eta\hat g_t(\phi_t)$ satisfy
\begin{align}
\label{eqn:appB-onestep}
\E[(\mathcal{L}(\phi_{t+1};\bX_{1:n},\pi)-\mathcal L^\star_{\Phi,\bX_{1:n},\pi}) \big| \mathcal F_t]
 \le 
(1-\eta q\mu_{\mathcal L})(\mathcal{L}(\phi_t;\bX_{1:n},\pi)-\mathcal L^\star_{\Phi,\bX_{1:n},\pi})
 + \frac{M_{\mathcal L}\eta^2}{2} \tau^2,
\end{align}
where $\tau^2$ is an upper bound of $\E[\|\xi_t\|^2\mid\mathcal F_t] =\E[\|\xi_t^{\rm MC}\|^2\mid\mathcal F_t] +\E[\|\xi_t^{\rm Pois}\|^2\mid\mathcal F_t] +\E[\|\xi_t^{\rm DP}\|^2\mid\mathcal F_t]$ by~\eqref{eqn:appB-vars}-\eqref{eqn:appB-var-MC}.
\end{lemma}

\begin{proof}
Since $M_{\mathcal L}\le\beta_1^2 K_2+\beta_2 C<\infty$, $\mathcal{L}(\phi;\bX_{1:n},\pi)$ is an $M_{\cal L}$-smooth function over the convex domain $\Phi$ for any given $\bX_{1:n},\pi$. Thus, the descent lemma gives $\mathcal{L}(\phi_{t+1};\bX_{1:n},\pi)\le \mathcal{L}(\phi_t;\bX_{1:n},\pi)+\langle\nabla \mathcal{L}(\phi_t;\bX_{1:n},\pi),\phi_{t+1}-\phi_t\rangle+\tfrac{M_{\mathcal L}}{2}\|\phi_{t+1}-\phi_t\|^2$. Substitute $\phi_{t+1}-\phi_t=-\eta(q\nabla \mathcal{L}(\phi_t;\bX_{1:n},\pi)+\xi_t)$ and take $\E[\cdot\mid\mathcal F_t]$. The linear cross term vanishes by Lemma~\ref{lem:appB-doob}, and
\begin{align*}
\E[\|q\nabla \mathcal{L}(\phi_t;\bX_{1:n},\pi)+\xi_t\|^2|\mathcal F_t]
 = q^2\|\nabla \mathcal{L}(\phi_t;\bX_{1:n},\pi)\|^2+\tau_t^2.
\end{align*}
where $\tau_t^2$ has the uniform upper bound $\tau^2$ from Lemma~\ref{lem:appB-doob}. Then we have \( \E[\mathcal{L}(\phi_{t+1};\bX_{1:n},\pi)\mid\mathcal F_t]\le \mathcal{L}(\phi_t;\bX_{1:n},\pi)-\eta q(1-\eta qM_{\mathcal L}/2)\|\nabla \mathcal{L}(\phi_t;\bX_{1:n},\pi)\|^2+(M_{\mathcal L}\eta^2/2)\tau^2 \) and subtract $\mathcal L^\star_{\Phi,\bX_{1:n},\pi}$ in both sides. The step-size choice $\eta\le 1/(qM_{\mathcal L})$ gives $1-\eta qM_{\mathcal L}/2\ge 1/2$, and the PL condition leads to $-\eta q(1-\eta qM_{\mathcal L}/2)\|\nabla\mathcal L\|^2\le -\eta q\mu_{\mathcal L}(\mathcal L(\phi_t;\bX_{1:n},\pi)-\mathcal L^\star_{\Phi,\bX_{1:n},\pi})$.
\end{proof}

\begin{lemma}
\label{lem:appB-var-bridge}
Assume (A1-5). Let $\mathcal E:=\mathcal G_n\cap\tilde E_n^\sub$ with $\mathcal G_n$ and $\tilde E_n^\sub$ as in Lemmas~\ref{lem:cm-reg} and~\ref{lem:cm-lin}, and adopt the Lemma~\ref{lem:cm-lin} notation $\theta^\sub:=\theta^\star_\sub(\bar\bw)$, $\mu:=\mu_\infty/2$, and $r$ the localization radius of Lemma~\ref{lem:cm-reg}. Then, it follows that
\begin{enumerate}
\item[(i)] on $\mathcal E$, for the realized block weight vector
$\bar\bw=\bar\bw(\bu)$ and every $\theta\in\Theta$,
\begin{align}
\label{eqn:appB-trunc-loj}
L(\theta;\bar\bw)-L(\theta^\sub;\bar\bw)
 \ge 
\begin{cases}
\dfrac{\mu}{4} \|\theta-\theta^\sub\|^2,
& \|\theta-\theta^\sub\|\le r/2,\\[8pt]
\dfrac{\mu r}{8} \|\theta-\theta^\sub\|,
& \|\theta-\theta^\sub\|>r/2.
\end{cases}
\end{align}
\item[(ii)] for every $\theta\in\Theta$,
\begin{align}
\label{eqn:appB-bridge}
\E[\mathbf 1_{\mathcal E} \|\theta-\theta^\sub\|^2]
 \le 
\frac{4}{\mu}\left(1+\frac{2D}{r}\right) 
\E[\mathbf 1_{\mathcal E}(L(\theta;\bar\bw)-L(\theta^\sub;\bar\bw))],
\end{align}
where $D_\Theta:=\diam(\Theta)<\infty$ by (A2).
\end{enumerate}
\end{lemma}

\begin{proof}
We first prove the deterministic growth bound~\eqref{eqn:appB-trunc-loj} on $\mathcal E$ and then deduce the expectation bound~\eqref{eqn:appB-bridge} by a two-regime split.

\paragraph{Setup on the good event.} By Lemma~\ref{lem:cm-lin}, applied at concentration $\gamma$, on $\mathcal E$ the weighted minimizer $\theta^\sub$ satisfies the quantitative bound $\|\theta^\sub-\theta_0^\star\|\le r/2$, sharper than the $\overline{B(\theta_0^\star,r)}$ containment.  By Lemma~\ref{lem:cm-lin}, the weighted Hessian satisfies $\nabla_\theta^2 L(\theta;\bar\bw)\succeq(\mu/2)I$ for every $\theta\in\overline{B(\theta_0^\star,r)}$.  The triangle inequality then gives $\overline{B(\theta^\sub,r/2)}\subset\overline{B(\theta_0^\star,r)}$, so the weighted Hessian is bounded below by $(\mu/2)I$ throughout
$\overline{B(\theta^\sub,r/2)}$. 

\paragraph{Inside the ball.} Fix $\theta\in \overline{B(\theta^\sub,r/2)}$.  The segment from $\theta^\sub$ to $\theta$ lies in this ball, so the mean-value form of second-order Taylor's theorem yields some $\xi$ on the segment with
\begin{align*}
L(\theta;\bar\bw)
 = L(\theta^\sub;\bar\bw)
+\nabla_\theta L(\theta^\sub;\bar\bw)^\top(\theta-\theta^\sub)
+\tfrac{1}{2}(\theta-\theta^\sub)^\top H(\xi;\bar\bw)(\theta-\theta^\sub).
\end{align*}
where the first term disappears by the first-order condition. The Hessian PD bound $H(\xi;\bar\bw)\succeq(\mu/2)I$ at this $\xi$ gives the first branch of~\eqref{eqn:appB-trunc-loj}:
\begin{align}
\label{eqn:appB-A5-inside}
L(\theta;\bar\bw)-L(\theta^\sub;\bar\bw)
 \ge \tfrac{\mu}{4} \|\theta-\theta^\sub\|^2.
\end{align}

\paragraph{Outside the ball.} Fix $\theta\in\Theta$ with
$\|\theta-\theta^\sub\|>r/2$, and set $t:=(r/2)/\|\theta-\theta^\sub\|\in(0,1)$. The point \( \theta_t:=\theta^\sub+t(\theta-\theta^\sub) \) lies on the boundary $\partial B(\theta^\sub,r/2)$, so~\eqref{eqn:appB-A5-inside} applied at $\theta_t$ gives \( L(\theta_t;\bar\bw)-L(\theta^\sub;\bar\bw)\ge(\mu/4)(r/2)^2=\mu r^2/16. \) By (A4), $L(\cdot;\bar\bw)$ is convex on $\Theta$, so the restriction $s\mapsto L(\theta^\sub+s(\theta-\theta^\sub);\bar\bw)$ to $s\in[0,1]$ is convex with value $L(\theta^\sub;\bar\bw)$ at $s=0$. The standard chord inequality $f(s)\le sf(1)+(1-s)f(0)$ applied at $s=t$ gives \( L(\theta_t;\bar\bw)-L(\theta^\sub;\bar\bw) \le t(L(\theta;\bar\bw)-L(\theta^\sub;\bar\bw)), \) so
\begin{align}
\label{eqn:appB-A5-outside}
L(\theta;\bar\bw)-L(\theta^\sub;\bar\bw)
 \ge \frac{1}{t}\cdot\frac{\mu r^2}{16}
 = \frac{\mu r}{8} \|\theta-\theta^\sub\|,
\end{align}
which is the second branch of~\eqref{eqn:appB-trunc-loj}.

\paragraph{From growth to expected squared norm.}
For a random $\theta\in\Theta$, split on the event $A:=\{\|\theta-\theta^\sub\|\le r/2\}$. On $A$, inverting~\eqref{eqn:appB-A5-inside}, \( \|\theta-\theta^\sub\|^2 \le (4/\mu)(L(\theta;\bar\bw)-L(\theta^\sub;\bar\bw)). \) On $A^c$, inverting~\eqref{eqn:appB-A5-outside}, \( \|\theta-\theta^\sub\| \le (8/(\mu r))(L(\theta;\bar\bw)-L(\theta^\sub;\bar\bw)). \) Since $\theta,\theta^\sub\in\Theta$ and $\diam(\Theta)=D_\Theta$, $\|\theta-\theta^\sub\|\le D_\Theta$, hence \( \|\theta-\theta^\sub\|^2 \le D_\Theta \|\theta-\theta^\sub\| \le (8D_\Theta/(\mu r))(L(\theta;\bar\bw)-L(\theta^\sub;\bar\bw)). \) Adding the two pieces with the indicator $\mathbf 1_\mathcal E$,
\begin{align*}
\mathbf 1_\mathcal E \|\theta-\theta^\sub\|^2
 \le 
(\tfrac{4}{\mu}+\tfrac{8D}{\mu r}) 
\mathbf 1_\mathcal E(L(\theta;\bar\bw)-L(\theta^\sub;\bar\bw)).
\end{align*}
Taking expectations and factoring $(4/\mu)(1+2D_\Theta/r)$ gives~\eqref{eqn:appB-bridge}.
\end{proof}

\subsubsection{Proof of Theorem~\ref{thm:conv-gbb}}
\label{app:thm-conv-pf}

The proof assembles the three lemmas (Lemmas~\ref{lem:appB-doob}-\ref{lem:appB-var-bridge}).

\paragraph{Unrolling the one-step descent.}
Define the $\phi$-space function-value excess $u_t := \E_{\bX_{1:n},\pi,\bomega_T}[\mathcal{L}(\phi_t;\bX_{1:n},\pi)-\mathcal L^\star_{\Phi,\bX_{1:n},\pi}]$, and it is assumed that each update satisfies $\phi_t\in\Phi$ for every $t$. Recall that $\mathcal L(\cdot;\bX_{1:n},\pi)$ is $M_{\mathcal L}$-smooth, i.e., the finite constant $M_{\mathcal L}=\beta_1^2K_2+\beta_2C$. Taking $\E_{\bX_{1:n},\pi,\bomega_T}[\cdot]$ on both sides of~\eqref{eqn:appB-onestep}, we observe $u_{t+1}\;\le\;\rho u_t\;+\;(M_{\mathcal L}\eta^2/2) \tau^2$ with $\rho:=1-\eta q\mu_{\mathcal L}\in[0,1)$. By unrolling it,
\begin{align}
\label{eqn:appB-uT}
u_T
 \le \rho^T u_0 + \tfrac{M_{\mathcal L}\eta^2\tau^2}{2}\sum_{s=0}^{T-1}\rho^s
 \le \rho^T\E_{\bX_{1:n},\pi,\phi_0}(\mathcal{L}(\phi_0)-\mathcal L^\star_{\Phi,\bX_{1:n},\pi}) + \tfrac{M_{\mathcal L}\eta\tau^2}{2q\mu_{\mathcal L}},
\end{align}
since $u_0=\E_{\bX_{1:n},\pi,\phi_0}\mathcal{L}(\phi_0)-\mathcal L^\star_{\Phi,\bX_{1:n},\pi}$ and $\sum_{s\ge 0}\rho^s\le 1/(1-\rho)=1/(\eta q\mu_{\mathcal L})$.

\paragraph{Converting $\phi$-space into $\theta$-space.}
Define the function-value excess of the iterate against the weighted minimizer,
\begin{align}
\label{eqn:appB-DeltaT}
\Delta_T
 :=
\E_{\bX_{1:n},\pi,\bomega_T,\bu}[L(G_{\phi_T}(\bu);\bar\bw(\bu))-L(\theta^\star_\sub(\bar\bw(\bu));\bar\bw(\bu))],
\end{align}
where the expectation is over the training randomness $\bomega_T$, the fresh draw $\bu$, and the data $(\bX_{1:n},\pi)$. The integrand in~\eqref{eqn:appB-DeltaT} is non-negative because $\theta^\star_\sub(\bar\bw)$ is by definition a global minimizer of $L(\cdot;\bar\bw)$ on $\Theta$. With $\mathcal{L}_{{\rm opt},\bX_{1:n},\pi}:=\E_\bu[L(\theta^\star_\sub(\bar\bw(\bu));\bar\bw(\bu))]$ and $\mathcal{L}(\phi;\bX_{1:n},\pi):=\E_\bu[L(G_\phi(\bu);\bar\bw(\bu))]$, the tower property gives \( \Delta_T=\E_{\bX_{1:n},\pi,\bomega_T}[\mathcal{L}(\phi_T;\bX_{1:n},\pi)-\mathcal{L}_{{\rm opt},\bX_{1:n},\pi}]=\E[u_T]+\varepsilon_{\rm app}, \) where $\varepsilon_{\rm app}=\E_{\bX_{1:n},\pi}[\mathcal L^\star_{\Phi,\bX_{1:n},\pi}-\mathcal L_{{\rm opt},\bX_{1:n},\pi}]\ge 0$ is the network-class approximation gap of the in-class minimum to the per-realization target. Hence, \eqref{eqn:appB-uT} leads to
\begin{align}
\label{eqn:appB-DeltaT-bound}
\Delta_T
 \le \rho^T\E_{\bX_{1:n},\pi,\phi_0}[\mathcal{L}(\phi_0;\bX_{1:n},\pi)-\mathcal L^\star_{\Phi,\bX_{1:n},\pi}]+\tfrac{M_{\mathcal L}\eta\tau^2}{2q\mu_{\mathcal L}}+\varepsilon_{\rm app}.
\end{align}
With $\theta=G_{\phi_T}(\bu)\in\Theta$, Lemma~\ref{lem:appB-var-bridge} is applied to
\begin{align}
\label{eqn:appB-onE}
\E_{\bX_{1:n},\pi,\bomega_T,\bu}[\mathbf 1_{\mathcal E} \|G_{\phi_T}(\bu)-\theta^\star_\sub(\bar\bw(\bu))\|^2]
 \le 
C_\sub \Delta_T,
\end{align}
where $C_\sub=(8/\mu_\infty)(1+2D_\Theta/r)$.

\paragraph{The bad set $\mathcal E^c$.}
On the complement event $\mathcal E^c$, both $G_{\phi_T}(\bu)$ and $\theta^\star_\sub(\bar\bw(\bu))$ lie in $\Theta$ (by (A2) and the definition of $\theta^\star_\sub$ as a minimizer on $\Theta$), so the diameter bound gives, pointwise, $\|G_{\phi_T}(\bu)-\theta^\star_\sub(\bar\bw(\bu))\|^2\le D_\Theta^2$. Therefore
\begin{align}
\label{eqn:appB-offE}
\E_{\bX_{1:n},\pi,\bomega_T,\bu}[\mathbf 1_{\mathcal E^c} \|G_{\phi_T}(\bu)-\theta^\star_\sub(\bar\bw(\bu))\|^2]
 \le D_\Theta^2 \P(\mathcal E^c).
\end{align}
By Lemma~\ref{lem:cm-reg} the regularity event obeys $\P(\mathcal G_n^c)=O(e^{-c'n})$. The localization thresholds $\mu_n/4$ and $\mu_n r/4$ are fixed constants, so a fourth-moment Markov inequality with Lemma~\ref{lem:cm-four}, which gives $\E_{u,\pi}[\|g_\sub\|^4\mid \bX_{1:n}]=O(\gamma^{-3}n^{-2})$ and $\E_{u,\pi}[\|B_\sub\|_F^4\mid \bX_{1:n}]=O(\gamma^{-3}n^{-2})$ on $\mathcal G_n$, yields
\begin{align*}
\P(\mathcal G_n\cap\tilde E_n^{\sub,c})
\le\frac{\E[\|B_\sub\|_F^4\mathbf 1_{\mathcal G_n}]}{(\mu_n/4)^4}
+\frac{\E[\|g_\sub\|^4\mathbf 1_{\mathcal G_n}]}{(\mu_n r/4)^4}
=O(n^{-2}).
\end{align*}
Hence $\P(\mathcal E^c)\le\P(\mathcal G_n^c)+\P(\mathcal G_n\cap\tilde E_n^{\sub,c})=O(n^{-2})$. Summing~\eqref{eqn:appB-onE} and~\eqref{eqn:appB-offE} and substituting~\eqref{eqn:appB-DeltaT-bound},
\begin{align*}
\E_{\bX_{1:n},\pi,\bomega_T,\bu}\|G_{\phi_T}(\bu)-\theta^\star_\sub(\bar\bw(\bu))\|^2
 \le C_\sub \Delta_T+D_\Theta^2 \P(\mathcal E^c),
\end{align*}
which is~\eqref{eqn:thm-conv-G}.\qed

\newpage 

\subsection{Proof of Corollary~\ref{prop:var-transfer}}
\label{app:proof-var-transfer}

Recall that the variance is defined as $\Var(Z):=\E\|Z-\E Z\|^2$.

\begin{lemma}
\label{lem:var-stable}
Let $A,B\sim P$ be jointly distributed random vectors in $\R^{d}$, taking a diameter $D$ with finite second moments. Then
\begin{align*}
\bigl|\Var_P(A)-\Var_P(B)\bigr| \le 2D \E_P\|A-B\| \le 2D \bigl(\E_P\|A-B\|^2\bigr)^{1/2}.
\end{align*}
\end{lemma}
\begin{proof}
Let $(A',B')$ be an independent copy of the pair $(A,B)$. Since $\E\langle Z-\E Z,Z'-\E Z\rangle=0$ for an independent copy $Z'$, $\Var(Z)=\tfrac12\E\|Z-Z'\|^2$ holds. Thus,
\begin{align*}
\Var(A)-\Var(B)=\tfrac12 \E\bigl[(\|A-A'\|-\|B-B'\|)(\|A-A'\|+\|B-B'\|)\bigr].
\end{align*}
By the reverse triangle inequality applied to $A-A'$ and $B-B'$, $|\,\|A-A'\|-\|B-B'\|\,|\le\|(A-A')-(B-B')\|=\|(A-B)-(A'-B')\|\le\|A-B\|+\|A'-B'\|$. Hence $|\Var(A)-\Var(B)|\le D(\E\|A-B\|+\E\|A'-B'\|)=2D\,\E\|A-B\|$, and Jensen's inequality gives the last bound.
\end{proof}

We split the deviation by inserting the ideal block target.

\smallskip\noindent\emph{Step 1 (generator vs.\ block).}\;
Apply Lemma~\ref{lem:var-stable} conditionally on $(\bX_{1:n},\pi,\bomega_T)$ with $A=G_{\phi_T}^{(\gamma)}(\bu)$ and $B=\theta^\star_\sub(\bar\bw_\gamma(\bu))$, the variance taken over the fresh draw $\bu$ (with an independent copy $\bu'$). Now, averaging over $(\bX_{1:n},\pi,\bomega_T)$ via $|\E(\cdot)|\le\E|\cdot|$ gives
\begin{align*}
\bigl|\E_{\bX_{1:n},\pi,\bomega_T}\Var_\bu[G_{\phi_T}^{(\gamma)}]-\E_{\bX_{1:n},\pi}\Var_\bu[\theta^\star_\sub]\bigr|
\le 2D_\Theta \E_{\bX_{1:n},\pi,\bomega_T,\bu}\bigl\|G_{\phi_T}^{(\gamma)}-\theta^\star_\sub\bigr\|\le 2D_\Theta \Delta_T(\gamma)^{1/2},
\end{align*}
where $\Delta_T(\gamma)=\E_{\bX_{1:n},\pi,\bomega_T,\bu}\bigl\|G_{\phi_T}^{(\gamma)}-\theta^\star_\sub\bigr\|^2$.

\smallskip\noindent\emph{Step 2 (triangle split).}\;
Inserting $\theta^\star_\sub$,
\begin{align*}
    &\bigl|\E_{\bX_{1:n},\pi,\bomega_T}\Var_\bu[G_{\phi_T}^{(\gamma)}]-\E_{\bX_{1:n}}\Var_\bw[\theta^\star_\Diri]\bigr|\\
&\le \bigl|\E_{\bX_{1:n},\pi,\bomega_T}\Var_\bu[G_{\phi_T}^{(\gamma)}]-\E_{\bX_{1:n},\pi}\Var_\bu[\theta^\star_\sub]\bigr|
+\bigl|\E_{\bX_{1:n},\pi}\Var_\bu[\theta^\star_\sub]-\E_{\bX_{1:n}}\Var_\bw[\theta^\star_\Diri]\bigr|.
\end{align*}
Therefore, the first term is handled by the quantity in Step 1.

\smallskip\noindent\emph{Step 3}
Proposition~\ref{prop:choice-alpha} gives
\begin{align*}
\E_{\bX_{1:n},\pi}\Var_\bu[\theta^\star_\sub]
-\E_{\bX_{1:n}}\Var_\bw[\theta^\star_\Diri]
=a_{S,n}(\gamma)a_n+O((1+\gamma^{-3/2})n^{-3/2}+\gamma^{-3}n^{-2}).
\end{align*}
By combining Steps 1-3, we have the statement. \qed

\newpage

\section{Implementation Details}
\label{app:simulation}

This section details the implementation of every method in the three simulation studies of Section~\ref{sec:simul}: the linear-regression study under heavy-tailed errors (Section~\ref{sec:simul-linear}) and the exponential-family study (Section~\ref{app:expo-family}).

\subsection{Competing Methods}
\label{sec:supp_competitors}

We compare against three differentially private baselines. Each represents a major branch of private Bayesian or bootstrap inference. We briefly describe them and give the method-specific implementation in the subsections that follow. Refer to the original papers for more details.

\paragraph{DP-SGLD.} Our customized differentially private stochastic-gradient Langevin dynamics privatizes the score at each Langevin step. Each update clips the per-sample gradient and adds Gaussian noise for privacy, and the iterates are taken as approximate posterior samples. The Langevin Gaussian noise decays by $t^{-1/3}$ following \citep{bai:etal:19}. We acknowledge that a single draw from DP-SGLD need not invite additional Gaussian noise for privacy under regularity conditions. However, for UQ with many draws, we perturb the gradient as a safeguard. Its privacy accountant has the same systematic structure as ours for fair comparison; the replace-one R\'enyi bound of \citet[Theorem~C.9]{birr:etal:24} is composed over the $T$ iterations and converted to $(\epsilon,\delta)$-DP by the Asoodeh conversion of Section~\ref{app:conversion-to-dp}.

\paragraph{NoisyGibbs \citep{bern:shel:18,bern:shel:19}.} It releases private sufficient statistics under the Gaussian mechanism, treats the noisy statistics as observed and the true statistics as latent, and runs a Gibbs sampler over the latent statistics and the parameters. The posterior draws are DP by post-processing. We use \cite{bern:shel:18} for the exponential-family experiment and \cite{bern:shel:19} for the linear-regression experiment. Both require specifying a working likelihood. The single release is the classical Gaussian mechanism, i.e., the sufficient statistics are perturbed with noise scale $\sigma=\Delta\sqrt{2\log(1.25/\delta)}\,/\,\epsilon$ \citep{dwor:etal:14}, where $\Delta$ is the $\ell_2$ sensitivity of the sufficient statistics over the bounded design. 

\paragraph{DP-Bootstrap \citep{wang:etal:25}.} This is a private nonparametric bootstrap. It resamples the data, fits an $\ell_2$-regularized $M$-estimator with output perturbation on each bootstrap replicate, and releases $B$ replicates. The collection satisfies $\mu$-GDP \citep{dong:etal:22} as $B\to\infty$, with output-perturbation scale $\sigma_e=\sqrt{(2-2/e)\,B}\,\Delta/\mu$ where $\Delta$ is the sensitivity. To match the other methods, we turn the target $(\epsilon,\delta)$ into the largest $\mu$ whose Gaussian-DP tradeoff function satisfies $\delta(\epsilon;\mu)\le\delta$ \citep[Corollary 2.13]{dong:etal:22}, so the release is calibrated to the same $(\epsilon,\delta)$ budget.

\subsection{Linear-regression and Exponential-family Experiments}
\label{sec:supp_linear}
\paragraph{Data.} The regression study is that of Section~\ref{sec:simul-linear}: $X_i\sim\mathrm{Uniform}[-1,1]^3$, $Y_i=X_i^\top\boldsymbol{\beta}^\star+\varepsilon_i$ with $\boldsymbol{\beta}^\star=(1.5,-1.5,1.0)$ and $\sigma_\varepsilon=0.5$, under Student-$t$ ($\mathrm{df}=3$) and contaminated-normal ($0.9\,N(0,\sigma_\varepsilon^2)+0.1\,N(0,64\sigma_\varepsilon^2)$) noise, with $n\in\{10\text{k},40\text{k}\}$. The exponential-family study is one-dimensional mean estimation under a correctly specified likelihood: a Gaussian model $X_i\sim N(\mu,\sigma^2)$ with known $\sigma=0.5$ and estimand $\mu=1.5$, and an Exponential model $X_i\sim\mathrm{Exp}(\lambda)$ with estimand $1/\lambda=2.0$, using $n\in\{500,5000\}$. Both studies use $\epsilon\in\{1,2,4\}$ at $\delta=1/n$ and $R=100$ replicates. In regression, the covariates lie in the public box $B_x=1$ ($\|X_i\|\le B_x\sqrt p$). In the exponential-family study, each datum is truncated to a public interval $[a,b]$ ($[-5,5]$ for the Gaussian, $[0,12]$ for the Exponential), which bounds the sufficient statistic $t(x)=x$ at replace-one sensitivity $\Delta=b-a$ and, for the score-based methods (DP-SGLD and PGBB), the per-sample derivative is clipped at $C=8$. For Bayesian methods, the credible interval is the $95\%$ quantile interval of the draws. 

In the regression study, DP-Bootstrap, DP-SGLD, and PGBB can use the robust Huber loss of \citet{avel:etal:23}, whose score is bounded independently of $Y$ and $\beta$. In the exponential-family study, DP-Bootstrap, DP-SGLD, and PGBB clip the negative log-likelihood at $C$. For the regression, they estimate $\beta$ and the scale $\sigma$ with Huber constant $c=1.345$ and consistency constant $\kappa=0.71$, and the per-sample gradient has $\ell_2$ norm bounded by $C_H=\sqrt{c^2 p B_x^2+(\tfrac12(c^2+\kappa))^2}$. 

\paragraph{PGBB.} The generator $G_\phi:\mathrm{Dir}(\gamma\mathbf 1_S)\to\mathbb R^p$ is a three-hidden-layer width-$256$ ReLU multilayer perceptron with the Dirichlet input re-concatenated after each hidden block \citep{shin:etal:21} and the linear output clipped to the parameter box. In the regression, it uses $p=3$, $S=20$ blocks ($S\in\{5,20\}$ in the block-count study of Table~\ref{tab:K-comparison}), $T=10000$ iterations, $K=1$ Monte Carlo replicate, Poisson subsampling rate $q=0.1$, the Adam optimizer with a cosine-decayed learning rate from $5\times10^{-4}$ to a $1\%$ floor, gradient bound $C=C_H$, truncation orders $J=32$ and $m=5$, and $\gamma=\gamma^\star(n,S)$. In the exponential-family study it uses $p=1$, $S=20$, $T=2000$, $K=1$, $q=0.2$, gradient bound $C=8$, $J=30$, $m=5$, and $\gamma=1$. The noise scale $\sigma$ is set by the bisection of Section~\ref{app:bisection-sigma}, and we release $B=2000$ posterior draws.

\paragraph{DP-SGLD.} DP-SGLD targets the Gibbs posterior rather than a parametric likelihood so that the comparison isolates the privacy mechanism. In the regression, it runs $T=20000$ Langevin updates at tempering $\omega=1$ with step size $\eta_0$ chosen so $\eta_0 n=20$ ($\eta_0=2\times10^{-3}$ at $n=10$k and $5\times10^{-4}$ at $n=40$k), Poisson subsampling rate $q=0.1$ matched to PGBB, and $50\%$ burn-in. In the exponential-family study, it clips the per-sample score at $C=8$, runs $T=40000$ updates with $50\%$ burn-in and thinning $1$. 

\paragraph{NoisyGibbs.} NoisyGibbs perturbs the sufficient statistics. In the regression, we release $(\bX^\top\bX,\ \bX^\top Y,\ Y^\top Y)$ together with the second and fourth moments of $\bX$ under the Gaussian mechanism, splitting the budget as $(\tfrac12,\tfrac14,\tfrac14)$ of $(\epsilon,\delta)$. For the sufficient-statistic block the replace-one $\ell_2$ sensitivity is $\Delta=2\sqrt{\tfrac{p(p+1)}{2}B_x^4+p B_x^2 B_y^2+B_y^4}$, the factor $2$ accounting for the simultaneous change of the removed and inserted record and the three terms being the $p(p+1)/2$ unique entries of $\bX^\top\bX$ (each of sensitivity $B_x^2$), the $p$ entries of $\bX^\top Y$ (each $B_x B_y$), and the single $Y^\top Y$ entry (of sensitivity $B_y^2$). The moment vectors $m_2,m_4$ are sample means, so their sensitivities carry a $1/n$: $\Delta_2=\tfrac{2}{n}B_x^2\sqrt{\tfrac{p(p+1)}{2}}$ and $\Delta_4=\tfrac{2}{n}B_x^4\sqrt{D_4}$ with $D_4=\binom{p+3}{4}$. Each release uses $\sigma=\Delta\sqrt{2\log(1.25/\delta')}/\epsilon'$ for its split budget, and the noisy $m_2,m_4$ enter $\mu_t,\Sigma_t$ as the covariate moments $\eta_{ij},\eta_{ijkl}$ (equations (2)-(3) of \citealp{bern:shel:19}). The noise-aware Gibbs sampler runs $400\text{k}$ iterations with $200\text{k}$ burn-in and thinning $10$, returning $20\text{k}$ draws. In the exponential-family study, the sufficient statistic is the scalar $t(x)=x$: we perturb its truncated sum by the Gaussian mechanism at replace-one sensitivity $\Delta=b-a$ and run the noise-aware Gibbs sampler of \citet{bern:shel:18} for $4\text{k}$ iterations with half burn-in and thinning $2$. 

\paragraph{DP-Bootstrap.}

For the regression study, we implement the DP-ERM-like implementation of DP-Bootstrap that reports the asymptotic interval that accounts for the known output-perturbation variance, i.e., Algorithm 2 of \cite{wang:etal:25}. For the univariate exponential-family study, we implement Algorithm 3 of \cite{wang:etal:25}. The target is the mean of the truncated variable on $[a,b]$ where $a=-5,b=5$ in our study, so the base statistic is the bootstrap sample mean rather than an $\ell_2$-regularized ERM. Its replace-one sensitivity is therefore $\Delta=(b-a)/n$.

\subsection{Real Data Application}
\label{sec:supp_realdata}

Both applications in Section~\ref{sec:realdata} are differentially private quantile regressions over the levels $\tau\in\{0.1,0.3,0.5,0.7,0.9\}$, at privacy budgets $\epsilon\in\{1,4\}$ with $\delta=1/n$ and 10 replicates. The U.S. Census is the returns-to-schooling regression for the years 1980 ($n=65,023$) and 2000 ($n=97,397$), where the design is $\bX=[\,1,\text{schooling},\text{experience},\text{black}\,]$ with $p=4$. Natality is the birthweight regression on the 2024 U.S.\ records ($n\approx1.94$M; intercept plus $13$ covariates, $p=14$). In the U.S.\ Census regression, the response is the log weekly wage, \texttt{educ} is completed years of schooling, \texttt{exper} is potential labor-market experience (age minus schooling minus six), and \texttt{black} indicates Black respondents. Covariates are scaled into the unit box so that $B_x=1$ and $\|X_i\|\le B_x\sqrt p$. This rescaling uses only each variable's presumably publicly known range, i.e., the published Census and natality variable bounds. In both, PGBB uses $S=5$ and the variance-matched concentration $\gamma=\gamma^\star(n,S)$. We compare PGBB against DP-SGLD, where the non-private quantile fit serves as the reference. Both private methods use the convolution-smoothed (kernel) quantile loss $\rho_{\tau,h}=\rho_\tau*K_h$ \citep{tan:etal:22} at bandwidth $h=0.1$, whose derivative $\Phi(r/h)-(1-\tau)$ (with the Gaussian kernel $K$) lies in the open interval $(\tau-1,\tau)$ and is bounded in magnitude by $\max(\tau,1-\tau)$, so a single record's gradient is bounded by $\max(\tau,1-\tau)\,B_x\sqrt p$. The coefficients in Tables~\ref{tab:supp-acfv} and~\ref{tab:supp-natality} and Figures~\ref{fig:acfv_allvars} and~\ref{fig:natality_allvars} are reported in interpretable units rather than the scaled design, with \texttt{educ} in \%/yr, \texttt{exper} in /yr, and \texttt{black} in log-wage for the U.S.\ Census and grams for Natality. This conversion uses only the public variable ranges and so carries no privacy cost.

\paragraph{PGBB (joint-$\tau$).} A single $\tau$-joint generator $G_\phi(\bu,\tau)$ is trained once to cover the whole quantile levels: it takes the Dirichlet weights and the level $\tau$ as input and outputs $\beta(\tau)$. Because one mechanism serves every $\tau$, the per-record gradient has the uniform bound $C =  (\max_{\tau}\max(\tau,1-\tau)) B_x\sqrt p  =  0.9 B_x\sqrt p$, which dominates $\max(\tau,1-\tau)$ at every candidate level, so the entire process is released under one $(\epsilon,\delta)$ budget regardless of the number of levels. Training uses the Adam optimizer (learning rate $10^{-3}$), $S=5$, Poisson subsampling rate $q=0.01$, $K=5$ Monte Carlo draws, Beta-MGF truncation $J=32$, convolution-smoothing bandwidth $h=0.1$, and $T=5000$ (Census) / $T=10000$ (Natality) iterations, with $\gamma=\gamma^\star(n,S)$. After training, we draw $B=2000$ posterior samples of $\beta(\tau)$ at each $\tau$.


\paragraph{DP-SGLD (per-$\tau$).} DP-SGLD targets the tempered Gibbs posterior $p(\beta\mid\{Y_i,X_i\})\propto\exp\bigl\{-\omega\textstyle\sum_i \rho_{\tau,h}(Y_i-X_i^\top\beta)\bigr\}$ of the same convolution-smoothed loss, with scaler $\omega$. We report the performance of DP-SGLD for $\omega\in\{1,0.1\}$. The tempering enters only the target posterior, whose precision scales as $n\omega$, and leaves the privacy accounting unchanged. Its per-sample score $\{\Phi(r_i/h)-(1-\tau)\}X_i$ is clipped at the per-level threshold $L(\tau)=\max(\tau,1-\tau)\,B_x\sqrt p$. Because the smoothed score is strictly interior to $[\tau-1,\tau]$, the clip never binds for bounded $X_i$. The score is perturbed by Gaussian noise at each of $T=200\text{k}$ Langevin updates under the matched Poisson rate $q=0.01$. The second half of the chain is retained after thinning by $10$, giving $10^{4}$ iterates. Unlike PGBB, DP-SGLD is fit separately at each of the $M=5$ quantile levels. Rather than splitting the budget by basic $(\epsilon/M,\delta/M)$ composition, we account for the levels jointly at the R\'enyi differential privacy level \citep{miro:17}: the $M$ fits form $M\times T$ Poisson-subsampled Gaussian releases whose R\'enyi divergences add over a shared order grid \citep{birr:etal:24}, and a single R\'enyi-to-$(\epsilon,\delta)$ conversion \citep{asoo:etal:21} certifies the entire released quantile process at the total budget $(\epsilon,1/n)$. 

\paragraph{Non-private reference.} The comparison baseline is the ordinary (non-private) quantile fit on the same design and scaling. For the U.S.\ Census, whose samples are moderate, we use the \texttt{statsmodels} package for quantile regression in Python. For Natality, where $n\approx1.94$M makes the hard check loss costly, we use the convolution-smoothed estimator \citep{tan:etal:22} solved by Newton iteration. Its $95\%$ intervals come from a multiplier bootstrap that re-solves the smoothed fit under $\mathrm{Exp}(1)$ weights ($B=500$ replicates, the $2.5/97.5$ empirical percentiles). Both are evaluated at every $\tau$ and serve only as the non-private target. 


\renewcommand{\refname}{Supplementary References}
\putbib[ref]
\end{bibunit}
\end{document}